\documentclass[11pt]{article}
\usepackage[round]{natbib}

\usepackage{xr-hyper}
\usepackage{amsmath,amssymb,amsthm}
\usepackage{bm,bbm,commath,dsfont}
\usepackage{fullpage}
\usepackage{hyperref,cleveref}
\usepackage{mathrsfs}
\usepackage{xcolor}
\usepackage{enumitem}
\usepackage{graphicx}

\usepackage{booktabs}
\usepackage{authblk}

\usepackage{algorithm}
\usepackage{algpseudocode}
\algnewcommand\algorithmicinput{\textbf{INPUT:}}
\algnewcommand\INPUT{\item[\algorithmicinput]}
\algnewcommand\algorithmicoutput{\textbf{OUTPUT:}}
\algnewcommand\OUTPUT{\item[\algorithmicoutput]}

\usepackage{pifont}
\newcommand{\xmark}{\ding{55}}
\newcommand{\cmark}{\ding{51}} 

\definecolor{darkred}{RGB}{139,0,0}
\definecolor{darkgrey}{rgb}{0.4, 0.0, 0.2}

\usepackage{hyperref}[]
\hypersetup{
    colorlinks=true,
    linkcolor=blue,
    filecolor=blue,      
    urlcolor=cyan,
    citecolor=red,
    }

\newtheorem{lemma}{Lemma}
\newtheorem{proposition}{Proposition}
\newtheorem{corollary}{Corollary}

\newtheorem{definition}{Definition}
\newtheorem{remark}{Remark}
\newtheorem{theorem}{Theorem}
\newtheorem{assumption}{Assumption}
\newtheorem{example}{Example}

\newcommand{\op}{\text{op}}

\newcommand{\var}{\mathrm{Var}}

 \DeclareMathOperator*{\argmin}{arg\,min}

\newcommand{\R}{\mathbb{R}}

\newcommand{\K}{\mathcal{K}}
\newcommand{\lft}{\left}
\newcommand{\rgt}{\right}

\newcommand{\calL}{{\mathcal{L}}}

\newcommand{\bbN}{{\mathbb{N}}}

\newcommand{\what}[1]{\widehat{#1}}

\newcommand{\tbar}{{\underline{ \mathsf t }}}

\newcommand{\real}{{\mathbb{R}}}

\newcommand{\calN}{{\mathcal{N}}}

\newcommand{\calD}{{\mathcal{D}}}

\newcommand{\bbP}{{\mathbb{P}}}
\newcommand{\rmW}{{\mathrm{W}}}

\newcommand{\be}{\mathbf{e}}

\newcommand{\bu}{\mathbf{u}}
\newcommand{\bv}{\mathbf{v}}
\newcommand{\bw}{\mathbf{w}}
\newcommand{\bx}{\mathbf{x}}

\newcommand{\by}{\mathbf{y}}

\newcommand{\balpha}{\bm{\alpha}}

\newcommand{\bgamma}{\bm{\gamma}}

\newcommand{\wgnp}{{\what{g}^{(n)}({\bf{Z^\prime}})}}
\newcommand{\wgn}{{\what{g}^{(n)}({\bf{Z}})}}
\newcommand{\wgnd}{{\what{g}^{(n)}({\bf{Z}},{\bf{Z^\prime}})}}

\newcommand{\gnkp}{{g}^{(n)}_k({\bf{Z^\prime}})}
\newcommand{\gnkd}{{g}^{(n)}_k  ({\bf{Z}}, {\bf{Z^\prime}}) }

\def\one{\mathbbm 1}

\def\var{\mathsf{Var}}

\def\d{\textup{d}}

\def\R{\mathbb{R}}

\makeatletter
\def\greekvectors#1{%
 \@for\next:=#1\do{%
    \def\X##1;{\expandafter\def\csname b##1\endcsname{\bm{\csname##1\endcsname}}}
    \expandafter\X\next;}
 \@for\next:=#1\do{%
    \def\X##1;{\expandafter\def\csname h##1\endcsname{\widehat{\csname##1\endcsname}}}
    \expandafter\X\next;}
 \@for\next:=#1\do{%
    \def\X##1;{\expandafter\def\csname c##1\endcsname{\check{\csname##1\endcsname}}}
    \expandafter\X\next;}
 \@for\next:=#1\do{%
    \def\X##1;{\expandafter\def\csname hb##1\endcsname{\widehat{\bm{\csname##1\endcsname}}}}
    \expandafter\X\next;}
}
\greekvectors{alpha,beta,gamma,delta,epsilon,zeta,theta,kappa,lambda,
    mu,nu,xi,pi,rho,tau,phi,chi,psi,omega,
    Delta,Gamma,Theta,Lambda,Xi,Pi,Sigma,Phi,Psi,Omega
}
\@tfor\next:=abcfghijklmnopqrstuvwxyzABCDGHIJKLMOQSTUVWXYZ\do{%
    \def\command@factory#1{\expandafter\def\csname #1\endcsname{\mathbf{#1}} }
    \expandafter\command@factory\next
}
\@tfor\next:=abcdefghijklmnpqrstuvwxyzABCDEFGHIJKLMNOPQRSTUVWXYZ\do{%
    \def\command@factory#1{\expandafter\def\csname t#1\endcsname{\widetilde{#1}} }
    \expandafter\command@factory\next
}
\@tfor\next:=abcdefghijklmnopqrstuvwxyzABCDEFGHIJKLMNOPQRSTUVWXYZ\do{%
    \def\command@factory#1{\expandafter\def\csname tb#1\endcsname{\tilde{\mathbf{#1}}} }
    \expandafter\command@factory\next
}
\@tfor\next:=abcdefghijklmnopqrstuvwxyzABCDEFGHIJKLMNOPQRSTUVWXYZ\do{%
    \def\command@factory#1{\expandafter\def\csname hb#1\endcsname{\widehat{\mathbf{#1}}} }
    \expandafter\command@factory\next
}
\@tfor\next:=ABCDEFGHIJKLMNOPQRSTUVWXYZ\do{%
    \def\command@factory#1{\expandafter\def\csname b#1\endcsname{\mathbbm{#1}} }
    \expandafter\command@factory\next
}
\@tfor\next:=ABCDEFGHIJKLMNOPQRSTUVWXYZ\do{%
    \def\command@factory#1{\expandafter\def\csname c#1\endcsname{\mathcal{#1}} }
    \expandafter\command@factory\next
}
\@tfor\next:=abcdefghjklmnopqrstuvwxyzABCDEFGHIJKLMNOPQRSTUVWXYZ\do{
    \def\command@factory#1{\expandafter\def\csname f#1\endcsname{\mathfrak{#1}} }
    \expandafter\command@factory\next
}
\@tfor\next:=abcdefghjklmnopqrstuvwxyzABCDEFGHIJKLMNOPQRSTUVWXYZ\do{%
    \def\command@factory#1{\expandafter\def\csname s#1\endcsname{\mathsf{#1}} }
    \expandafter\command@factory\next
}

\newcommand{\lrs}[1]{\left( #1 \right)}
\newcommand{\lrm}[1]{\left\{ #1 \right\}}
\newcommand{\lrl}[1]{\left[ #1 \right]}
\newcommand{\lrvert}[1]{\left| #1 \right|}
\newcommand{\lrnorm}[1]{\left\| #1 \right\|}
\newcommand{\lrangle}[1]{\left\langle #1 \right\rangle}

\newcommand{\proj}{{\mathrm{Proj}}}

\newcommand{\Wslice}{\mathrm{W}_{1,\mathrm{slice}}}
\newcommand{\WsliceD}{\mathrm{W}_{1,\mathrm{slice}}^{784}}
\newcommand{\WsliceBL}{\mathrm{W}_{1,\mathrm{slice}}^{\mathrm{BL}}}
\newcommand{\Wstd}{\mathrm{W}_{1,\mathrm{slice}}^{\mathrm{std}}}
\newcommand{\WstdBL}{\mathrm{W}_{1,\mathrm{slice}}^{\mathrm{std,BL}}}
\DeclareMathOperator{\Cov}{Cov}
\DeclareMathOperator{\Vol}{Vol}
\DeclareMathOperator{\Var}{Var}

\date{\vspace{-5ex}}
\title{Flow Matching is Adaptive to Manifold Structures}

\begin{document}
\author[1]{Shivam Kumar}
\author[2]{Yixin Wang}
\author[3]{Lizhen Lin}
\affil[1]{Booth School of Business, University of Chicago}
\affil[2]{Department of Statistics, University of Michigan}
\affil[3]{Department of Mathematics, University of Maryland, College Park}
\maketitle
\begin{abstract}
    Flow matching has emerged as a simulation-free alternative to diffusion-based generative modeling, producing samples by solving an ODE whose time-dependent velocity field is learned along an interpolation between a simple source distribution (e.g., a standard normal) and a target data distribution. Flow-based methods often exhibit greater training stability and have achieved strong empirical performance in high-dimensional settings where data concentrate near a low-dimensional manifold, such as text-to-image synthesis, video generation, and molecular structure generation. Despite this success, existing theoretical analyses of flow matching assume target distributions with smooth, full-dimensional densities, leaving its effectiveness in manifold-supported settings largely unexplained. To this end, we theoretically analyze flow matching with linear interpolation when the target distribution is supported on a smooth manifold. We establish a non-asymptotic convergence guarantee for the learned velocity field, and then propagate this estimation error through the ODE to obtain statistical consistency of the implicit density estimator induced by the flow-matching objective. The resulting convergence rate is near minimax-optimal, depends only on the intrinsic dimension, and reflects the smoothness of both the manifold and the target distribution. Together, these results provide a principled explanation for how flow matching adapts to intrinsic data geometry and circumvents the curse of dimensionality. 
\end{abstract}

\section{Introduction} 







Flow matching~\citep{albergo2023stochastic, liu2022flow, albergo2022building, lipman2022flow} has recently emerged as a simulation-free alternative to diffusion-based generative modeling, producing samples by solving an ordinary differential equation (ODE) whose time-dependent velocity field transports probability mass between distributions. Unlike diffusion models, which rely on stochastic perturbations and reverse-time SDE simulation, flow matching learns a deterministic transport map along a prescribed interpolation between a simple source distribution (e.g., a standard normal) and a target data distribution. 

The deterministic formulation of flow matching yields favorable computational properties, including stable training, flexible discretization at sampling time, and compatibility with modern continuous normalizing flow (CNF) architectures~\citep{lipman2022flow,liu2022flow}. Empirically, flow matching has achieved strong performance in high-dimensional generative tasks such as text-to-image synthesis, video generation, and molecular structure modeling, where data are known to concentrate near low-dimensional manifolds~\citep{bose2023se, graham2024proceedings, esser2024scaling, ma2024sit}.



Despite this empirical success,  theoretical foundations of flow matching remain limited. Existing analyses typically assume that the target distribution admits a smooth, full-dimensional density with respect to Lebesgue measure. This assumption is misaligned with many modern applications, where the data distribution is intrinsically low-dimensional and supported on or near a smooth manifold embedded in a high-dimensional ambient space. As a result, current theory does not explain why flow matching avoids the curse of dimensionality in practice, nor how its performance depends on intrinsic geometric structure. 

To formalize this setting, we observe an i.i.d. dataset $\calD_1 = \lft\{X_{1,j} \rgt\}_{j=1}^n$, where $X_1 \sim \bpi_1$ is drawn from a target distribution supported on a $\sd$-dimensional manifold $\cM$ embedded in the ambient space $\bR^\sD$. Flow matching constructs a continuous probability path \((X_t)_{t\in[0,1]}\) connecting a simple reference distribution \(\bpi_0\), from which sampling is straightforward, to the target distribution \(\bpi_1\). This path is governed by a time-dependent vector field \(v^\star: \bR^\sD \times [0,1] \to \bR^\sD\), and the state evolves according to the transport ODE
\begin{equation}\label{eq: FM}
\frac{dX_t}{dt} = v^\star(X_t,t), \qquad X_0 \sim \bpi_0 = {\mathtt{N}}(\bm{0},\bI_\sD), \qquad X_1 \sim \bpi_1.
\end{equation}
The goal of flow matching is to estimate the velocity field \(v^\star\) from data. Once an estimate \(\what{v}\) is obtained,  approximate samples of \(\bpi_1\) are generated  by drawing \(X_0\sim \bpi_0\) and numerically integrating the ODE \eqref{eq: FM} forward in time from \(t=0\) to \(t\approx1\).

When \(\bpi_1\) is supported on a manifold, it is singular with respect to Lebesgue measure, so the appropriate statistical target is the pushforward distribution induced by the learned dynamics. We therefore treat \(\widehat{\bpi}_{1-\tbar}\) as an implicit estimator of \(\bpi_1\) (see \eqref{eq: neural ode}), and derive non-asymptotic convergence bounds that are intrinsically nonparametric and governed by the manifold dimension.

We provide a theoretical analysis of distribution estimation using flow matching with linear interpolation, in the manifold-supported setting. Our analysis yields non-asymptotic convergence guarantees for estimating the velocity field and propagate this estimation error through the transport ODE to obtain statistical consistency of the implicit density estimator. The resulting convergence rates are near minimax-optimal, depend only on the intrinsic dimension $\sd$, and capture the smoothness of both the manifold and the target distribution.

Together, these results provide a principled explanation for why flow matching can adapt to intrinsic geometry and mitigate the curse of dimensionality.

\subsection{List of contributions}
We briefly summarize the main contributions of this paper as follows.
\begin{itemize}[leftmargin=*,itemsep=0pt]
    \item We provide a non-asymptotic error analysis of flow matching with linear interpolation when the target distribution is supported on a low-dimensional manifold embedded in \(\bR^{\sD}\). The resulting rate is near-minimax optimal and depends only on structural properties of the target distribution.
    
    \item Our convergence guarantees show that flow matching adapts to the manifold structure of the data: the statistical complexity is governed by the intrinsic dimension rather than the ambient dimension. To the best of our knowledge, this is the first work to develop a finite-sample error analysis of flow matching in the manifold-supported setting.
    
    \item We establish consistency rates for estimating the velocity field \(v^\star(\x,t)\). In particular, the estimator attains fast convergence for times bounded away from \(t=1\), while the rate deteriorates as \(t\to 1\) due to the singular behavior of the linear-path velocity field.
\end{itemize}

\subsection{Other relevant literature}

In the context of manifold-based generative modeling, our work is most closely related to \citet{tang2024adaptivity, judith2024convergence}, which develop diffusion-model theory showing how diffusion adapts to data geometry. While conceptually aligned, our setting differs in a fundamental way: flow matching is a {simulation-free} alternative to diffusion, with a distinct training objective and proof strategy. Accordingly, our technical approach is closer in spirit to the tools used in \citet{gao2024gaussian} and \citet{Kunkel2025_1000188527} to derive non-asymptotic convergence guarantees. The work of \citet{chen2023flow} studies an empirical form of flow matching on manifolds in a different regime, where both the learned velocity field and the induced flow remain entirely supported on the manifold. In contrast, our analysis allows the dynamics to evolve in the ambient space, while still adapting to the intrinsic geometry through the target distribution.

A few recent works study error analysis and convergence rates for flow matching \citep{gao2024convergence, marzouk2024distribution, fukumizu2024flow, kunkel2025distribution, zhouerror}. However, these results focus on targets supported in the full ambient space and do not explicitly exploit manifold geometry. In particular, the rates in \citet{gao2024convergence} and \citet{zhouerror} are not near minimax-optimal. Concurrent work by \citet{roy2026low} establishes iteration complexity bounds for rectified flow that adapt to the intrinsic dimension of the target support.

Beyond statistical error analysis, flow matching has also been studied from several complementary perspectives, including deterministic straightening \citep{liu2022flow, bansal2024wasserstein,  kornilov2024optimal}, fast sampling \citep{hu2024flow, gui2025depthfm}, latent structures \citep{dao2023flow, hu2024latent}, and discrete analogues \citep{davis2024fisher, gat2024discrete, su2025theoretical, cheng2025alpha}, among others.



\subsection{Notations}
We write $\bbN$ for the positive integers and $\R^m$ for $m$-dimensional Euclidean space.
For $r>0$ and $\x\in\R^{\sD}$, $\bB_r(\x)$ denotes the (closed) Euclidean ball of radius $r$ centered at $\x$. We use $a\vee b:=\max\{a,b\}$ and $a\wedge b:=\min\{a,b\}$.  Scalars are denoted by lower-case letters, vectors by bold lower-case (e.g.\ $\bx$), and matrices by bold upper-case (e.g.\ $\A$). We write \(\bI_{\sD}\in\bR^{\sD\times\sD}\) for the identity matrix. For \(p\in[1,\infty]\), \(\|\cdot\|_p\) denotes the usual \(\ell_p\) norm (and the induced operator norm for matrices). For a function \(f\), \(\|f\|_\infty := \sup_{x}|f(x)|\). The indicator of an event \(A\) is denoted by \(\one_A\). For sequences $a_n,b_n\ge 0$, we write $a_n\lesssim b_n$ if there exists an absolute constant $C>0$ (independent of $n$) such that $a_n\le C b_n$; similarly $a_n\gtrsim b_n$ and $a_n\asymp b_n$. We use $\mathcal{O}(\cdot)$ and $o(\cdot)$ in the standard sense. We write \(\mathtt{N}(\bm{m},\bSigma)\) for a Gaussian distribution with mean \(\bm{m}\) and covariance \(\bSigma\). We denote probability and expectation by \(\bP\) and \(\bE\), and conditional expectation by \(\bE[\cdot\,|\,\cdot]\). For two probability densities $\mu,\nu$ on $\R^{\sD}$ with finite $p$-th moments, $\rmW_p(\mu,\nu)$ denotes the $p$-Wasserstein distance. For a multi-index $\balpha=(\alpha_1,\dots,\alpha_{\sd})\in\bbN^{\sd}$, let $|\balpha| := \sum_{j=1}^{\sd}\alpha_j$ and $\partial^{\balpha} := \partial_1^{\alpha_1}\cdots\partial_{\sd}^{\alpha_{\sd}}$. For $\beta>0$ and a domain $D\subset\R^{\sd}$, the $\beta$-H\"older class $\mathcal{H}_{\sd}^{\beta}(D,K)$ is
\begin{equation*}
    \begin{aligned}
        \mathcal{H}_{\sd}^{\beta}(D,K) &:= \Bigg\{f:D\to\R \;:\;\sum_{|\balpha|<\beta}\|\partial^{\balpha}f\|_\infty+
        \\
        &\sum_{|\balpha|=\lfloor \beta\rfloor}\sup_{\substack{\bu_1,\bu_2\in D\\ \bu_1\neq \bu_2}} \frac{|\partial^{\balpha}f(\bu_1)-\partial^{\balpha}f(\bu_2)|}{\|\bu_1-\bu_2\|_\infty^{\beta-\lfloor\beta\rfloor}} \le K \Bigg\}.
    \end{aligned}
\end{equation*}

A map $f:\R^{\sD}\to\R^{m}$ is $L$-Lipschitz if $\|f(\x)-f(\y)\|_\infty \le L\|\x-\y\|_\infty$ for all $\x,\y$. { For $\A\in\R^{D\times D}$, the \emph{logarithmic norm} with respect to the $\ell_2$-norm is
\begin{equation}\label{eq:lognorm}
    \mu_2(\A) \;:=\; \lambda_{\max}\!\!\left(\frac{\A+\A^\top}{2}\right).
\end{equation}
}
\section{Flow matching}  
The evolution of the probability density $\bpi_t(\x)$ associated with a flow $(X_t)_{t \in [0,1]}$ is governed by the continuity (or transport) equation:
\begin{equation}\label{eq:continuity}
    \begin{aligned}
        &\partial_t \bpi_t(\x) + \nabla\cdot\big(\bpi_t(\x)v^\star(\x,t)\big)= 0,
        \\
        &\bpi_0(\x) = (\sqrt{2\pi})^{-\sD/2} \exp (-|\x|_2^2/2), \quad \bpi_1(\x).
    \end{aligned}
\end{equation}
A popular strategy is to construct a coupling \((X_0,X_1)\) and define an interpolation
\(X_t=\sF(X_0,X_1,t)\) for \(t\in[0,1]\). The resulting curve \((X_t)_{t\in [0,1]}\) induces a time-dependent
velocity field. Under appropriate regularity assumptions on the interpolation path, it is known \citep[Theorem~6]{albergo2023stochastic} that the velocity field $v^\star$ is given by the conditional expectation
\[
{{v^\star(\x,t)}} = \bE\lrl{\dot X_t \,\middle|\, X_t=\x}.
\]
\paragraph{Linear interpolation.}
Throughout this paper we focus on flow matching with the \emph{linear} interpolation path
\begin{equation}\label{eq:linear_flow}
    X_t := tX_1 + (1-t)X_0,\qquad t\in[0,1],    
\end{equation}
where \(X_0\sim \bpi_0 = \mathtt{N}(\bm{0},\bI_\sD)\) and \(X_1\sim \bpi_1\) (with \(X_0\) independent of \(X_1\)). Since
\(\dot X_t = X_1 - X_0\), the induced velocity field admits the conditional-expectation
representation
\begin{equation}\label{eq: VF}
    \begin{aligned}
        v^\star(\x,t) &= \bE\lrl{X_1 - X_0|X_t=\x} 
        \\
        &= \frac{1}{1-t} \left[ \frac{\int_{\y \in \cM} \,\y\,\bpi_1(\y)\, e^{-\frac{|\x -t\y|_2^2}{2(1-t)^2}} \, d\y}{\int_{\y \in \cM} \, \bpi_1(\y) e^{-\frac{|\x -t\y|_2^2}{2(1-t)^2}} \, d\y} - \x\right],
    \end{aligned}
\end{equation}
with a short derivation deferred to \Cref{sec: vf expression}. The derivation uses the linearity of the flow \eqref{eq:linear_flow} so that the instantaneous change is independent of time apart from the interpolation weights. Linear-interpolation flow matching has demonstrated strong empirical performance in large-scale generative modeling \citep{liu2022flow, tong2023improving,esser2024scaling}.
\paragraph{Optimization.}
Learning the velocity field in this setting amounts to formulating an optimization problem whose solution recovers $v^\star$ as in \eqref{eq: VF}. Consider the population risk functional
\begin{equation}\label{eq: opt alt}
\min_{u}\ \mathcal{L}(u) \quad \text{where}\quad \mathcal{L}(u) := \int_{0}^{1}\mathbb{E}\!\left[\bigl\|u(X_t,t)-\dot X_t\bigr\|_2^2\right]dt .
\end{equation}
In \Cref{lemma: Optimizer}, we show that \(v^\star\) is a minimizer of \(\mathcal{L}\), i.e.,
\(v^\star\in\arg\min_{u}\mathcal{L}(u)\).

\subsection{Neural network class}\label{sec: NN class}
A neural network with $\sL\in\bN$ layers, $n_l\in\bN$ many nodes at the $l$-th hidden layer for $l=1,\dots, \sL$, input of dimension $n_0$, output of dimension $n_{L+1}$ and nonlinear activation function ReLU $\rho:\R\to\R$ is expressed as
    \begin{equation}
    \label{eq:nn}
        \sN_\rho(\x|\btheta):=\sA_{\sL+1}\circ\sigma_{\sL}\circ\sA_{\sL}\circ\cdots \circ\sigma_1\circ\sA_1(\x),
    \end{equation}
where $\sA_l:\R^{n_{l-1}}\to \R^{n_l}$ is an affine linear map defined by $\sA_l(\x)=\W_l\x+\b_l$ for given  $n_l\times n_{l-1}$ dimensional weight matrix $\W_l$ and $n_l$ dimensional bias vector $\b_l$ and $\sigma_l:\R^{n_l}\to\R^{n_l}$ is an element-wise nonlinear activation map defined by $\sigma_l(\z):=(\sigma(z_1),\dots, \sigma(z_{n_l}))^\top$. We use $\btheta$ to denote the set of all weight matrices and bias vectors  $\btheta:=\del[1]{(\W_1,\b_1),(\W_2,\b_2),\dots, (\W_{L+1}, \b_{L+1})}$. 

\
\\
Following a standard convention, we say that $\sL(\btheta)$ is the {depth} of the deep neural network and $n_{\max}(\btheta)$ is the {width}. We let $\abs{\btheta}_0$ be the number of nonzero elements of $\btheta$, i.e.,
      \begin{equation*}
         \abs{\btheta}_0:=\sum_{l=1}^{\sL+1}\left( \abs{\text{vec}(\W_l)}_0 +\abs{\b_l}_0\right),
      \end{equation*}
where $\text{vec}(\W_l)$ transforms the matrix $\W_l$ into the corresponding vector by concatenating the column vectors. We call $\abs{\btheta}_0$ {sparsity} of the deep neural network. Let $\abs{\btheta}_\infty$ be the largest absolute value of elements of $\btheta$, i.e., 
    \begin{equation*}
        \abs{\btheta}_\infty :=\max \left\{ \max_{1\le l\le \sL+1} \abs{\text{vec}(\W_l)}_\infty, 
\max_{1\le l\le L+1} \abs{\b_l}_\infty\right\}.
    \end{equation*}
We denote by $\Theta_{\sd,\so}(\sL,\sW, \sS, \sB)$ the set of network parameters with depth $\sL$, width $\sW$, sparsity $\sS$, absolute value $\sB$, input dimension $\sd$ and output dimension $\so$,  that is,
\begin{equation}\label{eq: NN class}
    \begin{aligned}
        \Theta_{\sd,\so}(\sL, \sW, \sS, \sB):= \Big\{\btheta:\sL(\btheta)\le \sL, n_{\max}(\btheta)\le \sW, 
        \\
        \abs{\btheta}_0\le \sS,\abs{\btheta}_\infty \le \sB, \textsf{in}(\btheta)=\sd, \textsf{out}(\btheta)=\so\Big\}.
    \end{aligned}
\end{equation}   
\subsection{Estimation and sampling}
Denote by \(\cD := \cD_1 \cup \cD_0\) the full collection of samples used for training, where
\(\cD_1=\{X_{1,j}\}_{j=1}^n\) consists of i.i.d.\ observations \(X_{1,j}\sim \bpi_1\), and
\(\cD_0=\{X_{0,j}\}_{j=1}^n\) consists of i.i.d.\ samples generated from $\bpi_0$ (since \(\bpi_0\) is known) and are independent of \(\mathcal{D}_1\). 

Let \(\{t_k\}_{k=0}^{\sK}\) be a strictly \emph{decreasing} time grid with \(t_0=1\) and
\(t_{\sK}=\tbar>0\). For each \(j\in[n]\) and \(t\in[0,1]\), denote the linear interpolation \(X_{t,j} = tX_{1,j} + (1-t)X_{0,j}\). We estimate the velocity field by empirical risk minimization:
\begin{equation}\label{eq: opt emp}
    \begin{aligned}
        &\widehat v \in \arg\min_{u\in\mathcal U}\ \widehat{\mathcal L}(u),
        \\
        &\what{\cL}(u) := \frac{1}{n}\sum_{j=1}^n \!  \int_{0}^{1-\tbar} \!\!\!\!\!\!\! \lrnorm{u(X_{t,j},t) - (X_{1,j}-X_{0,j})}_2^2 dt .
    \end{aligned}
\end{equation}
We take $\cU$ to be the class of deep neural networks
\begin{equation}\label{eq: search class}
    \begin{aligned}
         \cU =& \Bigg\{u \! = \! \sum_{k=1}^\sK \! u_k(\x,t) \! \cdot \! \one_{\lrm{1 - t_{k-1} \le t < 1 - t_k}}: 
         \\
         &u_k(\x,t) = \sN_\rho(\x,t|\btheta_k), \btheta_k  \in \Theta_{\sd,\sd}(\sL_k, \sW_k, \sS_k, \sB_k) \Bigg\}.
    \end{aligned}
\end{equation}
Each \(u \in \mathcal U\) is assumed to satisfy the following uniform constraints for all \(t\in[0,1-\bar t]\)
$$
\lrnorm{u(\cdot,t)}_\infty \lesssim \frac{\sqrt{\log(n)}}{1 - t}, \quad { \mu_2\!\left(\frac{\partial u}{\partial\,\cdot}(\cdot,t)\right) \le \frac{\sC_{\mathrm{Lip}}}{(1 - t)^{1-\xi}} },
$$
$t \mapsto u(\x,t) \textnormal{ is continuous}$, for some constant \(\sC_{\mathrm{Lip}}>0\). These constraints hold for the true velocity field \(v^\star\), as shown in the next section, and are therefore not merely artifacts of our analysis. They ensure that the candidate functions adhere to the desired regularity conditions. Once the velocity field is estimated, the flow-matching sampler is defined by the neural ODE
\begin{equation}\label{eq: neural ode}
\frac{d \widehat{X}_t}{dt} = \widehat{v}(\widehat{X}_t, t),
\qquad \hat{X}_0 \sim \bpi_0,
\qquad t \in [0,1- \tbar
].
\end{equation}
Since \(\bpi_0\) is easy to sample from, we generate samples by drawing \(\what{X}_0\sim\bpi_0\) and pushing them forward through \eqref{eq: neural ode} using a numerical ODE solver. In what follows, we study the statistical consistency of \(\what v\) and of the induced pushforward density \(\what{\bpi}_{1-\tbar}\) of \(\widehat{X}_{1-\tbar}\).


{\paragraph{Regularity.}
A standard sufficient condition for existence and uniqueness of 
solutions to the ODE \eqref{eq: FM} is given by the 
Picard-Lindel\"of theorem. In particular, suppose the velocity 
field \(v^\star:\mathbb{R}^{\sD} \times [0,1) \to\mathbb{R}^{\sD}\) 
satisfies:
\begin{itemize}[leftmargin=*,itemsep=0pt]
    \item \textbf{Lipschitz continuity in $\x$:} for each 
    $t \in [0,1-\tbar)$, there exists $L_t > 0$ such that for all 
    $\x, \y \in \mathbb{R}^{\sD}$,
    \[
    |v^\star(\x,t)-v^\star(\y,t)|_\infty \leq L_t\, 
    |\x-\y|_\infty;
    \]
    \item \textbf{Continuity in $t$:} The map $t\mapsto v^\star(\x,t)$ 
    is continuous for every fixed $\x$;
\end{itemize}
then there exists a unique solution $X_t$ to \eqref{eq: FM} on $[0,1)$ (see, e.g., \citet{coddington1955theory}). The Lipschitz constant $L_t$ is allowed to depend on $t$ and may diverge as $t \to 1$; this is handled in our framework through early stopping at $t = 1 - \bar{t}$.
}

Note that for a solution to exist, the minimizer must exhibit well-behaved properties-specifically, it should be Lipschitz in space and continuous in time. We enforce these properties is by restricting the search space \(\mathcal{U}\), ensuring that the candidate functions adhere to the desired regularity conditions.



\section{Theoretical results}
In this section, we state our main statistical consistency results for velocity-field estimation, which in turn yield error bounds for implicit density estimation via flow matching.

We work in an ambient space \(\bR^{\sD}\), while the data concentrate on a \(\sd\)-dimensional embedded manifold \(\cM\subset\bR^{\sD}\) with \(\sd\ll \sD\). For \(\y\in\cM\), let \(\sT_{\y}(\cM)\subset\bR^{\sD}\) denote the tangent space at \(\y\), and let \(\proj_{\sT_{\y}(\cM)}\) be the orthogonal projection onto \(\sT_{\y}(\cM)\). We write \(\mathrm{Vol}_{\cM}\) for the \(\sd\)-dimensional
volume measure on \(\cM\) induced by the embedding. Whenever we refer to a ``density'' \(\bpi_1\) on \(\cM\), it is understood as a Radon--Nikodym derivative with respect to \(\mathrm{Vol}_{\cM}\).

\paragraph{Smooth manifold.}
We quantify the regularity of \(\cM\) via local charts induced by tangent projections. Fix \(\beta>0\). We say that \(\cM\) is \(\beta\)-smooth if there exist constants \(r_0>0\) and \(L>0\) such that for every \(\y\in\cM\), the tangent-projection map
\[
\Phi_{\y}:\cM\to\sT_{\y}(\cM),\qquad \Phi_{\y}(\x):=\proj_{\sT_y(\cM)}(\x-\y),
\]
is a local diffeomorphism in a neighborhood of \(\y\), with inverse chart \(\Psi_{\y}\) defined on \(\bB_{r_0}(\bm{0}_\sD)\cap\sT_{\y}(\cM)\). Moreover, the inverse chart \(\Psi_{\y}\) is \(\beta\)-H\"older smooth with H\"older norm bounded by \(L\), uniformly over \(\y\in\cM\).

\begin{assumption}\label{assume: distribution}
The target distribution admits a density $\bpi_1$ (with respect to the \(\sd\)-dimensional volume measure on \(\cM\)) supported on a \(\sd\)-dimensional manifold
\(\cM \subset [-\sC_\cM,\sC_\cM]^{\sD}\) embedded in \(\mathbb{R}^{\sD}\).
The manifold \(\cM\) is compact and without boundary. Moreover, \(\cM\) is \(\beta\)-smooth for some \(\beta\ge 2\), and has reach bounded below by a positive constant.
\end{assumption}

\begin{assumption}\label{assume: Holder}
    The density $\bpi_1$ relative to the volume measure
    of $\cM$ is $\alpha$--H\"older smooth with $\alpha \in [0, \beta -1]$, and is uniformly bounded away from zero on $\cM$.
\end{assumption}


{

\begin{assumption}[One-sided Lipschitz regularity]\label{assume: Lipschitz}
There exist constants $\xi\in(0,1)$ and ${\mathtt{L}_\star}>0$ such that the true velocity field $v^\star(\x,t)$ satisfies
\begin{equation}\nonumber
    \mu_2\!\left(\frac{\partial v^\star}{\partial\bx}(\bx,t)\right)
    \le\frac{\mathtt{L}_\star}{(1-t)^{1-\xi}}, \quad\forall\bx\in\R^\sD,t\in[0,1-\tbar),
\end{equation}
where the \emph{logarithmic norm} $\mu_2(\cdot)$ is defined in \eqref{eq:lognorm} and studied in \Cref{sec:app_lognorm}.
\end{assumption}
}

Assumption~\ref{assume: distribution} formalizes the low intrinsic-dimensional structure of the target distribution.
The \(\beta\)-smoothness controls the regularity of \(\cM\) (e.g., via local chart/projection representations), while the positive reach ensures the associated local projection maps are well-defined in a tubular neighborhood of \(\cM\). 
Assumption~\ref{assume: Holder} enforces both smoothness and non-degeneracy of the target distribution along \(\cM\). The restriction \(\alpha\le \beta-1\) aligns the regularity of \(\bpi_1\) with the geometric smoothness of \(\cM\), ensuring that the density is well-defined and stable under local projection representations. Similar assumptions are standard in manifold-based analyses of generative modeling; see, e.g., \cite{tang2024adaptivity} and \cite{judith2024convergence}.
Assumption~\ref{assume: Lipschitz} is primarily technical: it provides the stability needed to utilize \Cref{thm: W under switching} and transfer velocity-field estimation rates to density error bounds efficiently. In the absence of such a condition, existing analyses can incur a worse dependence on the terminal time, scaling as \((1-t)^{-3}\) \citep{gao2024convergence, zhouerror}. Similar Lipschitz-in-space assumptions (with time-dependent constants) have also been adopted in the ambient-space setting without manifold structure \citep{fukumizu2024flow}.

%

{

Assumption~\ref{assume: Lipschitz} provides the stability needed to utilize the ODE error bounds and transfer velocity-field estimation rates to density error bounds efficiently.  We now show that it is satisfied by a broad and natural class of target measures on manifolds.

The key condition is \emph{semi-convexity} of the log-density.  Write $\bpi_1(\by)=e^{-V(\by)}/Z$ with $V:\cM\to\R$.  Let $\mathrm{Hess}_\cM V$ denote the Riemannian Hessian of $V$ on $(\cM,\mathfrak{g})$ (i.e., the intrinsic second-order derivative; in normal coordinates at $\by_0\in\cM$, it coincides with the matrix of second partial derivatives).  We say $V$ is $M$-\emph{semi-convex} for $M\ge 0$ if
\begin{equation}\label{eq:semiconvex_main}
    \mathrm{Hess}_\cM V(\by)
    \;\succeq\;
    -M\,\mathfrak{g}(\by),
    \qquad\forall\;\by\in\cM.
\end{equation}
In words, the potential $V$ may be non-convex, but its non-convexity is controlled: the most negative eigenvalue of $\mathrm{Hess}_\cM V$ is bounded below by $-M$.  The case $M=0$ corresponds to geodesic convexity of $V$, i.e., \emph{log-concavity} of $\bpi_1$.  A formal treatment, including the Riemannian definitions, is given in \Cref{sec:app_classes}.

Under semi-convexity, Assumption~\ref{assume: Lipschitz} holds with a \emph{uniformly bounded} logarithmic norm.  More precisely (see \Cref{cor:semiconvex_osl} for a formal statement): if $V$ is $M$-semi-convex, then
\begin{equation}\label{eq:mu2_main_informal}
    \mu_2\!\left(\frac{\partial v_*}{\partial\bx}(\bx,t)\right)
    \le
    C(M,\sC_\cM,\sD),
    \quad\forall\bx\in\R^\sD,\;t\in[0,1),
\end{equation}
where $C$ is a finite constant depending on the semi-convexity constant $M$, the manifold diameter $\sC_\cM$, and the ambient dimension~$\sD$.  The mechanism is as follows: the Gaussian tilting in the posterior $p_t(\by\mid\bx)\propto\bpi_1(\by)\,e^{-\|\bx-t\by\|^2/(2(1-t)^2)}$ contributes $+t^2/(1-t)^2$ to the Riemannian Hessian of the posterior potential, which overwhelms the $-M$ non-convexity of $V$ for $t$ sufficiently close to~$1$.  A Brascamp--Lieb argument then controls the tangential posterior covariance.

The semi-convex class encompasses a rich family of distributions on manifolds:
\begin{enumerate}[label=\textnormal{(\alph*)},leftmargin=*,itemsep=4pt]
    \item \emph{Log-concave densities on $\cM$} ($M=0$):  these are densities of the form $\bpi_1\propto e^{-V}$ where $V$ is geodesically convex, i.e., $V(\gamma(s))\le(1-s)V(\gamma(0))+sV(\gamma(1))$ along every minimizing geodesic.  Examples include the \emph{uniform density} ($V\equiv\mathrm{const}$), \emph{von~Mises--Fisher distributions} on $S^\sd$ (with $\bpi_1(\by)\propto e^{\kappa\langle\bm\mu,\by\rangle}$), and \emph{projected Gaussians} on $S^\sd$ (used in our numerical experiments).

    \item \emph{$C^2$ densities bounded away from zero on compact $\cM$}:  if $\bpi_1\in C^2(\cM)$ with $\bpi_1\ge c_0>0$, then $V=-\log\bpi_1\in C^2(\cM)$, and compactness of $\cM$ ensures $M<\infty$ automatically (\Cref{prop:semiconvex}).  No convexity of $V$ is required.  This covers Assumptions~3.1 and~3.2 when $\alpha\ge 2$.
\end{enumerate}
}

We now state the convergence rate for the estimated velocity field obtained in \eqref{eq: opt emp}.
\begin{theorem}[Velocity field estimation]\label{thm: vel field estimation}
    Let $\sd \ge 3$. Suppose $\{t_k\}$ is time grid as follows
    \begin{equation}\label{eq: time seq}
        \begin{aligned}
            1=  t_0 > t_1 > \cdots > \st_\sb = n^{-\frac{2}{2\alpha+\sd}} > \cdots > t_{\sK} = 
            \\\tbar = n^{-\frac{\beta}{2\alpha + \sd}}\log^{\beta+1}(n),\quad 1 < \frac{t_k}{t_{k+1}} \le  2
        \end{aligned}
    \end{equation}
    for $ k = 0, 1, \ldots, \sK$. Let $\what{v}(\x,t)$ be estimated velocity field obtained with the empirical optimization as in \eqref{eq: opt emp}. Under the Assumptions \ref{assume: distribution}, \ref{assume: Holder}, and \ref{assume: Lipschitz}, we have: 
    \begin{enumerate}[label=\Alph*.,leftmargin=*,itemsep=0pt]   
        \item \label{item: vf estim t zero} for $ n^{-\frac{\beta}{2\alpha + \sd}}\log^{\beta}(n) \le \st_k < n^{-\frac{2}{2\alpha + \sd}}$, 
        \begin{equation*}
            \begin{aligned}
                &\bE_{\mathcal{D}}\left[ \int_{\x}\int_{t=1 - t_{k}}^{1 - t_{k+1}} \norm{\widehat{v}(\x,t) - v^\star(\x,t)}^2 \bpi_t(x) \, dt \, d\x \right] 
                \\
                \le & \, C\lrs{ \frac{n^{-\frac{2\beta}{2\alpha + \sd}}}{t_k} + n^{-\frac{2\alpha}{2\alpha + \sd}}\cdot \log^{\alpha+1}(n) + \frac{\log^2(n)}{n}},
            \end{aligned}
        \end{equation*}
        where the neural network parameters satisfies
        \begin{equation*}
            \begin{aligned}
                &\sL_k = \cO\lrs{\log^4(n)}, \sW_k = \cO\lrs{n^{\frac{\sd}{2\alpha + \sd}} \log^{{\lrs{6\vee3+\sd}}}(n)},
                \\
                &\sS_k = \cO\lrs{n^{\frac{\sd}{2\alpha + \sd}} \log^{\lrs{{8\vee (5+\sd)}}}(n)}, \sB_k = e^{\cO\lrs{\log^4(n)}}. 
            \end{aligned}
        \end{equation*}

        \item\label{item: vf estim t away zero0} for $ n^{-\frac{2}{2\alpha + \sd}} \le t_k < n^{-\frac{1}{6(2\alpha + \sd)}} \log^{-3}(n)$,
        \begin{equation*}
            \begin{aligned}
                &\bE_{\mathcal{D}}\left[ \int_{\x}\int_{t=1 - t_{k}}^{1 - t_{k+1}} \norm{\widehat{v}(\x,t) - v^\star(\x,t)}^2 \bpi_t(x) \, dt \, d\x \right]   
                \\
                \le & \; C \lrs{ \frac{\log^4(n)}{n} + \frac{t_k^{-\sd/2}}{n} \cdot \log^{14+\sd/2}(n)},
            \end{aligned}
        \end{equation*}
        where the neural network parameters satisfies
        \begin{equation*}
            \begin{aligned}
                &\!\!\! \sL_k \!=\! \cO(\log^4(n)), \sW_k \!=\! \cO\lrs{t_k^{-\sd/2}\log^{(6\vee (\sd+3)-\sd/2)}(n)}\!\!,
                \\
                &\!\!\! \sS_k \!=\! \cO\lrs{t_k^{-\sd/2}\log^{(8\vee (\sd+5)-\sd/2)}(n)}\!, \sB_k \!=\! e^{\cO\lrs{\log^4(n)}}.
            \end{aligned}
        \end{equation*}        

        \item \label{item: vf estim t away zero} for $ n^{-\frac{1}{6(2\alpha + \sd)}} \log^{-3}(n) \le t_k < 1$,
        \begin{equation*}
            \begin{aligned}
                &\bE_{\mathcal{D}}\left[ \int_{\x}\int_{t=1 - t_{k}}^{1 - t_{k+1}} \norm{\widehat{v}(\x,t) - v^\star(\x,t)}^2 \bpi_t(x) \, dt \, d\x \right] 
                \\
                \le & \, C \lrs{ \frac{\log^5(n)}{n} + n^{-\frac{(2\alpha + 2)}{2\alpha+\sd}} \cdot \log^{2\sd + 9}(n)},
            \end{aligned}
        \end{equation*}        
        where the neural network parameters satisfies
        \begin{equation*}
            \begin{aligned}
                &\sL_k = \cO\lrs{\log^2(n)}\!, \sW_k = \cO\lrs{n^{\frac{\sd}{6(2\alpha + \sd)}} \log^{{{2\sd+3}}}(n)}\!,
                \\
                &\sS_k = \cO\lrs{n^{\frac{\sd}{6(2\alpha + \sd)}} \log^{{2\sd+4}}(n)}, \sB_k = e^{\cO\lrs{\log^4(n)}}.
            \end{aligned}
        \end{equation*}
    \end{enumerate}
    Here $C > 0$ is a constant depending on $\sD$, $\sC_\cM$ and $\beta$.
\end{theorem}

The proof of \Cref{thm: vel field estimation} is provided in \Cref{sec: vf estimation}.  At a high level, the argument decomposes the estimation error into a bias term and a variance term. The bias is controlled via the neural-network approximation result in \Cref{cor: score approximation no delta}, while the variance is bounded using a uniform bound based on the covering numbers of the loss function class in \Cref{lemma: loss cover}. These ingredients are then combined through the M-estimation result in \Cref{lemma: emp process bound2} to conclude the claim. As one can see, the rates are dependent on the intrinsic dimension $\sd$ instead of the ambient dimension $\sD$.

\begin{table*}[!t]
\centering
\caption{Comparison with existing theoretical results for flow matching.}
\label{tab: comparison}

\resizebox{\linewidth}{!}{
\begin{tabular}{c|c|c|c|c|c}
\toprule
& \shortstack{Key assumptions}
& \shortstack{Low-dimensional \\ structure}
& \shortstack{Velocity field\\estimation}
& \shortstack{Optimality}
& \shortstack{Metric}\\[2pt]
\hline
\citet{albergo2022building} & \shortstack{$\x \mapsto \what{v}(\x,t)$ is $\what{K}$-Lipschitz} & \xmark & \xmark & \xmark & $\rmW_2$ \\ [4pt]
\hline
\citet{fukumizu2024flow} & \shortstack{Bounded support \\ $\x \mapsto {v^\star}(\x,t)$ is differentiable \\ with $\|\nabla_\x v^\star(\x,t)\|_\op \lesssim \tfrac{1}{1-t}$ }                               & \xmark & \cmark & \cmark & $\rmW_2$ \\ [4pt]
\hline
\citet{gao2024convergence} & \shortstack{Log-concave and \\ Gaussian mixture targets \\ $\x \mapsto {v^\star}(\x,t)$ is $L_t$-Lipschitz }                               & \xmark & \cmark & \xmark & $\rmW_2$ \\ [4pt]
\hline
\citet{zhouerror} & \shortstack{Bounded support \\ Lipschitz score function }                               & \xmark & \cmark & \xmark & $\rmW_2$ \\ [4pt]

\hline
\citet{kunkel2025minimax} & \shortstack{Bounded support \\ $\x \mapsto {v^\star}(\x,t)$ is Lipschitz \\ continuous  }                               & \shortstack{\cmark \\ {single-chart manifold}\\{projected $\rmW_1$}} & \shortstack{\cmark \\ ({exponential size} \\ network)} & \cmark & $\rmW_1$ \\ [4pt]

\hline
Ours & \shortstack{Bounded support \\ $ {\mu_2\!\left({\partial_\x v^\star}(\bx,t)\right)} \lesssim \tfrac{1}{(1-t)^{1-\xi}}$ \\ $\xi \approx (\log\log(n))^{-1}$  }                               & \shortstack{\cmark \\ {(general manifold)}} & \cmark & \cmark & $\rmW_2$ \\ [4pt]
\bottomrule
\end{tabular}
}
\end{table*}

Our results rely on a carefully designed fixed time grid that reflects the non-uniform difficulty of learning the velocity field in \eqref{eq: VF}. In particular, the estimation problem becomes progressively harder as \(t\to 1\), mirroring the singular behavior of \(v^\star(\cdot,t)\) near the terminal time. We therefore refine the grid close to \(t=1\) and employ early stopping at \(t=1-\tbar \) to avoid the endpoint singularity. On each intermediate time slab, the appropriate network architecture, and the resulting estimation rate, depends on the local temporal resolution, quantified by the time grid width \( t_k - t_{k+1} = \cO(t_k) \). By contrast, at times away from from \(t=1\), the estimation error is essentially insensitive to \(t_k\), and the network parameters can be chosen as a function of \(n\) alone. Extending the analysis to random time-grid designs, which are commonly used in practice  \citep{lipman2022flow}, would substantially complicate the proof structure; we therefore leave a systematic treatment of such grids to future work.


\begin{theorem}[Main result] \label{thm: main}
    Let $\sd \ge 3$. Suppose $\what{\bpi}_{1-\tbar}$ denotes the density of $\what{X}_{1-\tbar}$ as in \eqref{eq: neural ode}. Under the Assumptions \ref{assume: distribution}, \ref{assume: Holder}, and \ref{assume: Lipschitz}, and the setup of \Cref{thm: vel field estimation}, assume $1 > \xi \ge \sC_{\mathrm{Lip}}/\log\log(n)$. Then 
    \begin{equation*}
        \begin{aligned}
            \bE_{\cD}\lrl{\rmW_2\lrs{\what{\pi}_{1-\tbar}, \bpi_1}} \le C \Big(  n^{-\frac{\beta}{2\alpha + \sd}} &\log^{\beta\vee2}(n) +
            \\
             n^{-\frac{\alpha + 1}{2\alpha + \sd}}\log^{\sd + 9}(n) \, + \, & n^{-1/2}  \log^4(n) \Big), 
        \end{aligned}
    \end{equation*}
    where $C > 0$ is a constant independent of \(n\) (depending only on \(\sD\), \(\sC_{\cM}\), and \(\beta\)).
\end{theorem}
The proof of \Cref{thm: main} is provided in \Cref{sec: thm main proof}. It is based on the error decomposition in \Cref{lemma: error accumulation}, which separates (i) the early-stopping error and (ii) an accumulated estimation error obtained by summing the velocity-field estimation error over the time-grid, weighted by the corresponding grid lengths. The early-stopping term is bounded in \Cref{lemma: Early stopping}, while the accumulated estimation term is controlled using \Cref{thm: vel field estimation}.

\Cref{thm: main} shows that flow matching with linear interpolation adapts to the (unknown) manifold structure underlying the data. The resulting convergence rate, up to $\log$ factors,  decomposes into three terms, \(n^{-\beta/(2\alpha+\sd)}\), \(n^{-(\alpha+1)/(2\alpha+\sd)}, \textnormal{ and } n^{-1/2}\). The second term matches the classical rate for density estimation on a \(\sd\)-dimensional manifold, whereas the first term captures an additional contribution that couples support (manifold) estimation with density estimation. In contrast, the minimax lower bound for this problem is \(n^{-\beta/\sd} + n^{-(\alpha+1)/(2\alpha+\sd)} + n^{-1/2}\) \citep[Theorem 1]{yunyangmanifold}. The first component corresponds to pure manifold recovery, while the second corresponds to density estimation given the manifold.

Our upper bound is therefore near-optimal: it recovers the density-estimation term exactly, and it is minimax optimal in regimes where this term dominates the overall error. The remaining gap lies in the support-estimation component: \(n^{-\beta/(2\alpha+\sd)}\) is slower than the optimal manifold-estimation rate \(n^{-\beta/\sd}\) \citep{amarimanifold, divol2022measure}. We conjecture that this discrepancy is driven by the interpolation-based training objective, which introduces additional statistical difficulty in the near-terminal (singular) time regime; related methods such as diffusion display similar rate degradations \citep{judith2024convergence,tang2024adaptivity}.

We compare our work  with prior results on flow matching in Table~\ref{tab: comparison}. A key distinction is the regularity imposed on the velocity field $v^\star$. In particular, the assumption that \(x \mapsto v^\star(x,t)\) is \(L\)-Lipschitz with \(L \lesssim 1\) is quite restrictive, as it effectively narrows the admissible class of target distributions. For instance, the analysis of \citet{gao2024convergence} applies primarily to log-concave $\bpi_1$ and closely related families, including certain near-Gaussian variants. Although \citet{kunkel2025minimax} remove the global Lipschitz requirement, their guarantees still rely on a vanilla KDE that adapts to the target ambient-space density. This assumption breaks down when \(\bpi_1\) is singular and is supported on an unknown low-dimensional manifold \citep{ozakin2009submanifold}.

\section{Numerical results}

We present numerical experiments across { two} synthetic data settings to validate the theoretical results on the manifold adaptivity of flow matching. In { both } cases the target law $\bpi_1$ is supported on a smooth, low-dimensional manifold $\cM\subset\R^{\sD}$ with intrinsic dimension $\sd\ll\sD$, while the source $\bpi_0$ is a standard Gaussian on $\R^{\sD}$. { Section~\ref{sec:add_num} of the appendix provides additional experiments, including a real data example (MNIST), ablation studies examining the dependence of the convergence rate on $n$, and an illustrative floral manifold example.}

\subsection{Example target distributions}

We present numerical results of flow matching on the following { two} example target distributions.

\begin{example}[Sphere embedded in high dimension]\label{ex: sphere}
Fix an intrinsic dimension $\sd\ge 2$ and define the manifold
\[
\cM \;=\; \mathbb{S}^{\sd} \times \{0\}^{\sD-(\sd+1)} \;\subset\;\R^{\sD},
\]
i.e., the unit $\sd$-sphere embedded in the first $\sd+1$ coordinates and padded with zeros in the remaining coordinates.

\textbf{Target distribution $\bpi_1$.} We use a smooth, non-uniform distribution on the sphere via a \emph{projected Gaussian}  Sample
\[
Z \sim \mathtt{N}\lrs{ \bgamma ,\bI_{\sd+1}},\qquad
Y := \frac{Z}{\|Z\|_2}\in\mathbb{S}^{\sd},
\]
and finally embed into $\R^{\sD}$ by padding
\(
X_1 := (Y,0,\dots,0)\in\R^{\sD}.
\)
\end{example}

\begin{example}[Rotated $\sd$-torus embedded in $\bR^\sD$]\label{ex: torus}
Define the \emph{axis-aligned} $d$-torus embedding in $\R^D$ by
\[
\cM_0 = \lrm{(\cos\theta_1,\sin\theta_1,\ldots,\cos\theta_d,\sin\theta_d,0,\ldots,0)\in\R^{\sD}},
\]
where $\btheta \in \bR^\sd$, and $\btheta_i = \phi + \gamma_1 \cdot i + \epsilon_i $. Here
$$
\phi \sim \mathrm{Unif}\lrm{-1,1} \qquad \textnormal{and} \qquad \epsilon_i = \mathtt{N}(-\gamma_1, \sigma_1^2).
$$
To remove axis alignment, let $\sO \in \mathbb{O}_{\sD}$ be an arbitrary orthogonal matrix. We define the rotated torus as
$$
\cM = \lrm{\x_0(\btheta) \cdot  \sO^\top : \x_0(\btheta) \in \cM_0}.
$$
\end{example}



\subsection{Implementation details}
\begin{itemize}[leftmargin=*,itemsep=0pt]
    \item \textbf{Sphere. } We set the parameter values $\bgamma = \bm{0}_{\sd+1}$ and consider intrinsic dimensions $ \sd \in \lrm{2, 3, 4, 5}$.  The ambient dimension chosen as $\sD \in \lrm{2\sd, 3\sd, 4\sd}$ for each $\sd$. The velocity field $v$ is parametrized by a multilayer perceptron network with width 256 and depth 4, ReLU activations, and a linear output layer of dimension $\sD$. Training is performed using AdamW with learning rate \(2 \times 10^{-4}\), batch size $2048$, and $1,000$ iterations. For generation, we solve the learned ODE with forward Euler using $N = 250$ steps on the nonuniform grid \(t_i = 1-(1-i/N)^2, \; i=0,\ldots,N\).
    
    \item \textbf{Torus. } In this experiment, we use the parameter values $\gamma_1 = 0.35$ and $\sigma_1^2 = 0.35^2 + 0.15^2$. The choice of $(\sd, \sD)$ is the same as in the previous case. All other training settings remain unchanged, except that the network depth is increased to $6$ instead of $4$.
    
\end{itemize}
\subsection{Evaluations}
We evaluate the quality of the generated samples in \Cref{ex: sphere,ex: torus} using two complementary metrics: (i) the sliced Wasserstein distance \citep{karras2018progressive,kolouri2019generalized}, which measures distributional discrepancy, and (ii) the distance to the manifold, which quantifies geometric fidelity. Specifically, we report the standardized empirical sliced Wasserstein distance (\(\rmW_{1,\mathrm{slice}}^{\mathrm{std}}\)) and an empirical estimate of the manifold distance (\(\mathrm{dist}_{\cM}\)).

For each $(\sd,\sD)$, we repeat evaluation over $R=5$ independent runs and report mean and standard deviation \Cref{table:sphere_agg,table:torus_agg}. Across both the sphere and torus families, $\rmW_{1,\mathrm{slice}}^{\mathrm{std}}$ remains of the same order across ambient dimensions, while $\mathrm{dist}_{\cM}$ stays small, indicating that the learned flow {  accurately recovers the manifold geometry}.


\begin{table}[!h]
\caption{Mean and standard deviation of $\rmW_{1,\mathrm{slice}}^{\mathrm{std}}$ and $\mathrm{dist}_{\cM}$ for estimated density in \Cref{ex: sphere} across $(\sd, \sD)$.}
\begin{center}
\begin{small}
\begin{tabular}{cccc}
\toprule
$\sd$ & $\sD$ & $\rmW_{1,\mathrm{slice}}^{\mathrm{std}}$ & $\mathrm{dist}_{\cM}$ \\
\midrule
    & 4  & 0.04177 $\pm$ 0.01935 & 0.05304 $\pm$ 0.00460 \\
2   & 6  & 0.03788 $\pm$ 0.00725 & 0.05920 $\pm$ 0.00339 \\
    & 8  & 0.04194 $\pm$ 0.01330 & 0.05861 $\pm$ 0.00589 \\
\midrule
    & 6  & 0.03277 $\pm$ 0.00573 & 0.07028 $\pm$ 0.00140 \\
3   & 9  & 0.03994 $\pm$ 0.00997 & 0.06962 $\pm$ 0.00622 \\
    & 12 & 0.04648 $\pm$ 0.01795 & 0.07906 $\pm$ 0.00353 \\
\midrule
    & 8  & 0.03084 $\pm$ 0.00732 & 0.07861 $\pm$ 0.00261 \\
4   & 12 & 0.04097 $\pm$ 0.00590 & 0.07979 $\pm$ 0.00180 \\
    & 16 & 0.05370 $\pm$ 0.01152 & 0.10544 $\pm$ 0.00294 \\
\midrule
    & 10 & 0.03768 $\pm$ 0.00901 & 0.08246 $\pm$ 0.00226 \\
5   & 15 & 0.04473 $\pm$ 0.00780 & 0.10290 $\pm$ 0.00450 \\
    & 20 & 0.05324 $\pm$ 0.00657 & 0.14034 $\pm$ 0.00131 \\
\bottomrule
\end{tabular}
\label{table:sphere_agg}
\end{small}
\end{center}
\end{table}

\begin{table}[!h]
\caption{Mean and standard deviation of $\rmW_{1,\mathrm{slice}}^{\mathrm{std}}$ and $\mathrm{dist}_{\cM}$ for estimated density in \Cref{ex: torus} across $(\sd, \sD)$.}
\begin{center}
\begin{small}
\begin{tabular}{cccc}
\toprule
$\sd$ & $\sD$ & $\rmW_{1,\mathrm{slice}}^{\mathrm{std}}$ & $\mathrm{dist}_{\cM}$ \\
\midrule
    & 4  & 0.04407 $\pm$ 0.01421 & 0.05022 $\pm$ 0.00291 \\
2   & 6  & 0.02430 $\pm$ 0.00407 & 0.06331 $\pm$ 0.00212 \\
    & 8  & 0.03548 $\pm$ 0.01123 & 0.07248 $\pm$ 0.00218 \\
\midrule
    & 6  & 0.02998 $\pm$ 0.01201 & 0.07268 $\pm$ 0.00218 \\
3   & 9  & 0.03790 $\pm$ 0.01672 & 0.09524 $\pm$ 0.00215 \\
    & 12 & 0.03792 $\pm$ 0.00951 & 0.11211 $\pm$ 0.00311 \\
\midrule
    & 8  & 0.02833 $\pm$ 0.01532 & 0.10091 $\pm$ 0.00330 \\
4   & 12 & 0.02788 $\pm$ 0.00577 & 0.13726 $\pm$ 0.00369 \\
    & 16 & 0.03859 $\pm$ 0.01159 & 0.17358 $\pm$ 0.00611 \\
\midrule
    & 10 & 0.03017 $\pm$ 0.01236 & 0.13094 $\pm$ 0.00347 \\
5   & 15 & 0.03492 $\pm$ 0.00572 & 0.18492 $\pm$ 0.00265 \\
    & 20 & 0.04205 $\pm$ 0.01155 & 0.23003 $\pm$ 0.00515 \\
\bottomrule
\end{tabular}
\label{table:torus_agg}
\end{small}
\end{center}
\end{table}

\section{Discussion}
We study the theoretical properties of flow matching with the linear interpolation path when the target distribution is supported on a low-dimensional manifold. We show that the convergence rate of the resulting implicit density estimator is governed by the manifold’s intrinsic dimension (rather than the ambient dimension). These results lay the statistical foundations of flow-matching based models by providing a principled explanation for why linear-path flow matching can mitigate the curse of dimensionality by adapting to the intrinsic geometry of the data.
\paragraph{Future work.} There are several interesting future directions to pursue: (i) Extend our theory to the more realistic setting where data are concentrated near a low-dimensional manifold. For instance, when observations are corrupted by small, decaying noise around a manifold-supported distribution. In this regime, we expect the early-stopping requirement may be removable and the regularity of the velocity field may improve, since the singular behavior near \(t=1\) should be smoothed out. (ii) Investigate stratified settings in which the target distribution lies on a union of disjoint manifolds, as suggested by the floral example. It would be interesting to characterize the resulting regularity properties and to derive estimation rates for both the velocity field and the induced implicit density estimator, and (iii) another interesting direction is to employ flow based models for conditional distribution estimation or distribution regression where one incorporates additional covariates or control information in modeling the underlying distribution. 



\bibliography{references.bib}

@article{suzuki023diffusion,
  title={Diffusion models are minimax optimal distribution estimators},
  author={Oko, Kazusato and Akiyama, Shunta and Suzuki, Taiji},
  journal={arXiv preprint arXiv:2303.01861},
  year={2023}
}

@book{coddington1955theory,
  title={Theory of Ordinary Differential Equations},
  author={Coddington, Earl A. and Levinson, Norman},
  year={1955},
  publisher={McGraw-Hill}
}

@article{albergo2023stochastic,
  title={Stochastic interpolants: A unifying framework for flows and diffusions},
  author={Albergo, Michael S and Boffi, Nicholas M and Vanden-Eijnden, Eric},
  journal={arXiv preprint arXiv:2303.08797},
  year={2023}
}

@article{schmidt2017nonparametric,
  author = {Schmidt-Hieber, J.},
  title = {Nonparametric regression using deep neural networks with relu activation function},
  journal = {arXiv preprint arXiv:1708.06633},
  year = {2017}
}

@article{suzuki2018adaptivity,
  author = {Suzuki, T.},
  title = {Adaptivity of deep relu network for learning in besov and mixed smooth besov spaces: optimal rate and curse of dimensionality},
  journal = {arXiv preprint arXiv:1810.08033},
  year = {2018}
}

@article{liu2022flow,
  title={Flow straight and fast: Learning to generate and transfer data with rectified flow},
  author={Liu, Xingchao and Gong, Chengyue and Liu, Qiang},
  journal={arXiv preprint arXiv:2209.03003},
  year={2022}
}

@article{albergo2022building,
  title={Building normalizing flows with stochastic interpolants},
  author={Albergo, Michael S and Vanden-Eijnden, Eric},
  journal={arXiv preprint arXiv:2209.15571},
  year={2022}
}

@article{lipman2022flow,
  title={Flow matching for generative modeling},
  author={Lipman, Yaron and Chen, Ricky TQ and Ben-Hamu, Heli and Nickel, Maximilian and Le, Matt},
  journal={arXiv preprint arXiv:2210.02747},
  year={2022}
}

@inproceedings{esser2024scaling,
  title={Scaling rectified flow transformers for high-resolution image synthesis},
  author={Esser, Patrick and Kulal, Sumith and Blattmann, Andreas and Entezari, Rahim and M{\"u}ller, Jonas and Saini, Harry and Levi, Yam and Lorenz, Dominik and Sauer, Axel and Boesel, Frederic and others},
  booktitle={Forty-first international conference on machine learning},
  year={2024}
}

@article{bose2023se,
  title={Se (3)-stochastic flow matching for protein backbone generation},
  author={Bose, Avishek Joey and Akhound-Sadegh, Tara and Huguet, Guillaume and Fatras, Kilian and Rector-Brooks, Jarrid and Liu, Cheng-Hao and Nica, Andrei Cristian and Korablyov, Maksym and Bronstein, Michael and Tong, Alexander},
  journal={arXiv preprint arXiv:2310.02391},
  year={2023}
}

@inproceedings{graham2024proceedings,
  title={Proceedings of the 18th Conference of the European Chapter of the Association for Computational Linguistics (Volume 1: Long Papers)},
  author={Graham, Yvette and Purver, Matthew},
  booktitle={Proceedings of the 18th Conference of the European Chapter of the Association for Computational Linguistics (Volume 1: Long Papers)},
  year={2024}
}

@inproceedings{ma2024sit,
  title={Sit: Exploring flow and diffusion-based generative models with scalable interpolant transformers},
  author={Ma, Nanye and Goldstein, Mark and Albergo, Michael S and Boffi, Nicholas M and Vanden-Eijnden, Eric and Xie, Saining},
  booktitle={European Conference on Computer Vision},
  pages={23--40},
  year={2024},
  organization={Springer}
}

@article{fukumizu2024flow,
  title={Flow matching achieves almost minimax optimal convergence},
  author={Fukumizu, Kenji and Suzuki, Taiji and Isobe, Noboru and Oko, Kazusato and Koyama, Masanori},
  journal={arXiv preprint arXiv:2405.20879},
  year={2024}
}

@article{linsparse,
  author  = {Minwoo Chae and Dongha Kim and Yongdai Kim and Lizhen Lin},
  title   = {A Likelihood Approach to Nonparametric Estimation of a Singular Distribution Using Deep Generative Models},
  journal = {Journal of Machine Learning Research},
  year    = {2023},
  volume  = {24},
  number  = {77},
  pages   = {1--42},
  url     = {http://jmlr.org/papers/v24/21-1099.html}
}

@article{ozakin2009submanifold,
  title={Submanifold density estimation},
  author={Ozakin, Arkadas and Gray, Alexander},
  journal={Advances in neural information processing systems},
  volume={22},
  year={2009}
}

@article{amarimanifold,
author = {Eddie Aamari and Cl{\'e}ment Levrard},
title = {{Nonasymptotic rates for manifold, tangent space and curvature estimation}},
volume = {47},
journal = {The Annals of Statistics},
number = {1},
publisher = {Institute of Mathematical Statistics},
pages = {177 -- 204},
year = {2019},
doi = {10.1214/18-AOS1685},
URL = {https://doi.org/10.1214/18-AOS1685}
}

@article{yunyangmanifold,
author = {Rong Tang and Yun Yang},
title = {{Minimax rate of distribution estimation on unknown submanifolds under adversarial losses}},
volume = {51},
journal = {The Annals of Statistics},
number = {3},
publisher = {Institute of Mathematical Statistics},
pages = {1282 -- 1308},
year = {2023},
doi = {10.1214/23-AOS2291},
URL = {https://doi.org/10.1214/23-AOS2291}
}

@inproceedings{tang2024adaptivity,
  title={Adaptivity of diffusion models to manifold structures},
  author={Tang, Rong and Yang, Yun},
  booktitle={International Conference on Artificial Intelligence and Statistics},
  pages={1648--1656},
  year={2024},
  organization={PMLR}
}

@article{tong2023improving,
  title={Improving and generalizing flow-based generative models with minibatch optimal transport},
  author={Tong, Alexander and Fatras, Kilian and Malkin, Nikolay and Huguet, Guillaume and Zhang, Yanlei and Rector-Brooks, Jarrid and Wolf, Guy and Bengio, Yoshua},
  journal={arXiv preprint arXiv:2302.00482},
  year={2023}
}

@article{judith2024convergence,
  title={Convergence of diffusion models under the manifold hypothesis in high-dimensions},
  author={Azangulov, Iskander and Deligiannidis, George and Rousseau, Judith},
  journal={arXiv preprint arXiv:2409.18804},
  year={2024}
}

@article{gao2024convergence,
  title={Convergence of continuous normalizing flows for learning probability distributions},
  author={Gao, Yuan and Huang, Jian and Jiao, Yuling and Zheng, Shurong},
  journal={arXiv preprint arXiv:2404.00551},
  year={2024}
}

@article{divol2022measure,
  title={Measure estimation on manifolds: an optimal transport approach},
  author={Divol, Vincent},
  journal={Probability Theory and Related Fields},
  volume={183},
  number={1},
  pages={581--647},
  year={2022},
  publisher={Springer}
}

@article{davis2024fisher,
  title={Fisher flow matching for generative modeling over discrete data},
  author={Davis, Oscar and Kessler, Samuel and Petrache, Mircea and Ceylan, {\.I}smail {\.I} and Bronstein, Michael and Bose, Avishek J},
  journal={Advances in Neural Information Processing Systems},
  volume={37},
  pages={139054--139084},
  year={2024}
}

@article{gat2024discrete,
  title={Discrete flow matching},
  author={Gat, Itai and Remez, Tal and Shaul, Neta and Kreuk, Felix and Chen, Ricky TQ and Synnaeve, Gabriel and Adi, Yossi and Lipman, Yaron},
  journal={Advances in Neural Information Processing Systems},
  volume={37},
  pages={133345--133385},
  year={2024}
}

@article{su2025theoretical,
  title={A theoretical analysis of discrete flow matching generative models},
  author={Su, Maojiang and Lu, Mingcheng and Hu, Jerry Yao-Chieh and Wu, Shang and Song, Zhao and Reneau, Alex and Liu, Han},
  journal={arXiv preprint arXiv:2509.22623},
  year={2025}
}

@article{cheng2025alpha,
  title={$\alpha$-Flow: A Unified Framework for Continuous-State Discrete Flow Matching Models},
  author={Cheng, Chaoran and Li, Jiahan and Fan, Jiajun and Liu, Ge},
  journal={arXiv preprint arXiv:2504.10283},
  year={2025}
}

@article{gao2024gaussian,
  title={Gaussian interpolation flows},
  author={Gao, Yuan and Huang, Jian and Jiao, Yuling},
  journal={Journal of Machine Learning Research},
  volume={25},
  number={253},
  pages={1--52},
  year={2024}
}

@phdthesis{Kunkel2025_1000188527,
    author       = {Kunkel, Lea Maria},
    year         = {2025},
    title        = {Statistical Guarantees for Generative Models as Distribution Estimators},
    doi          = {10.5445/IR/1000188527},
    publisher    = {{Karlsruher Institut für Technologie (KIT)}},
    pagetotal    = {187},
    school       = {Karlsruher Institut für Technologie (KIT)},
    language     = {english}
}

@InProceedings{zhouerror,
  title = 	 {An Error Analysis of Flow Matching for Deep Generative Modeling},
  author =       {Zhou, Zhengyu and Liu, Weiwei},
  booktitle = 	 {Proceedings of the 42nd International Conference on Machine Learning},
  pages = 	 {78903-78932},
  year = 	 {2025},
  volume = 	 {267},
  series = 	 {Proceedings of Machine Learning Research},
  month = 	 {13--19 Jul},
  publisher =    {PMLR},
  url = 	 {https://proceedings.mlr.press/v267/zhou25l.html}
}

@article{chen2023flow,
  title={Flow matching on general geometries},
  author={Chen, Ricky TQ and Lipman, Yaron},
  journal={arXiv preprint arXiv:2302.03660},
  year={2023}
}

@inproceedings{hu2024flow,
  title={Flow matching for conditional text generation in a few sampling steps},
  author={Hu, Vincent and Wu, Di and Asano, Yuki and Mettes, Pascal and Fernando, Basura and Ommer, Bj{\"o}rn and Snoek, Cees},
  booktitle={Proceedings of the 18th Conference of the European Chapter of the Association for Computational Linguistics (Volume 2: Short Papers)},
  pages={380--392},
  year={2024}
}

@article{gui2025depthfm, 
    title={DepthFM: Fast Generative Monocular Depth Estimation with Flow Matching}, 
    volume={39}, 
    url={https://ojs.aaai.org/index.php/AAAI/article/view/32330}, 
    DOI={10.1609/aaai.v39i3.32330}, number={3}, 
    journal={Proceedings of the AAAI Conference on Artificial Intelligence}, 
    author={Gui, Ming and Schusterbauer, Johannes and Prestel, Ulrich and Ma, Pingchuan and Kotovenko, Dmytro and Grebenkova, Olga and Baumann, Stefan Andreas and Hu, Vincent Tao and Ommer, Björn}, 
    year={2025}, 
    month={Apr.}, 
    pages={3203-3211} }

@article{bansal2024wasserstein,
  title={On the Wasserstein Convergence and Straightness of Rectified Flow},
  author={Bansal, Vansh and Roy, Saptarshi and Sarkar, Purnamrita and Rinaldo, Alessandro},
  journal={arXiv preprint arXiv:2410.14949},
  year={2024}
}

@article{kornilov2024optimal,
  title={Optimal flow matching: Learning straight trajectories in just one step},
  author={Kornilov, Nikita and Mokrov, Petr and Gasnikov, Alexander and Korotin, Aleksandr},
  journal={Advances in Neural Information Processing Systems},
  volume={37},
  pages={104180--104204},
  year={2024}
}

@article{dao2023flow,
  title={Flow matching in latent space},
  author={Dao, Quan and Phung, Hao and Nguyen, Binh and Tran, Anh},
  journal={arXiv preprint arXiv:2307.08698},
  year={2023}
}

@inproceedings{hu2024latent,
  title={Latent space editing in transformer-based flow matching},
  author={Hu, Vincent Tao and Zhang, Wei and Tang, Meng and Mettes, Pascal and Zhao, Deli and Snoek, Cees},
  booktitle={Proceedings of the AAAI conference on artificial intelligence},
  volume={38},
  pages={2247--2255},
  year={2024}
}

@article{marzouk2024distribution,
  title={Distribution learning via neural differential equations: a nonparametric statistical perspective},
  author={Marzouk, Youssef and Ren, Zhi Robert and Wang, Sven and Zech, Jakob},
  journal={Journal of Machine Learning Research},
  volume={25},
  number={232},
  pages={1--61},
  year={2024}
}

@inproceedings{karras2018progressive,
  title={Progressive Growing of {GAN}s for Improved Quality, Stability, and Variation},
  author={Karras, Tero and Aila, Timo and Laine, Samuli and Lehtinen, Jaakko},
  booktitle={International Conference on Learning Representations (ICLR)},
  year={2018}
}

@article{kolouri2019generalized,
  title={Generalized sliced wasserstein distances},
  author={Kolouri, Soheil and Nadjahi, Kimia and Simsekli, Umut and Badeau, Roland and Rohde, Gustavo},
  journal={Advances in neural information processing systems},
  volume={32},
  year={2019}
}

@article{kunkel2025minimax,
  title={On the minimax optimality of Flow Matching through the connection to kernel density estimation},
  author={Kunkel, Lea and Trabs, Mathias},
  journal={arXiv preprint arXiv:2504.13336},
  year={2025}
}

@article{kunkel2025distribution,
  title={Distribution estimation via Flow Matching with Lipschitz guarantees},
  author={Kunkel, Lea},
  journal={arXiv preprint arXiv:2509.02337},
  year={2025}
}

@article{brascamp1976extensions,
  title={On extensions of the Brunn-Minkowski and Pr{\'e}kopa-Leindler theorems, including inequalities for log concave functions, and with an application to the diffusion equation},
  author={Brascamp, Herm Jan and Lieb, Elliott H},
  journal={Journal of functional analysis},
  volume={22},
  number={4},
  pages={366--389},
  year={1976},
  publisher={Elsevier}
}

@book{ledoux2001concentration,
  author    = {Ledoux, Michel},
  title     = {The Concentration of Measure Phenomenon},
  series    = {Mathematical Surveys and Monographs},
  volume    = {89},
  year      = {2001},
  publisher = {American Mathematical Society},
  address   = {Providence, RI},
  doi       = {10.1090/surv/089}
}

@inproceedings{hein2005intrinsic,
  title={Intrinsic dimensionality estimation of submanifolds in Rd},
  author={Hein, Matthias and Audibert, Jean-Yves},
  booktitle={Proceedings of the 22nd international conference on Machine learning},
  pages={289--296},
  year={2005}
}

@article{lecun2002gradient,
  title={Gradient-based learning applied to document recognition},
  author={LeCun, Yann and Bottou, L{\'e}on and Bengio, Yoshua and Haffner, Patrick},
  journal={Proceedings of the IEEE},
  volume={86},
  number={11},
  pages={2278--2324},
  year={2002},
  publisher={Ieee}
}

@article{costa2004geodesic,
  title={Geodesic entropic graphs for dimension and entropy estimation in manifold learning},
  author={Costa, Jose A and Hero, Alfred O},
  journal={IEEE Transactions on Signal Processing},
  volume={52},
  number={8},
  pages={2210--2221},
  year={2004},
  publisher={IEEE}
}

@article{roy2026low,
  title={Low-Dimensional Adaptation of Rectified Flow: A New Perspective through the Lens of Diffusion and Stochastic Localization},
  author={Roy, Saptarshi and Rinaldo, Alessandro and Sarkar, Purnamrita},
  journal={arXiv preprint arXiv:2601.15500},
  year={2026}
}

@inproceedings{kumar2025likelihood,
  title={A Likelihood Based Approach to Distribution Regression Using Conditional Deep Generative Models},
  author={Kumar, Shivam and Yang, Yun and Lin, Lizhen},
  booktitle={International Conference on Machine Learning},
  pages={31964--31990},
  year={2025},
  organization={PMLR}
}
\bibliographystyle{apalike}

\newpage
\appendix
\onecolumn

\begin{center}
    {\bf \Large 
Supplementary Materials for ``Flow Matching is Adaptive to Manifold Structures''
 }
\end{center}

{ 
\section{Additional numerical experiments}\label{sec:add_num}
\subsection{Floral manifold}
The following example is designed to closely match our model assumptions while remaining visually interpretable.
\begin{example}[Floral segments embedded in $\bR^\sD$]\label{ex: floral}
Fix $\sd = 1$ and $\sD = 2$, and let $m \ge 2$ denote the number of petals.  For each $i \in \{0,1,\ldots,m-1\}$, define a spiral-segment curve
\[
    \psi_i(t)\;=\;\Big(r(t)\cos\theta_i(t),\; r(t)\sin\theta_i(t)\Big)\in \bR^2, \;\; t\in[0,1],
\]
where the radius increases linearly and the angle rotates slightly along the segment:
$$
    r(t) = r_{\mathrm{in}} + t\,(r_{\mathrm{out}} - r_{\mathrm{in}}),
    \qquad
    \theta_i(t) = \frac{2\pi i}{m} + 2\pi\,\tau\, t.
$$
Here $0 < r_{\mathrm{in}} < r_{\mathrm{out}}$ control the inner and outer radii, and
$\tau \in (0,1)$ determines the angular twist of each petal.

We define the manifold as the union of these spiral segments,
$$
    \mathcal{M}_0 =    \bigcup_{i=0}^{m-1} \lrm{ \psi_i(t) : t \in [0,1] }\subset \bR^2.
$$
\textbf{Target distribution $\bpi_1$.}
Draw $i \sim \mathrm{Unif}\{0,\dots,m-1\}$ and $t \sim \mathrm{Unif}[0,1]$, independently.
Let $Z_1, Z_2, Z_3 \sim \mathtt{N}(0,1)$ be independent noises. Define $\theta^\prime_i = \theta_i(t) + \sigma_\theta Z_1$, and generate the observed point in $\bR^2$ by
\[
X = \lrs{r(t) \cos(\btheta^\prime_i), r(t) \sin(\btheta^\prime_i) } + \sigma_r \cdot \lrs{Z_2, Z_3}.
\]
\end{example}
\paragraph{Implementation details}
We use the parameter values
\[
    (m, r_{\mathrm{in}}, r_{\mathrm{out}}, \tau, \sigma_r, \sigma_\theta) \;=\;  (5, 1, 4, 0.2, 0.05, 0.05).
\] 
The velocity field $v$ is parametrized by a multilayer perceptron conditioned on $t$ via a sinusoidal time embedding. We use a fully-connected network with width $256$ and depth $4$, ReLU activations, and a linear output layer in $\bR^2$. Training is performed using Adam with learning rate \(10^{-3}\), batch size $512$, and $5,000$ iterations. A cosine annealing learning-rate schedule is applied with \(T_{\max}=5000\) steps. For generation, we solve the ODE using the fourth-order Runge-Kutta with $N = 500$ time steps, using the discretization $t_i = [1 - (1- i/N)^2], \; i = 0, \ldots, N$.

\paragraph{Evaluations}
\begin{figure}[!h]
    \centering
    \includegraphics[width=0.8\linewidth]{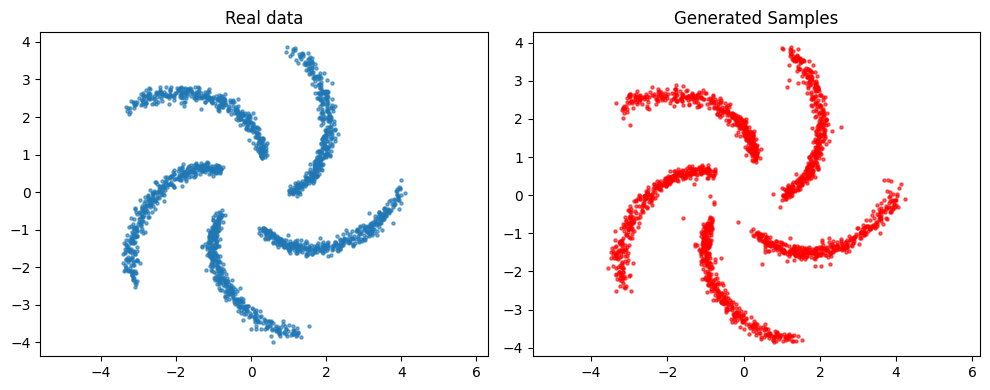}
    \caption{Comparison of generated samples and training data  for \Cref{ex: floral}. The learned flow generates samples that recover the petal geometry and place negligible mass in the regions between segments.}
    \label{fig:floral_generation}
\end{figure}
\Cref{ex: floral} provides an illustrative example that closely aligns with our model assumptions. Each spiral segment is a smooth one-dimensional curve ($\sd=1$), and the target distribution is supported on a union of such low-dimensional manifolds. This makes it a useful visual stress-test of the learned flow's ability to concentrate mass on $\cM$. We therefore provide samples in \Cref{fig:floral_generation}, which show that the learned sampler reproduces the multi-petal structure and generates points that lie on the spiral segments.

\subsection{Real data}
We validate manifold-adaptive convergence on MNIST handwritten digits \citep{lecun2002gradient}, a setting where the gap between ambient and intrinsic dimension is substantial and the ambient dimension is large. Each $28 \times 28$ grayscale image lies in $\bR^{784}$ ($\sD = 784$), yet prior work estimates the intrinsic dimension at $\sd \approx 10$--$15$ \citep{costa2004geodesic,hein2005intrinsic}. MNIST has also served as a standard testbed for studying generative modeling under the manifold hypothesis in  \citet{linsparse,kumar2025likelihood}. Our theory predicts that convergence rates should scale with $\sd$ rather than $\sD$; we test this by examining both generative quality and sample complexity.

\paragraph{Implementation details}
The velocity field $v$ is parametrized by a multilayer perceptron with width $1024$, depth $4$, LayerNorm, and ReLU activations; time conditioning uses a sinusoidal embedding of dimension $256$. Training uses Adam with learning rate $2 \times 10^{-4}$, batch size $512$, and $10{,}000$ iterations, with exponential moving average (decay $0.999$) applied to the weights. To handle the bounded pixel range $[0,1]$, we apply a logit transformation $\x \mapsto \log\bigl((\x + \alpha)/(1 - \x + \alpha)\bigr)$ with $\alpha = 0.05$ for dequantization, mapping images to $\bR^{784}$ where the Gaussian source is well-matched. We train separate models for each digit class $k \in \{0, \ldots, 9\}$, using the full training set per class ($\approx 5{,}000$--$6{,}000$ samples); this isolates each digit manifold and avoids confounding effects from multi-modal structure. For generation, we solve the learned ODE using forward Euler with $N = 500$ steps on the nonuniform grid $t_i = 1 - (1 - i/N)^2$, which clusters integration steps near $t = 1$ where the velocity field concentrates mass onto the target manifold.

\paragraph{Evaluation}
We evaluate distributional quality using the sliced 1-Wasserstein distance $\Wslice$ \citep{kolouri2019generalized}, which remains computationally tractable in high ambient dimension via random one-dimensional projections. For each digit, we generate $n_{\mathrm{eval}} = 1000$ samples from the learned flow and compare against $n_{\mathrm{eval}} = 1000$ held-out test samples, estimating $\Wslice$ with $K = 1000$ Monte Carlo directions. We report two quantities: $\WsliceD$ (generated vs.\ test) and $\WsliceBL$ (test vs.\ test), where the latter represents the irreducible finite-sample estimation error.

\begin{figure}[H]
\centering
\makebox[\textwidth]{%
\hfill
\includegraphics[width=0.4\textwidth]{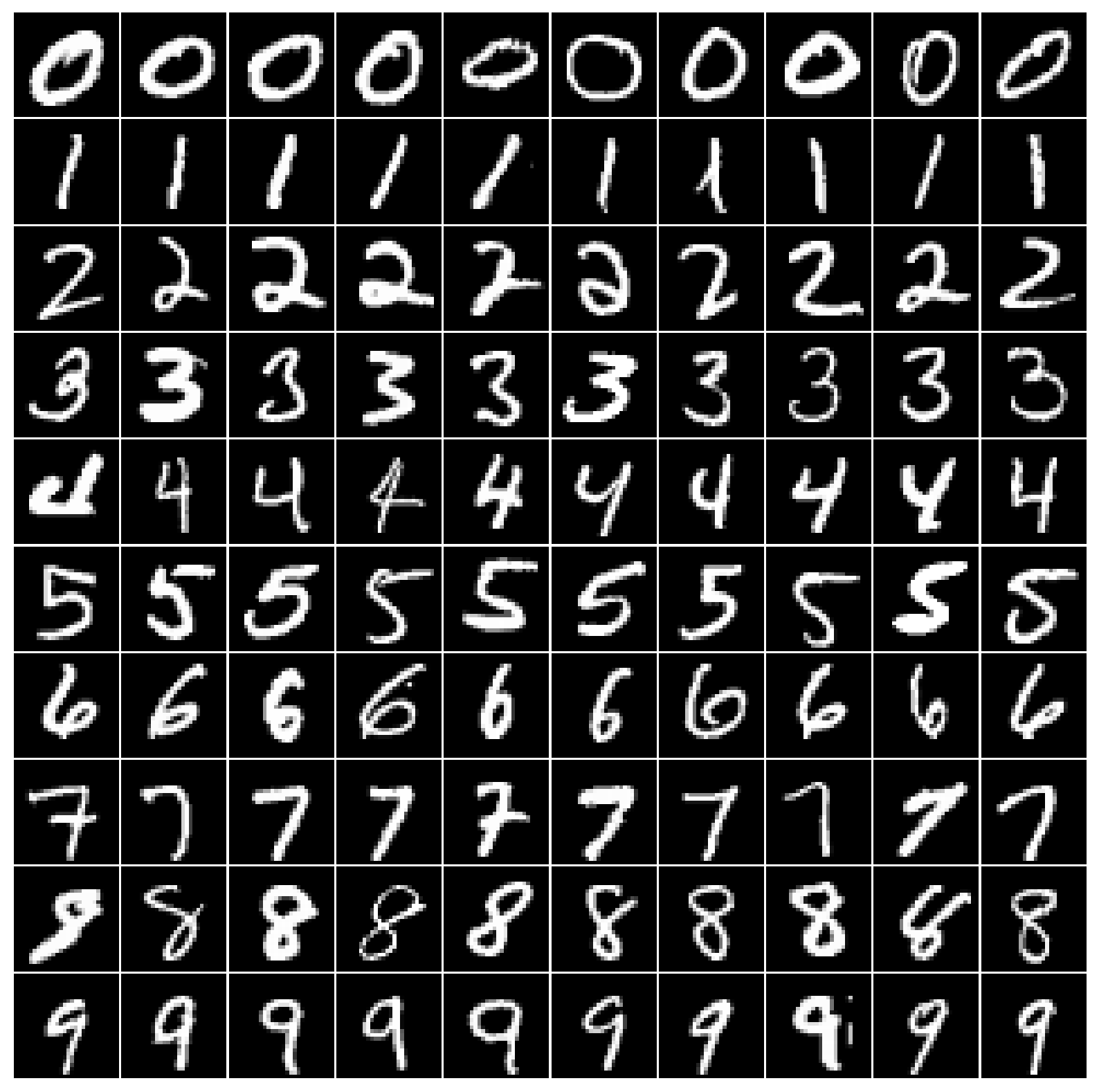}%
\hfill
\includegraphics[width=0.4\textwidth]{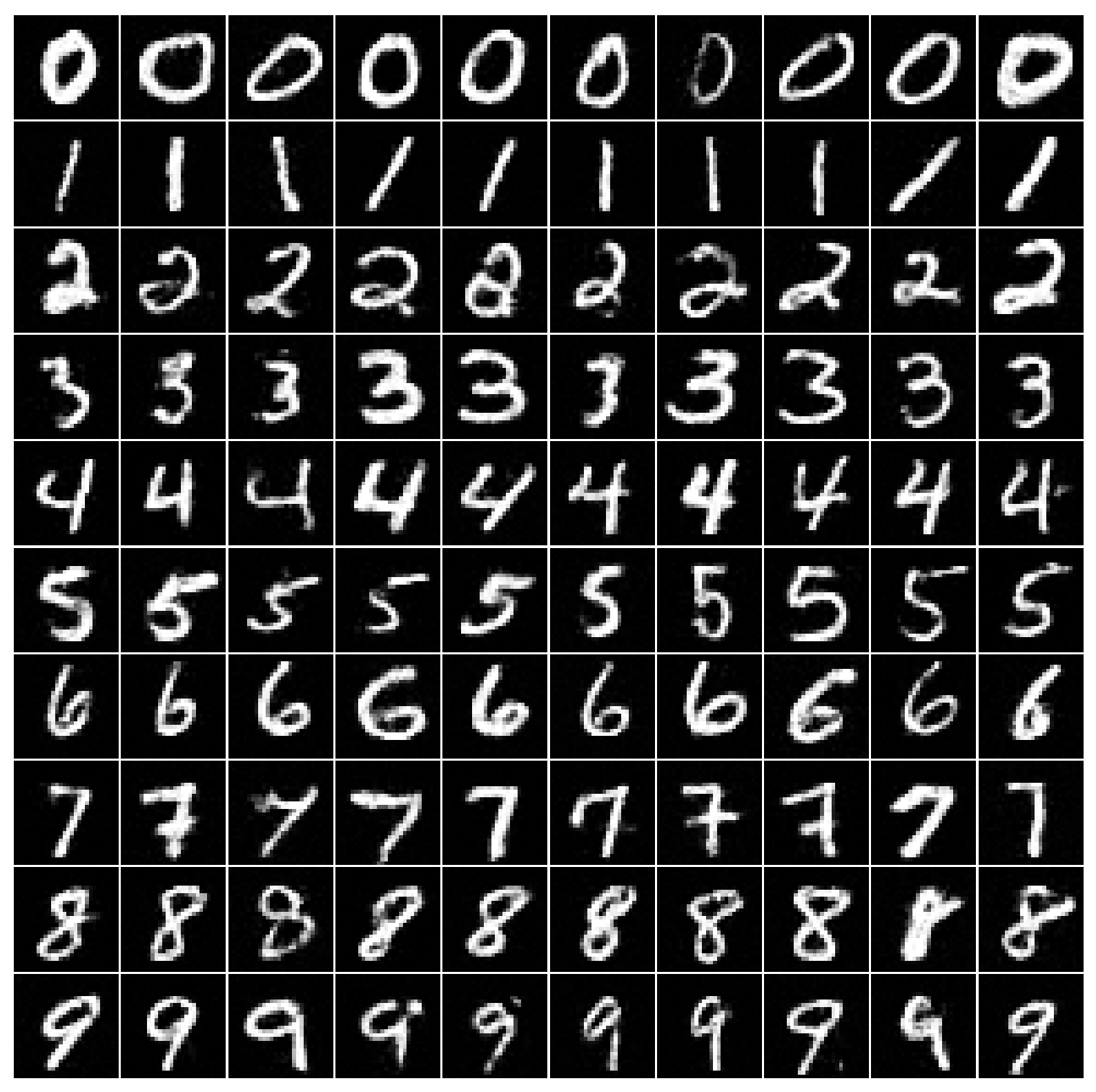}%
\hfill
}
\caption{Real (left) vs.\ generated (right) MNIST samples.}
\label{fig:mnist_samples}
\end{figure}

Table~\ref{tab:mnist_results} shows that across all digits, $\WsliceD$ lies within $1.1$--$2.0\times$ of $\WsliceBL$. This indicates that the learned flow produces samples whose distributional discrepancy from the true digit manifold is comparable to finite-sample noise, confirming that flow matching successfully learns the low-dimensional structure despite the high ambient dimension.

\begin{table}[H]
\centering
{
\caption{Per-digit evaluation on MNIST ($n_{\mathrm{eval}} = 1000$).}
\label{tab:mnist_results}
\vspace{0.5em}
\begin{tabular}{c|ccc}
\toprule
Digit & $\WsliceD$ & $\WsliceBL$ \\
\midrule
0 & $0.0229$ & $0.0212$ \\
1 & $0.0158$ & $0.0113$ \\
2 & $0.0273$ & $0.0195$ \\
3 & $0.0265$ & $0.0183$ \\
4 & $0.0246$ & $0.0177$ \\
5 & $0.0283$ & $0.0218$ \\
6 & $0.0333$ & $0.0172$ \\
7 & $0.0235$ & $0.0164$ \\
8 & $0.0304$ & $0.0188$ \\
9 & $0.0243$ & $0.0156$ \\
\bottomrule
\end{tabular}
}
\end{table}

\paragraph{Sample complexity ablation}

To directly probe the $n$-dependence predicted by our theory, we conduct an ablation study on digit~3. For each $n \in \{100, 250, 500, 1000, 2000, 5000\}$, we pre-generate a fixed training set of size $n$ (ensuring the same samples are used across all training runs at that $n$), train for $10{,}000$ iterations, and evaluate $\WsliceD$ against held-out test data.

\paragraph{Rate estimation}
We model the convergence as a power law $\WsliceD(n) = a \cdot n^{-\beta}$ and estimate $\beta$ via ordinary least squares on the log-transformed data. Table~\ref{tab:digit3_ablation} reports the results. Log-log regression yields $\hat{\beta} = 0.152$ with $R^2 = 0.867$ and 95\% confidence interval $[0.069, 0.234]$.

Under the theoretical rate $\beta = (\alpha + 1)/(2\alpha + \sd)$ with Lipschitz regularity $\alpha = 1$, inverting yields $\sd = 2/\beta - 2$. The observed $\hat{\beta} = 0.152$ implies
\[
\sd_{\mathrm{implied}} \;=\; \frac{2}{0.152} - 2 \;\approx\; 11.2,
\]
which falls squarely within the range $\sd \approx 10$--$15$ reported in prior intrinsic dimension studies \citep{costa2004geodesic,hein2005intrinsic}. This provides empirical support for the manifold-adaptive convergence predicted by \Cref{thm: main}.

Figure~\ref{fig:digit3_loglog} displays the log-log fit. The learned flow approaches the baseline $\WsliceBL = 0.0180$ as $n$ increases.

\begin{table}[H]
\centering
\caption{Sample complexity ablation for digit~3.}
\label{tab:digit3_ablation}
\vspace{0.5em}
{
\begin{tabular}{c|c}
\toprule
$n$ & $\WsliceD$ \\
\midrule
100 & $0.0416$ \\
250 & $0.0438$ \\
500 & $0.0406$ \\
1000 & $0.0325$ \\
2000 & $0.0275$ \\
5000 & $0.0254$ \\
\bottomrule
\end{tabular}
}
\end{table}

\begin{figure}[H]
\centering
\includegraphics[width=0.7\textwidth]{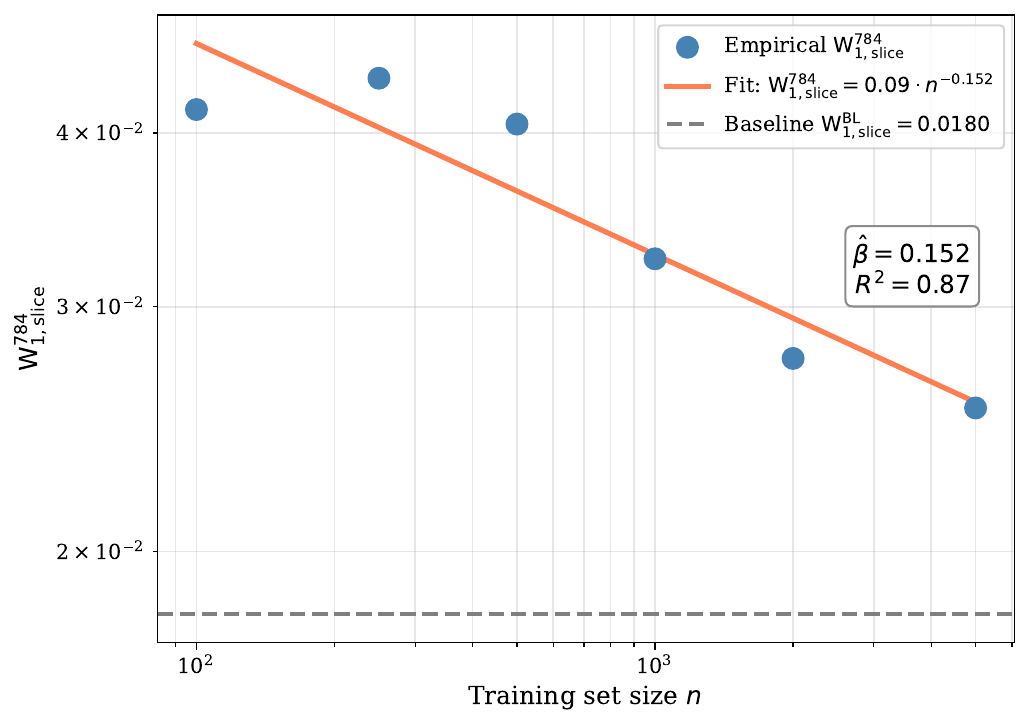}
\caption{Log-log regression for digit~3. Points show empirical $\WsliceD$; solid line shows the power-law fit $\WsliceD \propto n^{-0.152}$; dashed horizontal line indicates baseline $\WsliceBL = 0.0180$.}
\label{fig:digit3_loglog}
\end{figure}


\subsection{Sample complexity ablation on the sphere}

The projected sphere manifold (Example~\ref{ex: sphere}) provides a controlled setting to test two predictions of Theorem~\ref{thm: main}: (i) the convergence rate $\Wstd \propto n^{-\gamma}$ depends on the intrinsic dimension $\sd$, and (ii) fixing $\sd$, the rate is independent of the ambient dimension $\sD$. We conduct an $n$-ablation across multiple $(\sd, \sD)$ pairs to probe both predictions directly.

\paragraph{Experimental design}
For each $(\sd, \sD) \in \{(2,6), (2,9), (2,12), (3,8), (3,12), (4,10)\}$, we pre-generate a fixed training set of size $n \in \{256, 512, 1024, 2048, 4096, 8192, 16384\}$ from the target $\bpi_1$, train the velocity field, and evaluate $\Wstd$ against $N_{\mathrm{eval}} = 4096$ fresh samples. The baseline $\WstdBL$ is computed between two independent test batches, representing the irreducible finite-sample floor. All other settings follow Example~\ref{ex: sphere}.

\paragraph{Results}
Table~\ref{tab:sphere_n_ablation} reports $\Wstd$ as a function of $n$ for six $(\sd, \sD)$ configurations. Across all settings, $\Wstd$ decreases monotonically with $n$, approaching baselines $\WstdBL \approx 0.011$--$0.014$. The key observation is that, at fixed $\sd$, the values of $\Wstd$ are nearly identical across different $\sD$. For instance, at $\sd = 2$ and $n = 4096$, we obtain $\Wstd \approx 0.018$ for $\sD \in \{6, 9, 12\}$. This confirms that convergence is governed by the intrinsic dimension $\sd$, not the ambient dimension $\sD$, providing direct empirical support for manifold-adaptive convergence.

\begin{table}[H]
\centering
\caption{Sample complexity ablation on the sphere $\bS^{\sd} \subset \bR^{\sD}$.}
\label{tab:sphere_n_ablation}
\vspace{0.5em}
\resizebox{\columnwidth}{!}{%
{
\begin{tabular}{r|ccc|cc|c}
\toprule
& \multicolumn{3}{c|}{$\sd = 2$} & \multicolumn{2}{c|}{$\sd = 3$} & $\sd = 4$ \\
$n$ & $\sD = 6$ & $\sD = 9$ & $\sD = 12$ & $\sD = 8$ & $\sD = 12$ & $\sD = 10$ \\
\midrule
256   & $0.043 \pm 0.018$ & $0.052 \pm 0.018$ & $0.045 \pm 0.014$ & $0.047 \pm 0.013$ & $0.057 \pm 0.011$ & $0.042 \pm 0.006$ \\
512   & $0.029 \pm 0.008$ & $0.043 \pm 0.017$ & $0.035 \pm 0.009$ & $0.034 \pm 0.009$ & $0.036 \pm 0.005$ & $0.028 \pm 0.007$ \\
1024  & $0.033 \pm 0.006$ & $0.027 \pm 0.008$ & $0.034 \pm 0.005$ & $0.029 \pm 0.005$ & $0.033 \pm 0.005$ & $0.025 \pm 0.004$ \\
2048  & $0.022 \pm 0.005$ & $0.024 \pm 0.004$ & $0.026 \pm 0.005$ & $0.020 \pm 0.005$ & $0.022 \pm 0.006$ & $0.023 \pm 0.003$ \\
4096  & $0.019 \pm 0.004$ & $0.019 \pm 0.005$ & $0.017 \pm 0.005$ & $0.017 \pm 0.003$ & $0.020 \pm 0.005$ & $0.019 \pm 0.005$ \\
8192  & $0.018 \pm 0.001$ & $0.025 \pm 0.005$ & $0.020 \pm 0.002$ & $0.017 \pm 0.004$ & $0.020 \pm 0.004$ & $0.016 \pm 0.005$ \\
16384 & $0.020 \pm 0.006$ & $0.024 \pm 0.007$ & $0.017 \pm 0.006$ & $0.018 \pm 0.007$ & $0.017 \pm 0.004$ & $0.016 \pm 0.003$ \\
\midrule
$\WstdBL$ & $0.013 \pm 0.002$ & $0.014 \pm 0.002$ & $0.011 \pm 0.002$ & $0.012 \pm 0.002$ & $0.011 \pm 0.001$ & $0.012 \pm 0.002$ \\
\bottomrule
\end{tabular}%
}
}
\end{table}

\paragraph{Rate estimation}
We model convergence as $\Wstd(n) \propto n^{-\hat\gamma}$ and estimate $\hat\gamma$ via OLS on log-transformed data for $n \ge 512$ (excluding $n = 256$ due to high variance at small sample sizes). Since $\sD$-independence holds empirically, we pool data across ambient dimensions for each $\sd$, obtaining $\hat\gamma = 0.14$ for $\sd = 2$, $\hat\gamma = 0.19$ for $\sd = 3$, and $\hat\gamma = 0.17$ for $\sd = 4$.


\begin{figure}[H]
\centering
\includegraphics[width=\textwidth]{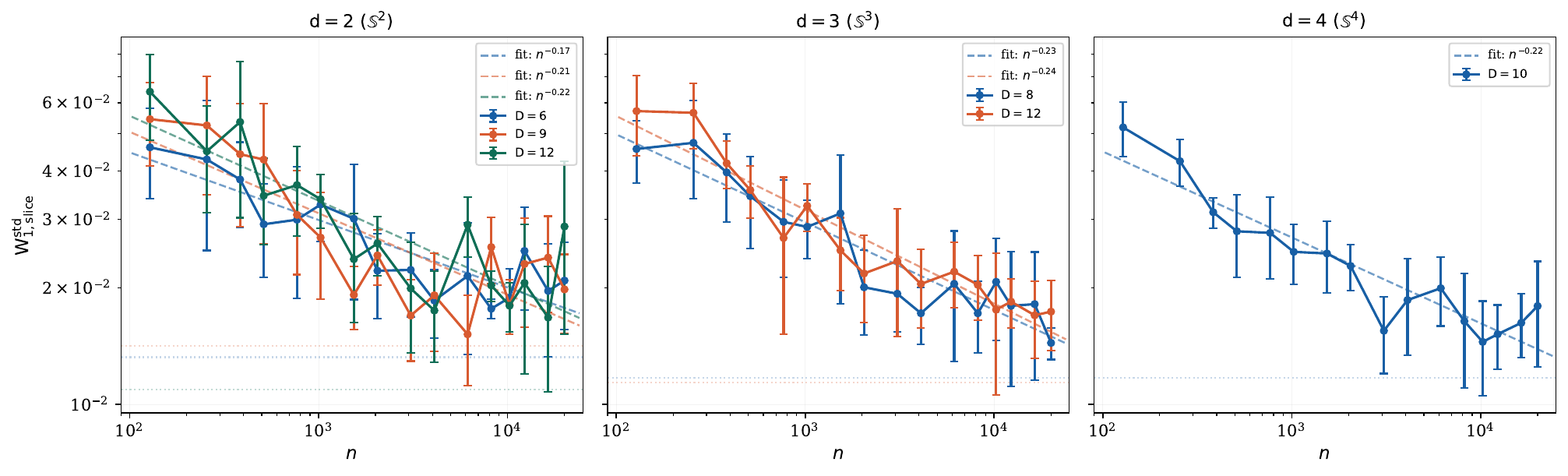}
\caption{Sample complexity on the sphere (log-log). Solid: empirical $\Wstd$; dashed: power-law fit; dotted: baseline $\WstdBL$.}
\label{fig:sphere_rate}
\end{figure}



}

{ 
\section{Logarithmic norm and one-sided Lipschitz condition}\label{sec:app_lognorm}

\begin{definition}[Logarithmic norm]
For $\bA\in\R^{D\times D}$, the \emph{logarithmic norm} with respect to the $\ell_2$-norm is
\begin{equation}\nonumber
    \mu_2(\bA) \;:=\; \lambda_{\max}\!\!\left(\frac{\bA+\bA^\top}{2}\right).
\end{equation}
\end{definition}

\begin{lemma}[Properties of $\mu_2$]\label{lem:mu2_props}
Let $\bA\in\R^{D\times D}$.
\begin{enumerate}[label=\textnormal{(\alph*)},leftmargin=*,itemsep=6pt]
    \item \label{item:mu_leq_op} $\mu_2(\bA) \le \|\bA\|_{\op}$.
    \item \label{item:mu_neg} $\mu_2(\bA)$ can be negative: $\mu_2(-c\,\bI) = -c$ for $c>0$.
    \item \label{item:mu_sym} If $\bA$ is symmetric, then $\mu_2(\bA) = \lambda_{\max}(\bA)$ and $\|\bA\|_{\op} = \max_j|\lambda_j(\bA)|$.
\end{enumerate}
\end{lemma}

\begin{proof}
\ref{item:mu_leq_op}:\ \ For any unit vector $\bu$,
\begin{align*}
    \bu^\top\frac{\bA+\bA^\top}{2}\bu
    = \mathrm{Re}(\bu^\top\bA\bu)
    \le |\bu^\top\bA\bu|
    \le \|\bA\bu\|\,\|\bu\|
    \le \|\bA\|_{\op}.
\end{align*}
Taking the supremum over unit $\bu$ gives $\mu_2(\bA)\le\|\bA\|_{\op}$.

\medskip
\ref{item:mu_neg}:\ \ $\frac{(-c\bI)+(-c\bI)^\top}{2} = -c\bI$, whose largest eigenvalue is $-c$.

\medskip
\ref{item:mu_sym}:\ \ When $\bA=\bA^\top$, $\frac{\bA+\bA^\top}{2}=\bA$, so $\mu_2(\bA)=\lambda_{\max}(\bA)$.  Meanwhile, $\|\bA\|_{\op}=\max_j|\lambda_j(\bA)|$.
\end{proof}

\begin{definition}[One-sided Lipschitz condition]\label{def:osl}
A vector field $b:\R^D\to\R^D$ is $\mu$-\emph{one-sided Lipschitz} ($\mu$-OSL) if
\begin{equation}\label{eq:osl_def}
    \langle b(\bx)-b(\by),\;\bx-\by\rangle
    \;\le\;\mu\,\|\bx-\by\|^2
    \qquad\forall\;\bx,\by\in\R^D.
\end{equation}
\end{definition}

\begin{lemma}\label{lem:osl_equiv}
If $b:\R^D\to\R^D$ is continuously differentiable, then $b$ is $\mu$-OSL if and only if $\mu_2\!\bigl(\frac{\partial b}{\partial\bx}(\bx)\bigr)\le\mu$ for all~$\bx$.
\end{lemma}

\begin{proof}
Set $\bw:=\bx-\by$.  By the mean-value theorem in integral form,
\begin{align*}
    b(\bx)-b(\by) &= \int_0^1 \frac{\partial b}{\partial\bx}\bigl(\by+s\bw\bigr)\,\bw\;ds.
\end{align*}
Taking the inner product with $\bw$ and using $\bw^\top\J\bw = \bw^\top\frac{\J+\J^\top}{2}\bw$:
\begin{align*}
    \langle b(\bx)-b(\by),\;\bw\rangle
    &= \int_0^1 \bw^\top \frac{\J(s)+\J(s)^\top}{2}\,\bw\;ds
    \;\le\; \sup_{s\in[0,1]}\mu_2\bigl(\J(s)\bigr)\,\|\bw\|^2,
\end{align*}
where $\J(s):=\frac{\partial b}{\partial\bx}(\by+s\bw)$.  Hence $\mu_2(\J)\le\mu$ everywhere implies $\mu$-OSL.  The converse follows by taking $\bx=\by+\epsilon\bu$ and sending $\epsilon\to 0$.
\end{proof}
}

{

\section{Semi-convex densities on manifolds}\label{sec:app_classes}
\noindent\textbf{Notation.}  We write $\sigma_t:=1-t$ for the noise scale.  The eigenvalues of the posterior covariance $\Sigma_{\mathrm{post}}(\bx,t):=\Cov(X_1\mid X_t=\bx)$ are denoted $\kappa_1^2\ge\kappa_2^2\ge\cdots\ge\kappa_\sD^2\ge 0$.  When necessary, we distinguish tangential eigenvalues $\kappa_{\mathrm{tan},j}^2$ ($j=1,\dots,\sd$) from normal eigenvalues $\kappa_{\mathrm{norm},j}^2$ ($j=1,\dots,\sD-\sd$).

\subsection{Posterior mean and its Jacobian}\label{subsec:Dg_app}

Recall from the paper that $X_t = tX_1+\sigma_tX_0$, with $X_0\sim\mathtt{N}(\bm{0},\bI_\sD)$ and $X_1\sim\bpi_1$.  The velocity field is
\begin{equation}\label{eq:vstar_app}
    v^\star(\bx,t) = \frac{g(\bx,t)-\bx}{\sigma_t}\,,
    \qquad
    g(\bx,t) := \bE[X_1\mid X_t=\bx],
\end{equation}
with spatial Jacobian
\begin{equation}\label{eq:Jac_vstar_app}
    \J(\bx,t)
    \;:=\;
    \frac{\partial v^\star}{\partial\bx}(\bx,t)
    \;=\;
    \frac{1}{\sigma_t}\!\left(\frac{\partial g}{\partial\bx}-\bI_\sD\right).
\end{equation}

\begin{proposition}\label{prop:Dg}
For every $\bx\in\R^\sD$ and $t\in(0,1)$,
\begin{equation}\label{eq:Dg}
    \frac{\partial g}{\partial\bx}(\bx,t)
    \;=\;
    \frac{t}{\sigma_t^2}\;\Sigma_{\mathrm{post}}(\bx,t).
\end{equation}
In particular, $\partial g/\partial\bx$ is symmetric positive semi-definite.
\end{proposition}

\begin{proof}
The conditional density is $p_t(\by\mid\bx) = \bpi_1(\by)\,\varphi_{\sigma_t}(\bx-t\by)\big/Z(\bx)$, where $\varphi_{\sigma_t}$ is the $\sD$-dimensional Gaussian density with variance $\sigma_t^2$ and $Z(\bx):=\int_\cM \bpi_1(\by)\,\varphi_{\sigma_t}(\bx-t\by)\,d\Vol_\cM(\by)$.  Define $N_i(\bx):=\int_\cM y_i\,\bpi_1(\by)\,\varphi_{\sigma_t}(\bx-t\by)\,d\Vol_\cM(\by)$, so $g_i=N_i/Z$.

\medskip
\noindent\textbf{Step~1.}  Differentiating the Gaussian kernel:
\begin{align}
    \frac{\partial}{\partial x_j}\varphi_{\sigma_t}(\bx-t\by)
    &= \varphi_{\sigma_t}(\bx-t\by)\cdot\frac{ty_j-x_j}{\sigma_t^2}\,. \label{eq:dphi_dxj}
\end{align}

\medskip
\noindent\textbf{Step~2.}  Differentiating $Z$ and $N_i$ under the integral:
\begin{align}
    \frac{\partial Z}{\partial x_j}
    &= \frac{Z}{\sigma_t^2}\bigl(t\,\bE_{p_t}[Y_j]-x_j\bigr), \label{eq:dZ_app}
    \\[8pt]
    \frac{\partial N_i}{\partial x_j}
    &= \frac{Z}{\sigma_t^2}\bigl(t\,\bE_{p_t}[Y_iY_j]-x_j\,\bE_{p_t}[Y_i]\bigr), \label{eq:dNi_app}
\end{align}
where $\bE_{p_t}[\cdot]$ denotes expectation under $p_t(\cdot\mid\bx)$.

\medskip
\noindent\textbf{Step~3.}  Applying the quotient rule $\partial_j g_i = Z^{-1}\partial_j N_i - N_i Z^{-2}\partial_j Z$ and substituting~\eqref{eq:dZ_app}--\eqref{eq:dNi_app}:
\begin{align}
    \frac{\partial g_i}{\partial x_j}
    &= \frac{1}{\sigma_t^2}\bigl(t\,\bE_{p_t}[Y_iY_j]-x_j\,\bE_{p_t}[Y_i]\bigr)
    \;-\;
    \frac{g_i}{\sigma_t^2}\bigl(t\,\bE_{p_t}[Y_j]-x_j\bigr)
    \nonumber\\[10pt]
    &= \frac{t}{\sigma_t^2}\Bigl(\bE_{p_t}[Y_iY_j]-\bE_{p_t}[Y_i]\,\bE_{p_t}[Y_j]\Bigr)
    \;-\;
    \frac{x_j}{\sigma_t^2}\bigl(\underbrace{\bE_{p_t}[Y_i]-g_i}_{=\,0}\bigr)
    \nonumber\\[10pt]
    &= \frac{t}{\sigma_t^2}\,\bigl[\Sigma_{\mathrm{post}}\bigr]_{ij}\,. \label{eq:dgij_app} \qedhere
\end{align}
\end{proof}

\subsection{Eigenvalue structure of $\J$}\label{subsec:eigval_app}

\begin{corollary}\label{cor:eigenvalues}
The Jacobian $\J(\bx,t)$ is symmetric.  Its eigenvalues are
\begin{equation}\label{eq:eigvals}
    \lambda_j
    \;=\;
    \frac{1}{\sigma_t}\!\left(\frac{t\,\kappa_j^2}{\sigma_t^2}-1\right),
    \qquad j=1,\dots,\sD.
\end{equation}
Moreover:
\begin{enumerate}[label=\textnormal{(\alph*)},leftmargin=*,itemsep=6pt]
    \item $\lambda_j\ge 0$ if and only if $\kappa_j^2\ge\sigma_t^2/t$.

    \item $\lambda_j = -1/\sigma_t$ when $\kappa_j^2=0$ (normal contraction towards~$\cM$).

    \item Since $\J$ is symmetric (\Cref{lem:mu2_props}\ref{item:mu_sym}):
    \begin{equation}\label{eq:mu_vs_op_app}
        \mu_2(\J) = \lambda_{\max}(\J) = \lambda_1,
        \qquad
        \|\J\|_{\op} = \max_{1\le j\le\sD}|\lambda_j|.
    \end{equation}
\end{enumerate}
\end{corollary}

\begin{proof}
From~\eqref{eq:Jac_vstar_app} and~\eqref{eq:Dg}:
\begin{align}
    \J
    &= \frac{1}{\sigma_t}\!\left(\frac{t}{\sigma_t^2}\,\Sigma_{\mathrm{post}}-\bI_\sD\right). \label{eq:J_formula}
\end{align}
Both $\Sigma_{\mathrm{post}}$ and $\bI_\sD$ are real symmetric, hence $\J$ is symmetric.  Let $\bu_j$ be a unit eigenvector of $\Sigma_{\mathrm{post}}$ with eigenvalue $\kappa_j^2\ge 0$.  Then:
\begin{align}
    \J\,\bu_j
    &= \frac{1}{\sigma_t}\!\left(\frac{t\,\kappa_j^2}{\sigma_t^2}\,\bu_j-\bu_j\right)
    \nonumber\\[8pt]
    &= \frac{1}{\sigma_t}\!\left(\frac{t\,\kappa_j^2}{\sigma_t^2}-1\right)\bu_j
    \;=:\;\lambda_j\,\bu_j\,. \label{eq:eigvec_comp}
\end{align}
Parts~(a)--(c) follow directly from~\eqref{eq:eigvec_comp}.
\end{proof}

\subsection{Riemannian preliminaries}\label{subsec:riem_prelim}

Let $(\cM,\mathfrak{g})$ be a $\sd$-dimensional complete Riemannian manifold.  The \emph{Riemannian Hessian} of $f:\cM\to\R$ is the symmetric $(0,2)$-tensor $(\mathrm{Hess}_\cM f)(\bv,\bw):=\mathfrak{g}(\nabla_\bv\nabla_\cM f,\;\bw)$.  In normal coordinates at~$\by_0$: $[\mathrm{Hess}_\cM f]_{ij}(\by_0)=\partial^2 f/\partial u_i\partial u_j(\by_0)$.

\begin{definition}[Semi-convexity]\label{def:semiconvex}
$V\in C^2(\cM)$ is $M$-\emph{semi-convex} ($M\ge 0$) if
\begin{equation}\label{eq:semiconvex_def}
    \mathrm{Hess}_\cM V(\by) \;\succeq\; -M\,\mathfrak{g}(\by)
    \qquad\forall\;\by\in\cM.
\end{equation}
The case $M=0$ (geodesic convexity) corresponds to log-concavity of $\bpi_1=e^{-V}/Z$.
\end{definition}

\subsection{Semi-convexity on compact manifolds}\label{subsec:auto_semiconvex}

\begin{proposition}\label{prop:semiconvex}
Let $\cM$ be compact without boundary, $\bpi_1\in C^2(\cM)$, $\bpi_1\ge c_0>0$.  Then $V:=-\log\bpi_1$ is $M_V$-semi-convex with
\begin{equation}\label{eq:M_bound}
    M_V
    \;:=\;
    \sup_{(\by,\bv)\in S^*\!\cM}
    \max\!\bigl\{0,\;-[\mathrm{Hess}_\cM V(\by)](\bv,\bv)\bigr\}
    \;<\;\infty,
\end{equation}
where $S^*\!\cM$ is the unit tangent bundle.
\end{proposition}

\begin{proof}
Since $\bpi_1\ge c_0>0$, $V=-\log\bpi_1\in C^2(\cM)$ with Hessian
\begin{align}
    \mathrm{Hess}_\cM V
    &= -\frac{\mathrm{Hess}_\cM\bpi_1}{\bpi_1}
    + \frac{\nabla_\cM\bpi_1\otimes\nabla_\cM\bpi_1}{\bpi_1^2}\,. \label{eq:hessV}
\end{align}
The map $(\by,\bv)\mapsto[\mathrm{Hess}_\cM V(\by)](\bv,\bv)$ is continuous on the compact set $S^*\!\cM$.  By the extreme value theorem, $M_V<\infty$.
\end{proof}

\begin{remark}[Examples of semi-convexity constants]\label{rem:M_examples}~

\begin{enumerate}[label=\textnormal{(\alph*)},leftmargin=*,itemsep=6pt]
    \item \emph{Uniform density} ($V\equiv\mathrm{const}$):\; $M_V=0$.  The density is log-concave.

    \item \emph{Von~Mises--Fisher on $S^\sd$}:\;
    $\bpi_1(\by)\propto e^{\kappa\langle\bm\mu,\by\rangle}$, so $V(\by)=-\kappa\langle\bm\mu,\by\rangle$.  The Riemannian Hessian on $S^\sd$ evaluates to $[\mathrm{Hess}_{S^\sd}V](\bv,\bv)=\kappa\langle\bm\mu,\by\rangle\|\bv\|^2$.  The minimum is $-\kappa$ (at $\by=-\bm\mu$), giving $M_V=\kappa$.

    \item \emph{Projected Gaussian on $S^\sd$}:\;
    finite $M_V$ by \Cref{prop:semiconvex} for moderate $\|\bm\gamma\|$.

    \item \emph{Any $C^2$ density bounded below on compact $\cM$}:\;
    finite $M_V$ by \Cref{prop:semiconvex}, with no convexity assumption.
\end{enumerate}
\end{remark}

\subsection{Posterior covariance bound}\label{subsec:BL}

The posterior of $X_1$ given $X_t=\bx$ has density $p_t(\by\mid\bx)\propto e^{-\Phi(\by)}$ on $\cM$, where
\begin{equation}\label{eq:Phi_def}
    \Phi(\by)
    \;:=\;
    V(\by) + \frac{\|\bx-t\by\|^2}{2\sigma_t^2}\,.
\end{equation}
The key tool is the following classical variance bound.

\begin{lemma}[Brascamp--Lieb inequality {\citep{brascamp1976extensions}}; see also {\citet{ledoux2001concentration}}]\label{lem:BL}
Let $\mu\propto e^{-\Phi}\,d\Vol_\cM$ with $\mathrm{Hess}_\cM\Phi\succeq\rho\,\mathfrak{g}$ for some $\rho>0$.  Then for every smooth $f:\cM\to\R$,
\begin{equation}\label{eq:BL_ineq}
    \Var_\mu(f)
    \;\le\;
    \frac{1}{\rho}\int_\cM \|\nabla_\cM f\|_\mathfrak{g}^2\;d\mu.
\end{equation}
\end{lemma}

\noindent Applied to the posterior $\mu=p_t(\cdot|\bx)$ with test functions $f(\by)=\langle\by,\bu_j\rangle$ (unit tangent vectors, so $\|\nabla_\cM f\|_\mathfrak{g}\le 1$), this gives $\kappa_{\mathrm{tan},j}^2\le 1/\rho$.

\
\\
It remains to establish a lower bound on $\rho$.  The Gaussian term $Q(\by):=\|\bx-t\by\|^2/(2\sigma_t^2)$ contributes a leading $t^2/\sigma_t^2$ to $\mathrm{Hess}_\cM\Phi$, plus a curvature correction from the second fundamental form $\mathrm{I\!I}$ of $\cM\hookrightarrow\R^\sD$.  In normal coordinates at $\by_0\in\cM$:
\begin{align}
    [\mathrm{Hess}_\cM Q]_{ij}(\by_0)
    &= \frac{t^2}{\sigma_t^2}\,\delta_{ij}
    - \frac{t}{\sigma_t^2}\,\bigl\langle\bx-t\by_0,\;\mathrm{I\!I}(\be_i,\be_j)\bigr\rangle. \label{eq:d2Q_du}
\end{align}
Since $\mathrm{I\!I}$ maps into the normal space, only the normal component of $\bx-t\by_0$ contributes.  We denote the resulting curvature correction by $C_{\mathrm{curv}}$, a constant depending on $\sD$, $\sC_\cM$, and the reach $\tau$ of~$\cM$ (cf.~Assumption~3.1).  Combining with $\mathrm{Hess}_\cM V\succeq -M_V\,\mathfrak{g}$, we obtain the following.

\begin{proposition}[Posterior covariance under semi-convexity]\label{prop:semiconvex_posterior}
Let $\bpi_1=e^{-V}/Z$ with $V$ being $M_V$-semi-convex on~$\cM$.  Define the effective constant
\begin{equation}\label{eq:M_eff}
    M \;:=\; M_V + C_{\mathrm{curv}}\,.
\end{equation}
If $\,t^2/\sigma_t^2>M$\, (equivalently, $t>t_M:=\sqrt{M}/(1+\sqrt{M})$), then
\begin{equation}\label{eq:semiconvex_bound}
    \kappa_{\mathrm{tan},j}^2
    \;\le\;
    \frac{\sigma_t^2}{t^2-M\sigma_t^2}\,,
    \qquad
    j=1,\dots,\sd.
\end{equation}
\end{proposition}

\begin{proof}
From~\eqref{eq:d2Q_du} and the bound on the curvature correction, $\mathrm{Hess}_\cM Q\succeq(t^2/\sigma_t^2 - C_{\mathrm{curv}}/\sigma_t^2)\,\mathfrak{g}$.  Combined with $\mathrm{Hess}_\cM V\succeq -M_V\,\mathfrak{g}$:
\begin{align}
    \mathrm{Hess}_\cM\Phi
    &\succeq \left(\frac{t^2-C_{\mathrm{curv}}}{\sigma_t^2} - M_V\right)\mathfrak{g}
    \;=\;
    \frac{t^2 - C_{\mathrm{curv}} - M_V\sigma_t^2}{\sigma_t^2}\,\mathfrak{g}. \label{eq:hess_Phi}
\end{align}
Since $\sigma_t^2\le 1$, we have $C_{\mathrm{curv}}+M_V\sigma_t^2\le M_V+C_{\mathrm{curv}}=M$, and therefore
\begin{align}
    t^2 - C_{\mathrm{curv}} - M_V\sigma_t^2
    &\;\ge\;
    t^2 - M\sigma_t^2\,. \label{eq:rho_lb}
\end{align}
Under the hypothesis $t^2/\sigma_t^2>M$, the right side is positive and we set $\rho:=(t^2-M\sigma_t^2)/\sigma_t^2>0$.  \Cref{lem:BL} then gives
\begin{align}
    \kappa_{\mathrm{tan},j}^2
    &\;\le\;
    \frac{1}{\rho}
    \;=\;
    \frac{\sigma_t^2}{t^2-M\sigma_t^2}\,. \label{eq:BL_applied} \qedhere
\end{align}
\end{proof}

\begin{remark}\label{rem:rho_inequality}
The inequality~\eqref{eq:rho_lb} deserves emphasis.  The actual convexity parameter from~\eqref{eq:hess_Phi} is $\rho_{\mathrm{exact}} = (t^2-C_{\mathrm{curv}}-M_V\sigma_t^2)/\sigma_t^2$, which is at least as large as $\rho=(t^2-M\sigma_t^2)/\sigma_t^2$.  As $t\to 1$ ($\sigma_t\to 0$):
\begin{align*}
    \rho_{\mathrm{exact}}
    \;=\;
    \frac{t^2-C_{\mathrm{curv}}}{\sigma_t^2} - M_V
    \;\longrightarrow\;
    +\infty,
\end{align*}
so the posterior becomes more strongly log-concave as $t\to 1$, regardless of the curvature constant.  The lower bound $\rho\ge(t^2-M\sigma_t^2)/\sigma_t^2\to 1/\sigma_t^2$ captures this.
\end{remark}

\subsection{From posterior covariance to logarithmic norm}\label{subsec:cov_to_mu2}

\begin{corollary}[Semi-convex densities satisfy Assumption~\ref{assume: Lipschitz}; formal version of~\eqref{eq:mu2_main_informal}]\label{cor:semiconvex_osl}
Let $\cM\subset[-\sC_\cM,\sC_\cM]^\sD$ be compact, $\bpi_1=e^{-V}/Z$ with effective constant $M$ as in~\eqref{eq:M_eff}.  {Define the crossover time
\begin{equation}\label{eq:t_dagger}
    t_\dagger
    \;:=\;
    \max\!\left(\frac{\sqrt{M}}{1+\sqrt{M}},\;\frac{1}{1+\sC_\cM}\right).
\end{equation}
Note that $1>t_\dagger>0$ always (since $\sC_\cM<\infty$).}
\begin{enumerate}[label=\textnormal{(\alph*)},leftmargin=*,itemsep=6pt]
    \item For $t\in({t_\dagger},1)$\; :
    \begin{equation}\label{eq:mu2_semiconvex}
        \mu_2(\J)
        \;\le\;
        \frac{t+M\sigma_t}{t^2-M\sigma_t^2}\,.
    \end{equation}
    \item For $t\in[0,{t_\dagger}]$\; :
    \begin{equation}\label{eq:mu2_small_t_statement}
        \mu_2(\J)
        \;\le\;
        {\frac{t_\dagger\,\sC_\cM^2}{(1-t_\dagger)^3}}
        \;=:\;C_0
        \;<\;\infty.
    \end{equation}
\end{enumerate}
In particular, $\mu_2(\J)$ is uniformly bounded over $t\in[0,1)$ and Assumption~\ref{assume: Lipschitz} holds with any $\xi\in(0,1)$.
\end{corollary}

\begin{proof}
By \Cref{cor:eigenvalues}, $\mu_2(\J)=\lambda_1$ where $\lambda_j = \sigma_t^{-1}(t\kappa_j^2/\sigma_t^2-1)$.

\medskip
\noindent\textbf{Case $t>{t_\dagger}$ (Brascamp--Lieb regime).} {Since $t_\dagger\ge\sqrt{M}/(1+\sqrt{M})$, we have $t^2/\sigma_t^2>M$, so \Cref{prop:semiconvex_posterior} applies.}  Substituting~\eqref{eq:semiconvex_bound} into the eigenvalue formula:
\begin{align}
    \lambda_{\mathrm{tan},j}
    &= \frac{1}{\sigma_t}\!\left(\frac{t\,\kappa_{\mathrm{tan},j}^2}{\sigma_t^2}-1\right)
    \;\le\;
    \frac{1}{\sigma_t}\!\left(\frac{t}{t^2-M\sigma_t^2}-1\right). \label{eq:lam_sub}
\end{align}
Simplifying the parenthesized expression:
\begin{align}
    \frac{t}{t^2-M\sigma_t^2}-1
    &= \frac{t - t^2 + M\sigma_t^2}{t^2-M\sigma_t^2}
    \;=\;
    \frac{t\sigma_t + M\sigma_t^2}{t^2-M\sigma_t^2}\,,  \label{eq:parens_simplify}
\end{align}
where we used $t-t^2=t(1-t)=t\sigma_t$.  Dividing by~$\sigma_t$:
\begin{align}
    \lambda_{\mathrm{tan},j}
    &\;\le\;
    \frac{t+M\sigma_t}{t^2-M\sigma_t^2}\,. \label{eq:lam_tan_final}
\end{align}
As $\sigma_t\to 0$: $(t+M\sigma_t)/(t^2-M\sigma_t^2)\to 1/t$.

For the normal eigenvalues, $\kappa_{\mathrm{norm},j}^2\to 0$ as $\sigma_t\to 0$ (since $X_1\in\cM$ a.s.), giving $\lambda_{\mathrm{norm},j}\to -1/\sigma_t<0$.  These do not contribute to $\mu_2(\J)=\lambda_{\max}$, so~\eqref{eq:mu2_semiconvex} follows.

\medskip
\noindent\textbf{Case $t\le {t_\dagger}$ (compact-support regime).} Since $X_1\in[-\sC_\cM,\sC_\cM]^\sD$ a.s., every eigenvalue of $\Sigma_{\mathrm{post}}$ satisfies
\begin{align}
    \kappa_j^2
    &\;\le\;
    \sC_\cM^2
    \qquad\forall\;j,\;\forall\;\bx,\;t. \label{eq:crude_var_bound}
\end{align}
Substituting into~\eqref{eq:eigvals} and using that $t\mapsto t\,\sC_\cM^2/\sigma_t^3$ is increasing on $[0,1)$:
\begin{align}
    \lambda_j
    &= \frac{t\,\kappa_j^2}{\sigma_t^3} - \frac{1}{\sigma_t}
    \;\le\;
    {\frac{t_\dagger\,\sC_\cM^2}{(1-t_\dagger)^3}}\,. \label{eq:crude_eigval}
\end{align}
This gives~\eqref{eq:mu2_small_t_statement}.

\medskip
\noindent\textbf{Conclusion.} Combining both cases: for all $t\in[0,1)$,
\begin{align}
    \mu_2(\J)
    &\;\le\;
    \max\!\left\{C_0,\;\;\sup_{t>{t_\dagger}}\frac{t+M\sigma_t}{t^2-M\sigma_t^2}\right\}
    \;=:\;
    C_1
    \;<\;\infty. \label{eq:mu2_global}
\end{align}
{Both terms are finite: $C_0<\infty$ since $t_\dagger<1$, and the supremum is finite since $t_\dagger>0$ ensures $(t+M\sigma_t)/(t^2-M\sigma_t^2)\le 1/t_\dagger+O(1)$ at the left endpoint.}  Since $C_1$ is independent of~$t$, we have $\mu_2(\J)\le C_1\le C_1/(1-t)^{1-\xi}$ for any $\xi\in(0,1)$, establishing Assumption~\ref{assume: Lipschitz}.
\end{proof}

{\begin{corollary}[Log-concave special case]\label{cor:logconcave_osl}
If $V$ is geodesically convex ($M_V=0$) and $\cM$ is flat ($C_{\mathrm{curv}}=0$), then $M=0$, $t_\dagger=1/(1+\sC_\cM)$, and
\begin{equation}\label{eq:mu2_logconcave}
    \mu_2(\J)
    \;\le\;
    1+\sC_\cM\,,
    \qquad\forall\;t\in[0,1).
\end{equation}
More generally, if $M_V=0$ but $C_{\mathrm{curv}}>0$ (curved manifold), then $M=C_{\mathrm{curv}}$ and $\mu_2(\J)$ remains uniformly bounded by a constant depending on $C_{\mathrm{curv}}$ and~$\sC_\cM$.
\end{corollary}

\begin{proof}
When $M=0$: the Brascamp--Lieb bound~\eqref{eq:mu2_semiconvex} gives $\mu_2(\J)\le 1/t$ for $t>t_\dagger$.  Since $t_\dagger=1/(1+\sC_\cM)$, this yields $\mu_2(\J)\le 1+\sC_\cM$.  For $t\le t_\dagger$, the compact-support bound $\kappa_j^2\le\sC_\cM^2$ and the full eigenvalue formula $\lambda_j = t\kappa_j^2/\sigma_t^3 - 1/\sigma_t$ give $\lambda_j\le t_\dagger\sC_\cM^2/(1-t_\dagger)^3 - 1/(1-t_\dagger) = 1+\sC_\cM$, where the last step uses $t_\dagger=1/(1+\sC_\cM)$.
\end{proof}}

}

\section{Proof of \Cref{thm: main}}\label{sec: thm main proof}
\begin{proof}[Proof of \Cref{thm: main}]
    Observe when $\xi \ge \sC_{\mathrm{Lip}}/\log\log(n)$, we have 
    $$
    e^{3\sC_{\mathrm{Lip}}/\xi} \le \log^3(n).
    $$
    Using \Cref{thm: vel field estimation}, we write
    \begin{align}
        &\sum_{k=0}^{\sK-1} t_{k}\, \bE_{\cD}\lrl{\int_{\x}\int_{t=1 - t_{k}}^{1 - t_{k+1}} \norm{\widehat{v}(\x,t) - v^\star(\x,t)}^2 \bpi_t(x) \, dt \, d\x}
        \\\nonumber
        = &  \sum_{k=0}^{\sK-1} \one_{\lrm{n^{-\frac{\beta}{2\alpha + \sd}}\log^{\beta}(n) \le \st_k < n^{-\frac{2}{2\alpha + \sd}}}} t_{k}\, \bE_{\cD}\lrl{\int_{\x}\int_{t=1 - t_{k}}^{1 - t_{k+1}} \norm{\widehat{v}(\x,t) - v^\star(\x,t)}^2 \bpi_t(x) \, dt \, d\x}
        \\\nonumber
        & + \sum_{k=0}^{\sK-1} \one_{\lrm{n^{-\frac{2}{2\alpha + \sd}} \le t_k < n^{-\frac{1}{6(2\alpha + \sd)}} \log^{-3}(n)}} t_{k}\, \bE_{\cD}\lrl{\int_{\x}\int_{t=1 - t_{k}}^{1 - t_{k+1}} \norm{\widehat{v}(\x,t) - v^\star(\x,t)}^2 \bpi_t(x) \, dt \, d\x}
        \\\nonumber
        & + \sum_{k=0}^{\sK-1} \one_{\lrm{n^{-\frac{1}{6(2\alpha + \sd)}} \log^{-3}(n) \le t_k < 1}} t_{k}\, \bE_{\cD}\lrl{\int_{\x}\int_{t=1 - t_{k}}^{1 - t_{k+1}} \norm{\widehat{v}(\x,t) - v^\star(\x,t)}^2 \bpi_t(x) \, dt \, d\x}
        \\\nonumber
        \le  & C(\sD, \sC_\cM, \beta)\Bigg( \sum_{k=0}^{\sK-1} \one_{\lrm{n^{-\frac{\beta}{2\alpha + \sd}}\log^{\beta}(n) \le \st_k < n^{-\frac{2}{2\alpha + \sd}}}} t_{k} \lrs{ \frac{n^{-\frac{2\beta}{2\alpha + \sd}}}{t_k} + n^{-\frac{2\alpha}{2\alpha + \sd}}\cdot \log^{\alpha+1}(n) + \frac{\log^2(n)}{n}}
        \\\nonumber
        & + \sum_{k=0}^{\sK-1} \one_{\lrm{n^{-\frac{2}{2\alpha + \sd}} \le t_k < n^{-\frac{1}{6(2\alpha + \sd)}} \log^{-3}(n)}} t_{k}\, \lrs{ \frac{\log^4(n)}{n} + \frac{t_k^{-\sd/2}}{n} \cdot \log^{14+\sd/2}(n)}
        \\\nonumber
        & + \sum_{k=0}^{\sK-1} \one_{\lrm{n^{-\frac{1}{6(2\alpha + \sd)}} \log^{-3}(n) \le t_k < 1}} t_{k}\,\lrs{ \frac{\log^5(n)}{n} + n^{-\frac{(2\alpha + 2)}{2\alpha+\sd}} \cdot \log^{2\sd + 9}(n)} \Bigg)
        \\\nonumber
        = & C(\sD, \sC_\cM, \beta)\Bigg( \sum_{k=0}^{\sK-1} \one_{\lrm{n^{-\frac{\beta}{2\alpha + \sd}}\log^{\beta}(n) \le \st_k < n^{-\frac{2}{2\alpha + \sd}}}}  \lrs{n^{-\frac{2\beta}{2\alpha + \sd}} + n^{-\frac{2\alpha+2}{2\alpha + \sd}}\cdot \log^{\alpha+1}(n) + \frac{\log^2(n)}{n}}
        \\\nonumber
        & + \sum_{k=0}^{\sK-1} \one_{\lrm{n^{-\frac{2}{2\alpha + \sd}} \le t_k < n^{-\frac{1}{6(2\alpha + \sd)}} \log^{-3}(n)}} \lrs{ \frac{\log^4(n)}{n} + \frac{n^{-\frac{2-\sd}{2\alpha + \sd}}}{n} \cdot \log^{14+\sd/2}(n)}
        \\\nonumber
        & + \sum_{k=0}^{\sK-1} \one_{\lrm{n^{-\frac{1}{6(2\alpha + \sd)}} \log^{-3}(n) \le t_k < 1}} \,\lrs{ \frac{\log^5(n)}{n} + n^{-\frac{(2\alpha + 2)}{2\alpha+\sd}} \cdot \log^{2\sd + 9}(n)} \Bigg)
        \\\nonumber
        \le & C^\prime (\sD, \sC_\cM, \beta) \lrs{n^{-\frac{2\beta}{2\alpha + \sd}} + n^{-\frac{(2\alpha + 2)}{2\alpha+\sd}} \cdot \log^{2\sd + 15}(n) + \frac{\log^5(n)}{n}}.
    \end{align}
    Following from \Cref{thm: vel field estimation} and \Cref{lemma: error accumulation}, we write 
    \begin{align}\nonumber
        &\bE_{\cD}\lrl{\mathrm{W}_2(\bpi_1, \widehat{\bpi}_{1-\tbar})}
        \\\nonumber
        \le & \tbar + \sqrt{\lrs{e^{3\sC_{\mathrm{Lip}}/\xi}}\sum_{k=0}^{\sK-1} t_{k}\, \bE_{\cD}\lrl{\int_{\x}\int_{t=1 - t_{k}}^{1 - t_{k+1}} \norm{\widehat{v}(\x,t) - v^\star(\x,t)}^2 \bpi_t(x) \, dt \, d\x}}
        \\\nonumber
        \le & \tbar + \sqrt{\log^3(n) \sum_{k=0}^{\sK-1} t_{k}\, \bE_{\cD}\lrl{\int_{\x}\int_{t=1 - t_{k}}^{1 - t_{k+1}} \norm{\widehat{v}(\x,t) - v^\star(\x,t)}^2 \bpi_t(x) \, dt \, d\x}}
        \\\nonumber
        \le & \underbrace{n^{-\frac{\beta}{2\alpha +\sd}}\log^{\beta}(n)}_{\tbar} +C^\prime (\sD, \sC_\cM, \beta) \lrs{n^{-\frac{\beta}{2\alpha + \sd}} \log^{1.5}(n)  + n^{-\frac{(\alpha + 1)}{2\alpha+\sd}} \cdot \log^{\sd + 9}(n) + \frac{\log^4(n)}{\sqrt{n}}}.   
    \end{align}
\end{proof}

\section{Error accumulation}

{

\begin{theorem}[Wasserstein Distance Bound Under Switching]\label{thm: W under switching}
Let $b_1, b_2 : [0, 1] \times \R^\sD \to \R^\sD$ be measurable functions.  Let $t \in [0, 1]$ and let $p_0$ be a probability distribution on $\R^\sD$ with finite second moment.  Define the processes $(U_t)_{t\in[0,1]}$ and $(V_t)_{t\in[0,1]}$ by
\begin{align*}
    \frac{dU_t}{dt} &= b_1(t, U_t),\quad U_0 \sim p_0, \\[6pt]
    \frac{dV_t}{dt} &= b_2(t, V_t),\quad V_0 \sim p_0.
\end{align*}
Denote by $\mu_t$ and $\nu_t$ the density of $U_t$ and $V_t$ respectively.  If $\bx \mapsto b_1(t, \bx)$ is $\mu_t^{\mathrm{osl}}$-one-sided Lipschitz for each $t$, then for any $t \in [0, 1]$,
\begin{equation}\nonumber
    \rmW_2\bigl(\mu_t, \nu_t\bigr)
    \;\le\;
    \sqrt{\,t}\left(\int_0^t e^{2\int_s^t \mu_u^{\mathrm{osl}}\,du}
    \int_{\bx} \|b_2(s, \bx) - b_1(s, \bx)\|_2^2\;\nu_s(\bx)\,d\bx\,ds\right)^{\!1/2}.
\end{equation}
\end{theorem}

\begin{proof}[Proof of \Cref{thm: W under switching}]
By the definition of the Wasserstein-$2$ distance, coupling pathwise with $U_0=V_0\sim p_0$, it holds
\begin{equation}\nonumber 
    \rmW_2^2(\mu_t,\nu_t)
    \;\le\;
    \int_{\R^\sD}\|U_t(\bx)-V_t(\bx)\|_2^2\;p_0(\bx)\,d\bx
    \;=:\;R_t,
\end{equation}
for any $t\in[0,1]$.  Since $U_0=V_0$, we have $R_0=0$.  By the definitions of the two processes, it follows that
\begin{align}
    \frac{dR_t}{dt}
    &= \int_{\R^\sD} 2\bigl\langle b_1(t,U_t(\bx))-b_2(t,V_t(\bx)),\;U_t(\bx)-V_t(\bx)\bigr\rangle\;p_0(\bx)\,d\bx \nonumber\\[8pt]
    &= \int_{\R^\sD} 2\bigl\langle b_1(t,U_t(\bx))-b_1(t,V_t(\bx)),\;U_t(\bx)-V_t(\bx)\bigr\rangle\;p_0(\bx)\,d\bx \label{eq:switch_term1}\\[4pt]
    &\quad+ \int_{\R^\sD} 2\bigl\langle b_1(t,V_t(\bx))-b_2(t,V_t(\bx)),\;U_t(\bx)-V_t(\bx)\bigr\rangle\;p_0(\bx)\,d\bx. \label{eq:switch_term2}
\end{align}
For term~\eqref{eq:switch_term1}, the one-sided Lipschitz condition on $b_1$ implies
\begin{align}
    \int_{\R^\sD} 2\bigl\langle b_1(t,U_t)-b_1(t,V_t),\;U_t-V_t\bigr\rangle\;p_0\,d\bx
    \;\le\;
    2\,\mu_t^{\mathrm{osl}}\;R_t. \label{eq:switch_osl}
\end{align}
For term~\eqref{eq:switch_term2}, denoting $\delta_t(\bx):=b_1(t,V_t(\bx))-b_2(t,V_t(\bx))$ and $D_t:=\int_{\R^\sD}\|\delta_t(\bx)\|_2^2\;p_0(\bx)\,d\bx$, the Cauchy--Schwarz inequality implies
\begin{align}
    \int_{\R^\sD} 2\bigl\langle\delta_t(\bx),\;U_t(\bx)-V_t(\bx)\bigr\rangle\;p_0(\bx)\,d\bx
    &\;\le\;
    2\sqrt{D_t}\;\sqrt{R_t}\,. \label{eq:switch_cs}
\end{align}
Combining~\eqref{eq:switch_osl} and~\eqref{eq:switch_cs}:
\begin{align}
    \frac{dR_t}{dt}
    &\;\le\;
    2\,\mu_t^{\mathrm{osl}}\,R_t
    \;+\;
    2\sqrt{D_t}\;\sqrt{R_t}\,. \label{eq:switch_Rdot}
\end{align}
Setting $r_t:=\sqrt{R_t}$ (so that $dR_t/dt = 2\,r_t\,\dot{r}_t$) and dividing both sides of~\eqref{eq:switch_Rdot} by $2\,r_t$ (when $r_t>0$):
\begin{align}
    \dot{r}_t
    &\;\le\;
    \mu_t^{\mathrm{osl}}\;r_t + \sqrt{D_t}\,. \nonumber 
\end{align}
By Gr\"onwall's inequality, 
\begin{align}
    r_t
    &\;\le\;
    \int_0^t e^{\int_s^t \mu_u^{\mathrm{osl}}\,du}\;\sqrt{D_s}\;ds. \nonumber 
\end{align}
Squaring both sides and applying Jensen's inequality yields
\begin{align}
    R_t = r_t^2
    &\;\le\;
    t\int_0^t e^{2\int_s^t \mu_u^{\mathrm{osl}}\,du}\;D_s\;ds
    \;=\;
    t\int_0^t e^{2\int_s^t \mu_u^{\mathrm{osl}}\,du}
    \int_{\bx}\|b_1(s,\bx)-b_2(s,\bx)\|_2^2\;\nu_s(\bx)\,d\bx\,ds, \nonumber
\end{align}
where the last equality uses $D_s = \bE_{V_s\sim\nu_s}[\|b_1(s,V_s)-b_2(s,V_s)\|_2^2]$.  The claim follows from $\rmW_2^2(\mu_t,\nu_t)\le R_t$.
\end{proof}
}


\
\\
{
\begin{lemma}[Wasserstein Bound, Different Initials]\label{thm: W under switching2}
Let $b_1 : [0, 1) \times \R^\sD \to \R^\sD$ be a measurable function.  Let $t \in [0, 1)$ and let $p_0, q_0$ be probability distributions on $\R^\sD$ with finite second moment.  Define the processes $(U_t)_{t\in[0,1)}$ and $(V_t)_{t\in[0,1)}$ by
\begin{align*}
    \frac{dU_t}{dt} &= b_1(t, U_t),\quad U_0 \sim p_0, \\[6pt]
    \frac{dV_t}{dt} &= b_1(t, V_t),\quad V_0 \sim q_0.
\end{align*}
Denote by $\mu_t$ and $\nu_t$ the density of $U_t$ and $V_t$ respectively.  If $\bx \mapsto b_1(t, \bx)$ is $\mu_t^{\mathrm{osl}}$-one-sided Lipschitz for each $t$, then for any $t \in (0, 1)$,
\begin{equation}\nonumber
    \rmW_2\bigl(\mu_t, \nu_t\bigr)
    \;\le\;
    e^{\int_0^t \mu_u^{\mathrm{osl}}\,du}\;\rmW_2(p_0, q_0).
\end{equation}
\end{lemma}
\begin{proof}[Proof of \Cref{thm: W under switching2}]
Observe that
\begin{align*}
    \frac{d}{dt}(U_t-V_t) &= b_1(t,U_t)-b_1(t,V_t).
\end{align*}
By the one-sided Lipschitz condition on $b_1$:
\begin{align*}
    \frac{d}{dt}\|U_t-V_t\|_2^2
    &= 2\bigl\langle b_1(t,U_t)-b_1(t,V_t),\;U_t-V_t\bigr\rangle
    \;\le\;
    2\,\mu_t^{\mathrm{osl}}\,\|U_t-V_t\|_2^2.
\end{align*}
By Gr\"onwall's inequality
\begin{align*}
    \|U_t-V_t\|_2^2
    &\;\le\;
    e^{2\int_0^t\mu_u^{\mathrm{osl}}\,du}\;\|U_0-V_0\|_2^2.
\end{align*}
The claim follows from $\rmW_2^2(p_0,q_0)\le\bE\bigl[\|U_0-V_0\|_2^2\bigr]$.
\end{proof}
}


\
\\
{

\begin{lemma}[Wasserstein Bound]\label{thm: W under switching main}
Let $b_1, b_2 : [0, 1) \times \R^\sD \to \R^\sD$ be measurable functions.  Let $t, t_1, t_2 \in [0, 1)$ such that $t > t_2 > t_1$, and let $p_0$ be a probability distribution on~$\R^\sD$.  Define the processes $(U_t)_{t\in[0,1)}$ and $(V_t)_{t\in[0,1)}$ by
\begin{align*}
    \frac{dU_t}{dt} &= b_1(t, U_t)\,\one_{\{t>t_1\}} + b_2(t, U_t)\,\one_{\{t\le t_1\}},\quad U_0 \sim p_0, \\[6pt]
    \frac{dV_t}{dt} &= b_1(t, V_t)\,\one_{\{t>t_2\}} + b_2(t, V_t)\,\one_{\{t\le t_2\}},\quad V_0 \sim p_0.
\end{align*}
Denote by $\mu_t$ and $\nu_t$ the density of $U_t$ and $V_t$ respectively.  If $\bx \mapsto b_1(t, \bx)$ is $\mu_t^{\mathrm{osl}}$-one-sided Lipschitz for each $t$, then for any $t \in (0, 1)$,
\begin{equation}\label{eq:W_combined}
    \rmW_2\bigl(\mu_t, \nu_t\bigr)
    \;\le\;
    e^{\int_{t_2}^t \mu_u^{\mathrm{osl}}\,du}
    \cdot\sqrt{t_2-t_1}\left(\int_{t_1}^{t_2} e^{2\int_s^{t_2}\mu_u^{\mathrm{osl}}\,du}
    \int_{\bx}\|b_2(s,\bx)-b_1(s,\bx)\|_2^2\;\nu_s(\bx)\,d\bx\,ds\right)^{\!1/2}.
\end{equation}
\end{lemma}
\begin{proof}[Proof of \Cref{thm: W under switching main}]
For $t\le t_1$:
\begin{align*}
    \frac{dU_t}{dt} = b_2(t,U_t),\qquad U_0\sim p_0,
    \qquad\qquad
    \frac{dV_t}{dt} = b_2(t,V_t),\qquad V_0\sim p_0.
\end{align*}
Therefore $U_{t_1}\stackrel{d}{=}V_{t_1}$.  For $t_2\ge t>t_1$:
\begin{align*}
    \frac{dU_t}{dt} = b_1(t,U_t),\quad U_{t_1}\sim\mu_{t_1},
    \qquad\qquad
    \frac{dV_t}{dt} = b_2(t,V_t),\quad V_{t_1}\sim\mu_{t_1}.
\end{align*}
By \Cref{thm: W under switching} applied on $[t_1,t_2]$:
\begin{align}
    \rmW_2(\mu_{t_2},\nu_{t_2})
    &\le \sqrt{t_2-t_1}\left(\int_{t_1}^{t_2} e^{2\int_s^{t_2}\mu_u^{\mathrm{osl}}\,du}
    \int_{\bx}\|b_2(s,\bx)-b_1(s,\bx)\|_2^2\;\nu_s(\bx)\,d\bx\,ds\right)^{1/2}. \label{eq:phase2}
\end{align}
Now for $1>t>t_2$:
\begin{align*}
    \frac{dU_t}{dt} = b_1(t,U_t),\quad U_{t_2}\sim\mu_{t_2},
    \qquad\qquad
    \frac{dV_t}{dt} = b_1(t,V_t),\quad V_{t_2}\sim\nu_{t_2}.
\end{align*}
By \Cref{thm: W under switching2} applied on $[t_2,t]$:
\begin{align}
    \rmW_2(\mu_t,\nu_t)
    &\le e^{\int_{t_2}^t\mu_u^{\mathrm{osl}}\,du}\;\rmW_2(\mu_{t_2},\nu_{t_2}). \label{eq:phase3}
\end{align}
The result follows by substituting~\eqref{eq:phase2} into~\eqref{eq:phase3}.
\end{proof}
}

\
\\
\begin{lemma}[Early stopping]\label{lemma: Early stopping}
    For any $t\in [0,1]$, we have:
    $$
    \mathrm{W}_2(\bpi_1, \bpi_{1-t}) \lesssim t.
    $$
\end{lemma}
\begin{proof}
    \begin{align}\nonumber
        \mathrm{W}_2^2(\bpi_1, \bpi_{1-t}) &\le \bE \sbr[2]{\| X_{1-t} - X_1 \|_2^2} 
        \\\label{eq: p11}
        & = \bE \sbr[2]{\| (1-t) X_1 +  \tbar Z - X_1 \|_2^2} 
        \\\nonumber
        & = t^2 \bE \sbr[2]{\|  X_1 - Z \|_2^2}
        \\\nonumber
        & \le t^2 \del[2]{\bE \sbr[2]{\|  X_1\|_2^2} +  \bE \sbr[2]{\| Z \|_2^2}}
        \\\label{eq: p12}
        & \lesssim t^2,
    \end{align}
    where \eqref{eq: p11} follows from \eqref{eq:linear_flow} and \eqref{eq: p12} follows from finiteness of second moments of $X_1$ and $Z$.
\end{proof}

\
\\
{

\begin{lemma}[Error accumulation]\label{lemma: error accumulation}
Suppose $\{t_k\}$ is the time grid as in \eqref{eq: time seq}.  Let $\what{v}(\bx,t)$ be the estimated velocity field obtained with the empirical optimization as in \eqref{eq: opt emp}. Then
\begin{equation}\nonumber
    \bE_\cD\!\left[\rmW_2\bigl(\bpi_1,\,\what{\bpi}_{1-\tbar}\bigr)\right]
    \;\le\;
    \tbar \;+\;
    \left(e^{3\sC_{\mathrm{Lip}}/\xi}\,\sum_{k=0}^{\sK-1}t_k\;\bE_\cD\!\left[\int_{\bx}\!\int_{1-t_k}^{1-t_{k+1}}
    \!\!\bigl\|\what{v}(\bx,t)-v_*(\bx,t)\bigr\|_2^2\;\bpi_t(\bx)\,dt\,d\bx\right]\right)^{\!1/2}\!\!,
\end{equation}
where $\cD$ denotes the training data and $\tbar$ is the early stopping time from \eqref{eq: time seq}.
\end{lemma}
\begin{proof}[Proof of \Cref{lemma: error accumulation}]
Denote $\what{\bpi}_t(\cdot)$ as the density corresponding to $\what{X}_t$. The accurate estimation of the target distribution $\bpi_1$ is facilitated by intermediate processes. Specifically, we define a sequence of intermediate stochastic processes via
\begin{equation}\nonumber
    \frac{d\what{X}_t^{(k)}}{dt}
    = v^\star(\what{X}_t^{(k)},t)\cdot\one_{\{0\le t<1-t_k\}}
    + \what{v}(\what{X}_t^{(k)},t)\cdot\one_{\{1-t_k\le t\le 1-\tbar\}},
    \quad \what{X}_0^{(k)}\sim\mathtt{N}(\bm{0},\bI_\sD),
\end{equation}
for $k=0,\dots,\sK$.  Observe that $\what{X}^{(0)}_{(\cdot)}=\what{X}_{(\cdot)}$ and $\what{X}^{(\sK)}_{(\cdot)}=X_{(\cdot)}$.  By the triangle inequality:
\begin{align}
    \rmW_2(\bpi_1,\what{\bpi}_{1-\tbar})
    &\le \underbrace{\rmW_2(\bpi_1,\bpi_{1-\tbar})}_{\le\;C\tbar\;\text{(Lemma~B.4)}}
    + \sum_{k=0}^{\sK-1}\rmW_2\!\left(\what{\bpi}_{1-\tbar}^{(k)},\what{\bpi}_{1-\tbar}^{(k+1)}\right). \label{eq:triangle}
\end{align}

We denote the density of $\widehat{X}_{t}^{k}$ as \( \widehat{\bpi}^{(k)}_t(\cdot) \). These intermediate processes $\widehat{X}_{t}^{(k)}$ bridge the estimated and the true velocity fields.

\
\\
To quantify this convergence, we employ the Wasserstein metric decomposition:
\begin{equation}\label{eq: Decomposition}
    \begin{split}
            \mathrm{W}_2(\bpi_1, \widehat{\bpi}_{1-\tbar}) &\le \underbrace{\mathrm{W}_2(\bpi_1, \bpi_{1-\tbar})}_{\text{Early stopping}} + \underbrace{\mathrm{W}_2(\bpi_{1-\tbar}, \widehat{\bpi}_{1-\tbar})}_{\text{Error control}} 
            \\
            &\le \mathrm{W}_2(\bpi_1, \bpi_{1-\tbar}) + \sum_{k=0}^{\sK-1} \mathrm{W}_2(\widehat{\bpi}^{(k)}_{1-\tbar}, \widehat{\bpi}^{(k+1)}_{1-\tbar}) .  
    \end{split}
\end{equation}
The \textit{early stopping} term captures the approximation error induced by terminating the flow process slightly earlier than at the target time. Using \Cref{lemma: Early stopping} we write $\mathrm{W}_2^2(\bpi_1, \bpi_{1-\tbar}) \lesssim \tbar^2$.

\
\\
The second term, \textit{error control}, quantifies the discrepancy arising from approximating the velocity field. To control this, we rely on a critical result relating the Wasserstein distance between two distributions to the differences between their corresponding vector fields outline.

We proceed with a detailed error analysis by leveraging the Wasserstein distance bound derived in \Cref{thm: W under switching main}. Specifically, applying \Cref{thm: W under switching main} to each telescoping term with $b_1=v_*$, $b_2=\what{v}$, $t_1=1-t_k$, $t_2=1-t_{k+1}$, $t=1-\tbar$, and noting the boundedness
\begin{align}
    \max\!\left\{e^{\int_{1-t}^{1-t_{k+1}}\mu_u^{\mathrm{osl}}\,du},\;
    e^{\int_{1-t_k}^{1-t_{k+1}}\mu_u^{\mathrm{osl}}\,du}\right\}
    &\;\le\;
    e^{\int_0^1\frac{\sC_{\mathrm{Lip}}}{(1-u)^{1-\xi}}\,du}
    \;=\;
    e^{\sC_{\mathrm{Lip}}/\xi},\nonumber 
\end{align}
which is finite since $\xi>0$, together with $t_k-t_{k+1}\le t_k$, we obtain
\begin{align}
    \rmW_2^2\!\left(\what{\bpi}_{1-\tbar}^{(k)},\what{\bpi}_{1-\tbar}^{(k+1)}\right)
    &\;\le\;
    e^{3\sC_{\mathrm{Lip}}/\xi}\;t_k
    \int_{\bx}\!\int_{1-t_k}^{1-t_{k+1}}\!\!\bigl\|\what{v}(\bx,t)-v^\star(\bx,t)\bigr\|_2^2\;\bpi_t(\bx)\,dt\,d\bx. \label{eq:each_k_bound}
\end{align}
Substituting~\eqref{eq:each_k_bound} into~\eqref{eq:triangle}, summing over $k$, applying Cauchy--Schwarz, taking expectations, and using Jensen's inequality $\bE[\sqrt{X}]\le\sqrt{\bE[X]}$ yields the required result.
\end{proof}
}

\section{Velocity field}
\subsection{Properties}\label{sec: vf expression}
When $X_t = t X_1 + (1-t) X_0$, with $X_1 = Y$ where $Y$ is supported in $\sd-$dimensional boundaryless manifold, we write
\begin{align}\nonumber
    v^\star(\x,t) = \bE \lrl{\dot{X_t} | X_t = \x } &= \bE \lrl{X_1 - X_0 | X_t = \x } = \frac{-1}{1-t}\bE \lrl{X_t - X_1 | X_t = \x } 
    \\\nonumber
    &= -\frac{\x}{1-t} + \frac{1}{1-t}\bE \lrl{X_1 | X_t = \x }
    \\\nonumber
    &= \frac{1}{1-t} \lrl{ \int_{\y \in \cM} \y \, p_{X_1 |X_t} (\y|\x) \,  d\y - \x}
    \\\nonumber    
    & = \frac{1}{1-t} \lrl{ \int_{\y \in \cM} \y \lrs{\frac{e^{-\frac{\lrnorm{\x - t\y}^2_2}{2(1-t)^2}}\bnu(\y)}{\int_{\y \in \cM} e^{-\frac{\lrnorm{\x - t\y}^2_2}{2(1-t)^2}}\bnu(\y)\, \d\y}} d\y - \x}
    \\\nonumber
    & = \frac{1}{1-t} \lrl{ \frac{\int_{\y \in \cM} \y \, e^{-\frac{\lrnorm{\x - t\y}^2_2}{2(1-t)^2}}\bnu(\y)\, d\y}{\int_{\y \in \cM} e^{-\frac{\lrnorm{\x - t\y}^2_2}{2(1-t)^2}}\bnu(\y)\, \d\y}  - \x}
\end{align}
where we used $X_t = t X_1 + (1-t) X_0 \implies -(1-t)^{-1}(X_t - X_1) = (X_1 - X_0)$ and
\begin{equation*}
    p_{X_1 |X_t} (\y|\x) = \frac{p_{X_t|X_1}(\x|\y) \, p_{X_1}(\y)}{\int p_{X_t|X_1}(\x|\y) \, p_{X_1}(\y) \, d\y} = \frac{e^{-\frac{\lrnorm{\x - t\y}^2_2}{2(1-t)^2}}\bnu(\y)}{\int_{\y \in \cM} e^{-\frac{\lrnorm{\x - t\y}^2_2}{2(1-t)^2}}\bnu(\y)\, \d\y}
\end{equation*}
since $X_t|X_1 \sim \mathtt{N}\lrs{t\,X_1,\, (1-t)^2 }$.
Therefore in the noiseless setting, the velocity field expression is
\begin{equation}\label{eq: vel field noiseless}
    v^\star(\x,t) = \frac{1}{1-t} \lrl{ \frac{\int_{\y \in \cM} \y \, e^{-\frac{\lrnorm{\x - t\y}^2_2}{2(1-t)^2}}\bnu(\y)\, d\y}{\int_{\y \in \cM} e^{-\frac{\lrnorm{\x - t\y}^2_2}{2(1-t)^2}}\bnu(\y)\, \d\y}  - \x}
\end{equation}

\paragraph{Optimizer}
The following result and the proof closely follows Theorem 7 of \cite{albergo2023stochastic}.
\begin{lemma}\label{lemma: Optimizer}
    Suppose
    $$
    \calL(u) = \int_0^1 {\bE_{\x\sim X_t}\lrl{\lrnorm{u(\x,t) - \dot{X}_t}_2^2}}\, dt.
    $$
    Then the minimizer of $\calL(u)$ is $v^\star(\x,t) = \bE\lrl{\dot{X}_t|X_t =\x}$. 
\end{lemma}
\begin{proof}
 Define $\epsilon_t = \dot{X}_t - v^\star(\x,t) = \dot{X}_t - \bE\lrl{\dot{X}_t|X_t =\x}$. Note that $\bE\lrl{\epsilon_t|X_t} = 0$. Observe that, for any $u(\x,t)$
 $$
\lrnorm{u(X_t,t) - \dot{X}_t }_2^2 = \lrnorm{u(X_t,t) - v^\star(X_t,t) }_2^2 + \lrnorm{\epsilon_t}_2^2 - 2 \lrangle{u(X_t,t) - v^\star(X_t,t),\,\epsilon_t}.
 $$
 Since $\bE\lrl{\lrangle{u(X_t,t) - v^\star(X_t,t),\,\epsilon_t}} = \bE\lrl{\lrangle{u(X_t,t) - v^\star(X_t,t),\,\bE\lrl{\epsilon_t|X_t}}} = 0$. We write
 $$
 \calL(u) = \calL(\v^\star) + \int_0^1 \bE\lrl{\lrnorm{\epsilon_t}^2_2} \, dt \ge \calL(\v^\star). 
 $$
\end{proof}

\subsection{Estimation}\label{sec: vf estimation}

\begin{proof}[Proof of \Cref{thm: vel field estimation}]
    Recall the time grid as in \eqref{eq: time seq} and the design of the search class $\cU$ as in \eqref{eq: search class}. The optimizer (empirical risk minimizer) $\what{v}(\x, t)$ in \eqref{eq: opt emp} admits the representation 
    \begin{align}
        \what{v}(\cdot,t) = \sum_{k=0}^{\sK-1} \sN_\rho(\cdot,t|, \what{\btheta}_k) \cdot \one_{1-t_k \le t < 1 - t_{k+1}}.
    \end{align}
    where $\what{\btheta}_k \in \bTheta^k_{\sD+1,\sD}(\sL_k, \sW_k, \sS_k, \sB_k)$ is such that $\what{v}(\x, t)$ continuous in $t$. Hence, for $ t \in [1-t_k, 1- t_{k+1})$ we have $\what{v}(\cdot,t) = \sN_\rho(\cdot,t|, \what{\btheta}_k)$. 
    \begin{enumerate}[label=\roman*.]
        \item \textbf{Case \ref{item: vf estim t zero}} Let $k$ be such that $\tbar \le t_k < \st_\sb$. Following from \Cref{lemma: vel field approximation}\ref{item: vf approx t zero} and \Cref{lemma: loss cover}, we write that
        $$
        \log(\cN^{(\delta)}_{\cL_k}) \lesssim n^{\frac{\sd}{2\alpha + \sd}} \log^9(n) \log^{3\vee\sd}(n) \lrs{\log^5(n) + \log(\delta^{-1}) + \log(C^\prime\, \log(n))}.
        $$
        With the choice $\cA = \lrm{\|X_0\|_\infty \le \sqrt{9 \log(n)}}$, we obtain $B_\bG^\cA = C \log^2(n)$ from \Cref{lemma: loss cover}. Using \Cref{lemma: outlier loss}, \Cref{lemma: vel field approximation}\ref{item: vf approx t zero}, $\bP(\cA^c) \le 2\sD n^{-9/2}$, and $\delta = 1/n$ in \Cref{lemma: emp process bound2}, we write 
        \begin{align}\nonumber
            \bE_{\mathcal{D}}\left[ \int_{\x}\int_{t=1 - t_{k}}^{1 - t_{k+1}} \norm{\widehat{v}(\x,t) - v^\star(\x,t)}^2 \bpi_t(x) \, dt \, d\x \right] \lesssim & \underbrace{C(\sD, \sC_\cM, \beta) n^{-9/2} \log^2(n)}_{\Cref{lemma: outlier loss}}
            \\\nonumber
             & \, + \underbrace{\frac{n^{-\frac{2\beta}{2\alpha + \sd}}}{t_k} + n^{-\frac{2\alpha}{2\alpha + \sd}}\cdot \log^{\alpha+1}(n)}_{\Cref{lemma: vel field approximation}\ref{item: vf approx t zero}}
             \\\nonumber
             & \, + \underbrace{n^{-\frac{2\alpha}{2\alpha + \sd}} \log^{16+\sd}(n)}_{\textnormal{Entropy}} + \frac{\log^2(n)}{n} 
             \\\nonumber
             & \, + n \log^2(n) \sD n^{-9/2}.
        \end{align}
        This reduces to
        \begin{align}\nonumber
            &\bE_{\mathcal{D}}\left[ \int_{\x}\int_{t=1 - t_{k}}^{1 - t_{k+1}} \norm{\widehat{v}(\x,t) - v^\star(\x,t)}^2 \bpi_t(x) \, dt \, d\x \right]
            \\\nonumber
            \le &C(\sD, \sC_\cM, \beta) \lrs{ \frac{n^{-\frac{2\beta}{2\alpha + \sd}}}{t_k} + n^{-\frac{2\alpha}{2\alpha + \sd}}\cdot \log^{\alpha+1}(n) + \frac{\log^2(n)}{n}}.
        \end{align}
        \item \textbf{Case \ref{item: vf estim t away zero0}} Let $k$ be such that $\st_\sb \le t_k < n^{-\frac{1}{6(2\alpha+\sd)}}\log^{-3}(n)$. Following from \Cref{lemma: vel field approximation}\ref{item: vf approx t away zero0} and \Cref{lemma: loss cover}, we write that
        $$
        \log(\cN^{(\delta)}_{\cL_k}) \lesssim t_k^{-\sd/2} \log^9(n) \log^{\sd/2}(n) \lrs{\log^5(n) + \log(\delta^{-1}) + \log(C^\prime\, \log(n))}.
        $$
        Using the last display and similar to the last case we obtain
        \begin{align}\nonumber
            &\bE_{\mathcal{D}}\left[ \int_{\x}\int_{t=1 - t_{k}}^{1 - t_{k+1}} \norm{\widehat{v}(\x,t) - v^\star(\x,t)}^2 \bpi_t(x) \, dt \, d\x \right]
            \\\nonumber
            \le &C(\sD, \sC_\cM, \beta) \lrs{ \frac{\log^4(n)}{n} + \frac{t_k^{-\sd/2}}{n} \cdot \log^{14+\sd/2}(n)}.
        \end{align}

        \item \textbf{Case \ref{item: vf estim t away zero}} Let $k$ be such that $n^{-\frac{1}{6(2\alpha+\sd)}}\log^{-3}(n)  \le t_k < t_0$. Following from \Cref{lemma: vel field approximation}\ref{item: vf approx t away zero} and \Cref{lemma: loss cover}, we write that
        $$
        \log(\cN^{(\delta)}_{\cL_k}) \lesssim n^{\frac{\sd}{6(2\alpha + \sd)}} \log^{2\sd+6}(n) \lrs{\log^3(n) + \log(\delta^{-1}) + \log(C^\prime\, \log(n))}.
        $$
        Using the last display and similar to the last case we obtain
        \begin{align}\nonumber
            &\bE_{\mathcal{D}}\left[ \int_{\x}\int_{t=1 - t_{k}}^{1 - t_{k+1}} \norm{\widehat{v}(\x,t) - v^\star(\x,t)}^2 \bpi_t(x) \, dt \, d\x \right]
            \\\nonumber
            \le &C(\sD, \sC_\cM, \beta) \lrs{ \frac{\log^5(n)}{n} + n^{-\frac{(2\alpha + 5\sd/6)}{2\alpha+\sd}} \cdot \log^{2\sd + 9}(n)}
            \\\nonumber
            \le &C(\sD, \sC_\cM, \beta) \lrs{ \frac{\log^5(n)}{n} + n^{-\frac{(2\alpha + 2)}{2\alpha+\sd}} \cdot \log^{2\sd + 9}(n)},
        \end{align}
        where the last inequality follows since $\sd\ge 3$.
    \end{enumerate}
\end{proof}

\paragraph{Properties of the loss function} 
\begin{lemma}[Cover]\label{lemma: loss cover}
    Let $\cA = \lrm{ \|X_0\|_\infty \le \sqrt{9\, \log(n)}}$ and $n \ge 2$. Suppose that $X_1 \sim \bpi_1(\cdot)$ (with $\|X_1\|_\infty \le \sC_\cM $) and $X_0 \sim \mathtt{N}(\bm{0},\bI_\sd)$, and define the interpolation $X_t = tX_1 + (1-t) X_0$ for $ t\in [0,1]$. Denote
    $$
        \ell_{\btheta_k}\lrs{X_1, X_0} \cdot \one_{\lrm{\cA}} = \int_{1-t_k}^{1-t_{k+1}} \lrnorm{\sN_{\rho} \lrs{X_t, t \big| \btheta_k} - \lrs{X_1 - X_0}}_2^2 \, dt \, \cdot \one_{\lrm{\cA}},
    $$
    for $\btheta_k \in \bTheta_{\sD+1,\sD}(\sL_k, \sW_k, \sS_k, \sB_k)$ such that $\sN_{\rho} \lrs{\cdot \big| \btheta_k}$ is neural network satisfying the uniform bound $ \| {\sN_{\rho} \lrs{\cdot,t \big| \btheta_k}} \|_\infty \lesssim \sqrt{\tfrac{\log(n)}{1-t}}$, and $t_{k}, t_{k+1}$ as in \eqref{eq: time seq} for $k =0,\ldots, \sK -1$. 

    \
    \\
    Denote the appropriate function class as
    $$
    \cL_k = \lrm{\ell_{\btheta_k}\lrs{X_1, X_0} \cdot \one_{\lrm{\cA}} \, :\, \btheta_k \in \bTheta_{\sD+1,\sD}(\sL_k, \sW_k, \sS_k, \sB_k), \| {\sN_{\rho} \lrs{\cdot,t \big| \btheta_k}} \|_\infty \lesssim \sqrt{\frac{\log(n)}{1-t}}, t_{k}, t_{k+1} \textnormal{ as in } \eqref{eq: time seq}}.
    $$
    Then for any $\delta \le 1$
    $$
       \log \lrs{\cN^{(\delta)}_{\cL_k}} = \log\lrs{\cN^{(\delta/C^\prime)}_{\bTheta_{\sD+1,\sD}}} \lesssim \sS_k \sL_k\lrs{\log(\sL_k \sB_k \sW_k) + \log(\delta^{-1}) + \log(C^\prime\, \log(n))}  
    $$
    where $\cN^{(\delta)}_{\cL_k} = \cN\lrs{\delta, \cL_k, \|\cdot \|_\infty }$ denote the covering number of $\cL_k$ in the $\|\cdot \|_\infty$ norm, and $C^\prime = C^\prime(\sD, \sC_\cM, \beta) >0 $ is a universal constant depending only on $\beta$, $\sC_\cM$ and $\sD$. Moreover,
    $$
    0 \le \ell_{\btheta_k}\lrs{X_1, X_0} \le C(\sD, \sC_\cM, \beta) \log^2(n),
    $$
    where $C = C(\sD, \sC_\cM, \beta) >0 $ is a universal constant depending only on $\beta$, $\sC_\cM$ and $\sD$.
\end{lemma}
\begin{proof}
    Observe that
    \begin{align}\nonumber
        0 \le \ell_{\btheta_k}\lrs{X_1, X_0} \cdot \one_{\cA} & \le 3 \lrs{\int_{1-t_k}^{1-t_{k+1}} \lrnorm{{\sN_{\rho} \lrs{\cdot,t \big| \btheta}}}_2^2 \, dt + \int_{1-t_k}^{1-t_{k+1}} \lrnorm{X_1}_2^2 \, dt + \int_{1-t_k}^{1-t_{k+1}} \lrnorm{X_0}_2^2 \, dt}  \cdot \one_{\cA}
        \\ \nonumber
        & \lesssim 3 \lrs{ \sD \int_0^{1-\tbar} \frac{\log(n)}{1-t} \, dt  + \sD (t_k - t_{k+1}) \lrnorm{X_1}_\infty^2 + (t_k - t_{k+1}) \lrnorm{X_0}_2^2 }\cdot \one_{\cA}
        \\\nonumber
        & \le 3 \lrs{ \frac{\beta\, \sD}{2\alpha + \sd} \log^2(n)   + \sD\, \sC_\cM^2 + \lrnorm{X_0}_2^2 }\cdot \one_{\cA},
        \\\nonumber
        & \le 3 \lrs{ \frac{\beta\, \sD}{2\alpha + \sd} \log^2(n)   + \sD\, \sC_\cM^2 + 9 \sD \log(n) }
        \\\label{eq: 7aa}
        & \le 3 \lrs{ {\beta\, \sD} \log^2(n)   + \sD\, \sC_\cM^2 + 9 \log(n) } =  C(\sD, \sC_\cM, \beta) \log^2(n).
    \end{align}
    where the first inequality follows from the identity $(a+b+c)^2 \le 3(a^2 + b^2 + c^2)$, and in the second and third inequality we uses $\|\cdot\|_2^2 \le \sD \, \|\cdot\|_\infty^2$.
    
    \
    \\
    Let $\btheta_k, \btheta_k^\prime \in \bTheta^k_{\sD+1,\sD} = \bTheta_{\sD+1, \sD}\lrs{\sL_k, \sW_k, \sS_k, \sB_k}$ such that $\lrnorm{\sN_\rho(\cdot,\cdot |\btheta) - \sN_\rho(\cdot,\cdot |\btheta^\prime)}_\infty < \delta$. Then
    \begin{align}\nonumber
        &\lrs{\ell_{\btheta_k} - \ell_{\btheta_k^\prime}}
        \\\nonumber
        =&  \int_{1-t_k}^{1-t_{k+1}} \lrs{ \lrnorm{\sN_{\rho} \lrs{X_t,t \big| \btheta_k} - \lrs{X_1 - X_0}}_2^2 - \lrnorm{\sN_{\rho} \lrs{X_t,t \big| \btheta_k^\prime} - \lrs{X_1 - X_0}}_2^2} \, dt
        \\\nonumber
        =& \int_{1-t_k}^{1-t_{k+1}} \lrs{ \lrnorm{\sN_{\rho} \lrs{X_t,t \big| \btheta_k} - \sN_{\rho} \lrs{X_t ,t \big| \btheta_k^\prime}}_2^2 +  \lrangle{\sN_{\rho} \lrs{X_t,t \big| \btheta_k} - \sN_{\rho} \lrs{X_t,t \big| \btheta_k^\prime}, \, \sN_{\rho} \lrs{X_t,t \big| \btheta_k^\prime} - \lrs{X_1 - X_0}}} \, dt
    \end{align}
    where the last display follows from $(b-a)^2 - (c-a)^2 = (b-c)^2 + 2(b-c)(c-a)$. Below we bound the expressions in the last display. Observe that
    \begin{align}\nonumber
        \int_{1-t_k}^{1-t_{k+1}} \lrnorm{\sN_{\rho} \lrs{X_t,t \big| \btheta_k} - \sN_{\rho} \lrs{X_t,t \big| \btheta_k^\prime}}_2^2 \, dt \le \sD \lrnorm{\sN_\rho(\cdot,\cdot |\btheta_k) - \sN_\rho(\cdot,\cdot |\btheta_k^\prime)}_\infty^2 \int_{1-t_k}^{1-t_{k+1}} \, dt \le \sD\, \delta^2,
    \end{align}
    and 
    \begin{align*}
        &\int_{1-t_k}^{1-t_{k+1}} { \lrangle{\sN_{\rho} \lrs{X_t,t \big| \btheta_k} - \sN_{\rho} \lrs{X_t,t \big| \btheta_k^\prime}, \, \sN_{\rho} \lrs{X_t,t \big| \btheta_k^\prime} - \lrs{X_1 - X_0}}} \one_{\cA} \, dt
        \\
        & \le \sqrt{\int_{1-t_k}^{1-t_{k+1}} \lrnorm{\sN_{\rho} \lrs{X_t,t \big| \btheta_k} - \sN_{\rho} \lrs{X_t,t \big| \btheta_k^\prime}}_2^2 \, dt}\, \sqrt{\ell_{\btheta_k^\prime}\one_{\cA}} \le \delta \sqrt{\sD} \sqrt{C(\sD, \sC_\cM, \beta)} \log(n)
    \end{align*}
    where the last display follows from Cauchy-Schwarz inequality and \eqref{eq: 7aa}. This allows us to write
    $$
    \lrvert{\lrs{\ell_{\btheta_k} - \ell_{\btheta_k^\prime}}\cdot \one_{\cA}} \le C^\prime(\sD, \sC_\cM, \beta) \lrs{\delta + \delta ^2},
    $$
    also $\delta^2 \le \delta $ provided that $\delta \le 1$.

    \
    \\
    Therefore
    $$
        \cN\lrs{\delta, \cL_k, \|\cdot \|_\infty } \le \cN\lrs{\delta/(C^\prime \log(n)), \bTheta^k_{\sD+1,\sD}, \|\cdot \|_\infty }.
    $$
    The required result now follows from (see e.g., Lemma 3 in \cite{suzuki2018adaptivity})
    \begin{equation}\label{eq: cover NN}
            \log \cN^{(\delta)} = \log\cN(\delta, \Theta_{\sD+1,\sD}, |\cdot|_\infty ) \lesssim  \,\sS\,\sL \{ \log(\sL\sB\sW) + \log \delta^{-1} \}.
    \end{equation}
\end{proof}

\
\\

\begin{lemma}\label{lemma: outlier loss}
    Let $\cA = \lrm{ \|X_0\|_\infty \le \sqrt{9\, \log(n)}}$ and $n \ge 2$. Suppose that $X_1 \sim \bpi_1(\cdot)$ (with $\|X_1\|_\infty \le \sC_\cM $) and $X_0 \sim \mathtt{N}(\bm{0},\bI_\sd)$, and define the interpolation $X_t = tX_1 + (1-t) X_0$ for $ t\in [0,1]$. Consider the loss function 
    $$
    \ell_{\btheta_k}\lrs{X_1, X_0} = \int_{1 - t_k}^{1-t_{k+1}} \lrnorm{\sN_{\rho} \lrs{X_t,t \big| \btheta} - \lrs{X_1 - X_0}}_2^2 \, dt,
    $$
    where $\sN_{\rho} \lrs{\cdot, t \big| \btheta_k}$ is neural network satisfying the uniform bound $ \| {\sN_{\rho} \lrs{\cdot, t \big| \btheta_k}} \|_\infty \lesssim \sqrt{\tfrac{\log(n)}{1-t}}$ for $\btheta_k \in \bTheta_{\sD+1,\sD}(\sL_k, \sW_k, \sS_k, \sB_k)$, and $t_k$ and $t_{k+1}$ as in \eqref{eq: time seq}. Then, the expected loss satisfies
    $$
        \bE\lrl{\ell_{\btheta_k}\lrs{X_1, X_0} \cdot \one_{\cA^c}} \le C n^{-9/2} \log^2(n), 
    $$
    where $C > 0$ is a universal constant depending only on $\beta$, $\sC_\cM$ and $\sD$.
\end{lemma}

\begin{proof}
    Observe that
    \begin{align}\nonumber
        \ell_{\btheta_k}\lrs{X_1, X_0} & \le 3 \lrs{\int_0^{1-\tbar} \lrnorm{{\sN_{\rho} \lrs{X_t,t \big| \btheta_k}}}_2^2 \, dt + \int_0^{1-\tbar} \lrnorm{X_1}_2^2 \, dt + \int_0^{1-\tbar} \lrnorm{X_0}_2^2 \, dt}  
        \\ \nonumber
        & \lesssim 3 \lrs{ \sD \int_0^{1-\tbar} \frac{\log(n)}{1-t} \, dt  + \sD \lrnorm{X_1}_\infty^2 + \lrnorm{X_0}_2^2 }
        \\\nonumber
        & \le 3 \lrs{ \frac{\beta\, \sD}{2\alpha + \sd} \log^2(n)   + \sD\, \sC_\cM^2 + \lrnorm{X_0}_2^2 },
    \end{align}
    where the first inequality follows from the identity $(a+b+c)^2 \le 3(a^2 + b^2 + c^2)$, and in the second and third inequality we uses $\|\cdot\|_2^2 \le \sD \, \|\cdot\|_\infty^2$.
    
    \
    \\
    Since $X_0 \sim \mathtt{N}(\bm{0},\bI_\sd)$, we have 
    $$
    \bP\lrs{\cA^c} = \bP\lrs{\|X_0\|_\infty \ge \sqrt{9\, \log(n)}} \le 2\,\sD\, n^{-4.5}.
    $$
    Therefore, the expectation of the loss under event $\cA$ satisfies
    \begin{align}\nonumber
        \bE\lrl{\ell_{\btheta_k}\lrs{X_1, X_0} \cdot \one_{\cA^c}} &\lesssim 3 \lrs{\beta\sD \log^2(n) + \sD\, \sC_\cM^2} \bP\lrs{\cA^c} + \bE\lrl{ \lrnorm{X_0}_2^2 \one_{\lrm{\lrnorm{X_0}_\infty \ge \sqrt{9 \log(n)}}} } 
        \\\nonumber
        & \lesssim 6 \sD^2 \lrs{\beta \log^2(n) + \sC^2_\cM} n^{-9/2} + n^{-9/2} \lrs{4 \sD^2 \log(n)},
    \end{align}
    where the last inequality uses the tail bound
    $$
        \bE\lrl{ \lrnorm{X_0}_2^2 \one_{\lrm{\lrnorm{X_0}_\infty \ge \sqrt{9 \log(n)}}} } \le \frac{2}{\sqrt{2\pi}} n^{-9/2} \lrs{\sD \sqrt{9\log(n)} + \frac{\sD^2}{\sqrt{9\log(n)}}} \le n^{-9/2} \lrs{4 \sD^2 \log(n)}
    $$
    which follows from \Cref{lemma: tail G}. This completes the proof.
\end{proof}

\subsection{Approximation}


\begin{lemma}[Velocity field approximation]\label{lemma: vel field approximation}
    Suppose $t \in [1 - \st_\sA, 1 - \st_\sZ]$ with $1 <  \frac{\st_\sA}{\st_\sZ} \le 2$ as in \eqref{eq: time seq}. Then
    \begin{enumerate}[label=\Alph*.]
        \item \label{item: vf approx t zero} For $ n^{-\frac{\beta}{2\alpha + \sd}}\log^{\beta}(n) \le \st_\sA \le n^{-\frac{2}{2\alpha + \sd}}$, there exists a network $\btheta_{\mathrm{vel}} \in \bTheta_{\sd+1, \sd}(\sL, \sW, \sS, \sB)$ satisfying
        $$
            \int_{1  - \st_\sA}^{1 - \st_\sZ} \int_{\bR^\sD} \lrnorm{ \sN_{\rho}(\x, t| \btheta_{\mathrm{vel}}) - v^\star(\x,t) }_2^2 \, \bpi_{t}(\x) \, d\x \, dt \,\lesssim\,  \frac{n^{-\frac{2\beta}{2\alpha + \sd}}}{\st_\sA} + n^{-\frac{2\alpha}{2\alpha + \sd}}\cdot \log^{\alpha+1}(n),
        $$
        with
        $$
        \lrvert{\sN_{\rho}\lrs{\x, t| \btheta_{\mathrm{vel}}}}_\infty \lesssim \frac{\sqrt{\log(n)}}{\st_\sA}
        $$
        where
        $$
        \sL = \cO\lrs{\log^4(n)},\; \sW = \cO\lrs{n^{\frac{\sd}{2\alpha + \sd}} \log^{{\lrs{\max\{6,\, 3+\sd\}}}}(n)},\; \sS = \cO\lrs{n^{\frac{\sd}{2\alpha + \sd}} \log^{\lrs{\max\{8,\, 5+\sd\}}}(n)},\, \sB = e^{\cO\lrs{\log^4(n)}}. 
        $$
        

        \item\label{item: vf approx t away zero0} For $ n^{-\frac{2}{2\alpha + \sd}} \le \st_\sA \le n^{-\frac{1}{6(2\alpha + \sd)}} \log^{-3}(n)$, there exists a network $\btheta_{\mathrm{vel}} \in \bTheta_{\sd+1, \sd}(\sL, \sW, \sS, \sB)$ satisfying
        $$
             \int_{1 - \st_\sA}^{1 - \st_\sZ} \int_{\bR^\sD} \lrnorm{ \sN_{\rho}(\x, t| \btheta_{\mathrm{vel}}) - v^\star(\x,t)}_2^2 \, \bpi_{t}(\x) \, d\x \, dt \,\lesssim\,  \frac{\log^4(n)}{n},
        $$
        with
        $$
        \lrvert{\sN_{\rho}\lrs{\x, t| \btheta_{\mathrm{vel}}}}_\infty \lesssim \frac{\sqrt{\log(n)}}{\st_\sA}
        $$
        where
        \begin{align*}
            &\sL = \cO\lrs{\log^4(n)}, \quad \sW = \cO\lrs{ \lrs{\st_\sA \log(n)}^{-\sd/2}\,\lrl{\log^6(n) + \log^{\sd+3}(n)\,\fL\, \binom{\fL + D}{D}} },
            \\
            &\sS = \cO\lrs{ \lrs{\st_\sA \log(n)}^{-\sd/2}\,\lrl{\log^8(n) + \log^{5+\sd}(n)\,\fL\, \binom{\fL + D}{D} }}, \quad \sB = e^{\cO\lrs{\log^4(n)}}, 
            \\
            &\textnormal{and}\quad \fL = \frac{-\log(\sqrt{n})}{\log\lrs{{\st_\sA}\sqrt{\log^3(n)}}}.
        \end{align*}


        \item \label{item: vf approx t away zero} For $ n^{-\frac{1}{6(2\alpha + \sd)}} \log^{-3}(n) \le \st_\sA < 1$, there exists a network $\btheta_{\mathrm{vel}} \in \bTheta_{\sd+1, \sd}(\sL, \sW, \sS, \sB)$ satisfying
        $$
            \int_{1 - \st_\sA}^{1 - \st_\sZ} \int_{\bR^\sD} \lrnorm{ \sN_{\rho}(\x, t| \btheta_{\mathrm{vel}}) - v^\star(\x,t) }_2^2 \, \bpi_{t}(\x) \, d\x  \, dt\,\lesssim\,  \frac{\log^4(n)}{n},
        $$
        with
        $$
        \lrvert{\sN_{\rho}\lrs{\x, t| \btheta_{\mathrm{vel}}}}_\infty \lesssim \frac{\sqrt{\log(n)}}{\st_\sA}
        $$
        where
        \begin{align*}
            &\sL = \cO\lrs{{\log^2(n)}},\qquad \sW = \cO\lrs{ {n^{\frac{\sd}{6(2\alpha+\sd)}} \log^{2\sd}(n)} \cdot \max\lrm{\log^3(n),\, \binom{\sD + 6(2\alpha + \sd)}{\sD}}},
            \\
            &\sS = \cO\lrs{ {n^{\frac{\sd}{6(2\alpha+\sd)}} \log^{2\sd+1}(n)}\cdot \max\lrm{\log^3(n),\, \binom{\sD + 6(2\alpha + \sd)}{\sD}}}, \qquad \sB = e^{\cO\lrs{{\log^2(n)}}}.
        \end{align*}

    \end{enumerate}
\end{lemma}
\begin{proof}
    Recall \eqref{eq: score vel} 
    $$
    v^\star(\x,t) = \frac{\x}{t} + \lrs{\frac{1-t}{t}} \nabla_{\x} \log \bpi_t(\x)
    $$
    where,
    $$
        \nabla_\x \log \bpi_t(\x) = \frac{-1}{1 - t} \left[ \frac{\int_{\y \in \cM} \,\lrs{\frac{\x - t \, \y}{1 - t}}\,\bnu(\y)\, e^{-\frac{|\x - t\, \y|_2^2}{2(1-t)^2}} \, d\y}{\int_{\y \in \cM} \, \bnu(\y) e^{-\frac{|\x - t\, \y|_2^2}{2(1-t)^2}} \, d\y}\right].
    $$
    \textbf{Case \ref{item: vf approx t zero}}
    Following \Cref{cor: score approximation no delta}\ref{item: score approx t zero} with $\tau = 1- t$, we may find network $\btheta_{\mathrm{score}}$ such that
    \begin{align*}
        & \int_{\bR^\sD} \lrnorm{ \sN_{\rho}(\x, t| \btheta_{\mathrm{score}}) - \nabla_\x \log \bpi_t(\x)}_2^2 \, \bpi_{t}(\x) \, d\x  \,\lesssim\,  \frac{n^{-\frac{2\beta}{2\alpha + \sd}} \cdot \log^{\beta+1}(n) }{(1-t)^4} + \frac{n^{-\frac{2\alpha}{2\alpha + \sd}}\cdot \log^{\alpha+1}(n)}{(1-t)^2},
    \end{align*}
    which led us to
        \begin{align*}
        & \int_{\bR^\sD} \lrnorm{ \frac{\x}{t} +  \frac{1-t}{t}\sN_{\rho}(\x, t| \btheta_{\mathrm{score}}) - v^\star(\x,t)  }_2^2 \, \bpi_{t}(\x) \, d\x  
        \\
        \lesssim &\,  \lrs{\frac{1-t}{t}}^2\lrs{\frac{n^{-\frac{2\beta}{2\alpha + \sd}} \cdot \log^{\beta+1}(n) }{(1-t)^4} + \frac{n^{-\frac{2\alpha}{2\alpha + \sd}}\cdot \log^{\alpha+1}(n)}{(1-t)^2}}
        \\
        = & \, {\frac{n^{-\frac{2\beta}{2\alpha + \sd}} \cdot \log^{\beta+1}(n) }{t^2 (1-t)^2} + \frac{n^{-\frac{2\alpha}{2\alpha + \sd}}\cdot \log^{\alpha+1}(n)}{t^2}}
        \\
        \\
        \le & \, 4 \frac{n^{-\frac{2\beta}{2\alpha + \sd}} \cdot \log^{\beta+1}(n)}{(1-t)^2} + \frac{n^{-\frac{2\alpha}{2\alpha + \sd}}\cdot \log^{\alpha+1}(n)}{t^2} 
    \end{align*}
    where the first line follows from \eqref{eq: score vel} and the last line follows from $1/2 < t$. Integrating both side we obtain
    \begin{align*}
        \int_{1-\st_\sA}^{1-\st_\sZ}  \int_{\bR^\sD} \lrnorm{ \frac{\x}{t} +  \frac{1-t}{t}\sN_{\rho}(\x, t| \btheta_{\mathrm{score}}) - v^\star(\x,t)  }_2^2 \, \bpi_{t}(\x) \, d\x \, dt &\lesssim\, \frac{n^{-\frac{2\beta}{2\alpha + \sd}} \cdot \log^{\beta+1}(n)}{\st_\sZ} + \frac{n^{-\frac{2\alpha}{2\alpha + \sd}}\cdot \log^{\alpha+1}(n)}{1 - \st_\sA}
        \\
        & \lesssim \, \frac{n^{-\frac{2\beta}{2\alpha + \sd}} \cdot \log^{\beta+1}(n)}{\st_\sA} + n^{-\frac{2\alpha}{2\alpha + \sd}}\cdot \log^{\alpha+1}(n),
    \end{align*}
    which follows from $\st_\sA/\st_\sZ \le 2$ and $\st_\sA < 1/2$. The remaining task is to construct a network by adding extra component to the network $\btheta_\mathrm{score}$ to efficiently estimate $\frac{\x}{t} +  \frac{1-t}{t}\sN_{\rho}(\x, t| \btheta_{\mathrm{score}})$ such that
    $$
    \lrnorm{ \sN_\rho\lrs{\x,t|\btheta_{\mathrm{vel}}} - \lrs{\frac{\x}{t} + \frac{1-t}{t} \sN_\rho\lrs{\x,t|\btheta_{\mathrm{score}}}}}_\infty \le \sqrt{\frac{\log(n)}{n}}.
    $$
    To achieve that:
    \begin{itemize}
        \item We approximate $(1-t)\sN_\rho\lrs{\x,t|\btheta_{\mathrm{score}}}$ using \Cref{lem:basic ReLU A}\ref{lem: ReLU prod A} by adding a network (in series/padding) with parameters $\sL = \cO\lrs{\log(n)}$, $\sW = \cO(1)$ with a error rate of $1/\sqrt{n}$.
        \item To approximate $\x + (1-t)\sN_\rho\lrs{\x,t|\btheta_{\mathrm{score}}} $ requires just adding $\x$ to the network obtained in the previous step, therefore the construction is exact. Overall error remains $\cO\lrs{1/\sqrt{n}}$ with net parameters $\sL = \cO\lrs{\log(n)}$ and $\sW = \cO(1)$ for the added network.
        \item We first approximate $1/t$ using \Cref{lemma: NN 1/x} (recall $t \ge 1/2$) with network with parameters $\sL = \cO\lrs{\log^2(n)}, \sW = \cO\lrs{\log^3(n)}, \sS = \cO\lrs{\log^4(n)}, \sB = \cO(1/n^2)$, up to an error rate of $\st_\sA/\sqrt{n}$. Then we use a product network to approximate the $t^{-1}(\x + (1-t)\sN_\rho\lrs{\x,t|\btheta_{\mathrm{score}}})$ by approximating the product $t^{-1}$ network and the network obtained in the last step; which required a network with parameters $\sL = \log(n)$ and $\sW = \cO(1)$ and $B = e^{\cO(\log(n))}$. The obtained error rate is $\sqrt{\log(n)/n}$. This completes the construction of $\btheta_{\mathrm{vel}}$
    \end{itemize}
    Overall we do no required network parameters in larger order than that of $\btheta_{\mathrm{score}}$.

    \
    \\
    \textbf{Case \ref{item: vf approx t away zero0}:}
    Following from \Cref{cor: score approximation no delta}\ref{item: score approx t away zero0} with $\tau = 1- t$, this case follows very similar to derivations of Case \ref{item: vf approx t zero}, and is therefore omitted.

    \
    \\
    \textbf{Case \ref{item: vf approx t away zero}:}
    Following \Cref{cor: score approximation no delta}\ref{item: score approx t away zero} with $\tau = 1- t$, this case follows very similar to derivations of Case \ref{item: vf approx t zero}, and is therefore omitted.

\end{proof}

\paragraph{Relating velocity and score}
\begin{lemma}[Tweedie's Formula]\label{lemma: tweedie}
Suppose \( U \sim \mu \) and \( \epsilon \sim \mathsf{N}(0, \sigma^2 \bI_d) \). Let \( V = U + \epsilon \). Then, the marginal density of \( V \), denoted as \( p(\v) \), satisfies the following equation:
\[
\bE\left[ U \mid V = \v \right] = \v + \sigma^2 \nabla_\v \log p(\v).
\]
\end{lemma}
\begin{proof}
    Observe that
    $$
        p_{U |V} (\u|\v) = \frac{p_{V|U}(\v|\u) \, p_{U}(\u)}{\int p_{V|U}(\v|\u) \, p_{U}(\u) \, d\u} = \frac{e^{-\frac{\lrnorm{\v - \u}^2_2}{2\sigma^2}}\mu(\u)}{\int e^{-\frac{\lrnorm{\v - \u}^2_2}{2\sigma^2}}\mu(\u)\, \d\u} \implies \bE\left[ U \mid V = \v \right] = \frac{ \int \u\, e^{-\frac{\lrnorm{\v - \u}^2_2}{2\sigma^2}}\mu(\u)\, \d\u}{\int e^{-\frac{\lrnorm{\v - \u}^2_2}{2\sigma^2}}\mu(\u)\, \d\u}  ,
    $$
    which follows from $V|U \sim \mathtt{N}\lrs{U, \sigma^2 \bI_d}$. And with use of $p(\v) = \int p_{V|U}(\v|\u) \, p_{U}(\u) \, d\u$, we write
    \begin{align}
        \nabla_\v  \log p(\v) = \frac{p^\prime(\v)}{p(\v)} = \frac{\frac{d}{d\v} \lrs{\int e^{-\frac{\lrnorm{\v - \u}^2_2}{2\sigma^2}}\mu(\u)\, \d\u}}{\int e^{-\frac{\lrnorm{\v - \u}^2_2}{2\sigma^2}}\mu(\u)\, \d\u} = &\frac{ {\int \lrs{-\frac{\v - \u}{\sigma^2}} e^{-\frac{\lrnorm{\v - \u}^2_2}{2\sigma^2}}\mu(\u)\, \d\u}}{\int e^{-\frac{\lrnorm{\v - \u}^2_2}{2\sigma^2}}\mu(\u)\, \d\u},
        \\
        =& -\frac{\v}{\sigma^2} + \frac{1}{\sigma^2} \frac{ {\int \, \u \, e^{-\frac{\lrnorm{\v - \u}^2_2}{2\sigma^2}}\mu(\u)\, \d\u}}{\int e^{-\frac{\lrnorm{\v - \u}^2_2}{2\sigma^2}}\mu(\u)\, \d\u}
        \\
        =& \frac{1}{\sigma^2} \lrs{-\v + \bE\lrs{U|V=\v}}.
    \end{align}
    Rearranging provides us with the needed result.
\end{proof}

\
\\
Recall
$$
v^\star(\x,t) = \frac{1}{1-t} \lrs{-\x + \bE\lrl{X_1 | X_t = \x}}.
$$
We have $X_t = t\,X_1 + (1-t)X_0$ with $X_0 \sim \mathtt{N\lrs{\mathbf{0}_\sD, \bI_\sD}}$. With the use of Tweedie's formula (\Cref{lemma: tweedie}), we write
$$
\bE\lrl{t X_1 | X_t = \x} = \x + (1-t)^2 \nabla_{\x} \log \bpi_t(\x),
$$
where $\bpi_t$  is density of $X_t$. Therefore
\begin{equation}\label{eq: score vel}
    v^\star(\x,t) = \frac{\x}{t} + \lrs{\frac{1-t}{t}} \nabla_{\x} \log \bpi_t(\x)
\end{equation}

\paragraph{Score approximation}
Define
    $$
        p_\tau(\x) = \frac{1}{\lrs{\sqrt{2\pi\,\tau^2}}^\sd} \int_{\y \in \cM} \, \bnu(\y) e^{-\frac{|\x -(1-\tau)\y|_2^2}{2\tau^2}} \, d\y, \qquad\qquad \tau \in (0,1]
    $$
and
    $$
        \nabla_\x \log p_\tau(\x) = \frac{-1}{\tau} \left[ \frac{\int_{\y \in \cM} \,\lrs{\frac{\x - (1-\tau)\y}{\tau}}\,\bnu(\y)\, e^{-\frac{|\x - (1-\tau)\y|_2^2}{2\tau^2}} \, d\y}{\int_{\y \in \cM} \, \bnu(\y) e^{-\frac{|\x - (1-\tau)\y|_2^2}{2\tau^2}} \, d\y}\right]
    $$
\begin{lemma}[Lemma B.3. of \cite{tang2024adaptivity}]\label{lemma: score approximation}
    Suppose $\tau \in [\st_\sA, \st_\sZ]$ with $1 < \frac{\st_\sA}{\st_\sZ} \le 2$. Then
    \begin{enumerate}[label=\Alph*.]
        \item   For $ n^{-\frac{\beta}{2\alpha + \sd}}\log^{\beta}(n) \le \st_\sA \le n^{-\frac{2}{2\alpha + \sd}}$, there exists a network $\btheta_{\mathrm{score}} \in \bTheta_{\sd, \sd}(\sL, \sW, \sS, \sB)$ satisfying
        $$
             \int_{\bR^\sD} \lrnorm{ \sN_{\rho}(\x, \tau| \btheta_{\mathrm{score}}) - \nabla_\x \log p_\tau(\x)}_2^2 \, p_{\tau}(\x) \, d\x \,\lesssim\,  \frac{n^{-\frac{2\beta}{2\alpha + \sd}} \cdot \log^{\beta+1}(n) }{\tau^4} + \frac{n^{-\frac{2\alpha}{2\alpha + \sd}}\cdot \log^{\alpha+1}(n)}{\tau^2},
        $$
        with
        $$
        \lrvert{\sN_{\rho}\lrs{\x, t| \btheta_{\mathrm{score}}}}_\infty \lesssim \frac{\sqrt{\log(n)}}{\st_\sA}
        $$
        where
        $$
        \sL = \cO\lrs{\log^4(n)},\; \sW = \cO\lrs{n^{\frac{\sd}{2\alpha + \sd}} \log^{{\lrs{\max\{6,\, 3+\sd\}}}}(n)},\; \sS = \cO\lrs{n^{\frac{\sd}{2\alpha + \sd}} \log^{\lrs{\max\{8,\, 5+\sd\}}}(n)},\, \sB = e^{\cO\lrs{\log^4(n)}}. 
        $$
        \item Let $\delta \in \lrl{\frac{3 \log\log(n)}{\log(n)},\, \frac{2}{2\alpha + \sd} - \frac{\log\log(n)}{\log(n)} }$.
        \begin{enumerate}[label=(\roman*)]
        \item For $n^{-\frac{2}{2\alpha + \sd}} \le \st_\sA \le n^{-2\delta} \log^{-3}(n)$, there exists a network $\btheta_{\mathrm{score}} \in \bTheta_{\sd, \sd}(\sL, \sW, \sS, \sB)$ satisfying
        $$
             \int_{\bR^\sD} \lrnorm{ \sN_{\rho}(\x, \tau| \btheta_{\mathrm{score}}) - \nabla_\x \log p_\tau(\x)}_2^2 \, p_{\tau}(\x) \, d\x  \,\lesssim\,  \frac{\log^4(n)}{n},
        $$
        with
        $$
        \lrvert{\sN_{\rho}\lrs{\x, t| \btheta_{\mathrm{score}}}}_\infty \lesssim \frac{\sqrt{\log(n)}}{\st_\sA}
        $$
        where
        \begin{align*}
            &\sL = \cO\lrs{\log^4(n)}, \quad \sW = \cO\lrs{ \lrs{\st_\sA \log(n)}^{-\sd/2}\,\lrl{\log^6(n) + \log^{\sd+3}(n)\,\fL\, \binom{\fL + D}{D}} },
            \\
            &\sS = \cO\lrs{ \lrs{\st_\sA \log(n)}^{-\sd/2}\,\lrl{\log^8(n) + \log^{5+\sd}(n)\,\fL\, \binom{\fL + D}{D} }}, \quad \sB = e^{\cO\lrs{\log^4(n)}}, 
            \\
            &\textnormal{and}\quad \fL = \frac{-\log(\sqrt{n})}{\log\lrs{{\st_\sA}\sqrt{\log^3(n)}}}.
        \end{align*}
        
        \item For $ n^{-2\delta} \log^{-3}(n) \le \st_\sA \le 1$, there exists a network $\btheta_{\mathrm{score}} \in \bTheta_{\sd, \sd}(\sL, \sW, \sS, \sB)$ satisfying
        $$
             \int_{\bR^\sD} \lrnorm{ \sN_{\rho}(\x, \tau| \btheta_{\mathrm{score}}) - \nabla_\x \log p_\tau(\x)}_2^2 \, p_{\tau}(\x) \, d\x \,\lesssim\,  \frac{\log^4(n)}{n},
        $$
        with
        $$
        \lrvert{\sN_{\rho}\lrs{\x, t| \btheta_{\mathrm{score}}}}_\infty \lesssim \frac{\sqrt{\log(n)}}{\st_\sA}
        $$
        where
        \begin{align*}
            &\sL = \cO\lrs{\frac{\log^2(n)}{\delta^2}},\qquad \sW = \cO\lrs{ \frac{n^{2\delta\sd} \log^{2\sd}(n)}{\delta^3}\cdot \max\lrm{\log^3(n),\, \binom{D + (1/2\delta)}{D}}},
            \\
            &\sS = \cO\lrs{ \frac{n^{2\delta\sd} \log^{2\sd+1}(n)}{\delta^4}\cdot \max\lrm{\log^3(n),\, \binom{D + (1/2\delta)}{D}}}, \qquad \sB = e^{\cO\lrs{\frac{\log^2(n)}{\delta^2}}}.
        \end{align*}
        \end{enumerate}
    \end{enumerate}
\end{lemma}

\begin{proof}
    \Cref{lemma: score approximation} is a restatement of Lemma B.3 of \cite{tang2024adaptivity} with the following modest simplifications in their statement and proof:
    \begin{itemize}
        \item \textbf{Score not integrated in time:} We rewrite the score approximation error bound statement for fixed $\tau \in [\st_\sA, \st_\sB]$. This is the original direct consequence of their proof.  
        \item \textbf{Change in mean:} We choose $m_\tau = 1 - \tau$. This a simpler choice of mean, the result follows with appropriate and simplified modifications in their proof.
        \item \textbf{Change in variance:} They have $\sigma_\tau = \cO\lrs{\sqrt{\tau} \vee 1}$. We choose $\sigma_\tau = \tau$. Again, this a simpler choice of variance, the result follows with appropriate and simplified modifications in their proof.
    \end{itemize}
\end{proof}

\begin{corollary}\label{cor: score approximation no delta}
    Suppose $\tau \in [\st_\sA, \st_\sZ]$ with $1 < \frac{\st_\sA}{\st_\sZ} \le 2$. Then
    \begin{enumerate}[label=\Alph*.]
        \item\label{item: score approx t zero}   For $ n^{-\frac{\beta}{2\alpha + \sd}}\log^{\beta}(n) \le \st_\sA \le n^{-\frac{2}{2\alpha + \sd}}$, there exists a network $\btheta_{\mathrm{score}} \in \bTheta_{\sd, \sd}(\sL, \sW, \sS, \sB)$ satisfying
        $$
             \int_{\bR^\sD} \lrnorm{ \sN_{\rho}(\x, \tau| \btheta_{\mathrm{score}}) - \nabla_\x \log p_\tau(\x)}_2^2 \, p_{\tau}(\x) \, d\x \,\lesssim\,  \frac{n^{-\frac{2\beta}{2\alpha + \sd}} \cdot \log^{\beta+1}(n) }{\tau^4} + \frac{n^{-\frac{2\alpha}{2\alpha + \sd}}\cdot \log^{\alpha+1}(n)}{\tau^2},
        $$
        with
        $$
        \lrvert{\sN_{\rho}\lrs{\x, t| \btheta_{\mathrm{score}}}}_\infty \lesssim \frac{\sqrt{\log(n)}}{\st_\sA}
        $$
        where
        $$
        \sL = \cO\lrs{\log^4(n)},\; \sW = \cO\lrs{n^{\frac{\sd}{2\alpha + \sd}} \log^{{\lrs{\max\{6,\, 3+\sd\}}}}(n)},\; \sS = \cO\lrs{n^{\frac{\sd}{2\alpha + \sd}} \log^{\lrs{\max\{8,\, 5+\sd\}}}(n)},\, \sB = e^{\cO\lrs{\log^4(n)}}. 
        $$

        \item\label{item: score approx t away zero0} For $ n^{-\frac{2}{2\alpha + d}} \le \st_\sA \le n^{-\frac{1}{6(2\alpha + d)}} \log^{-3}(n)$, there exists a network $\btheta_{\mathrm{score}} \in \bTheta_{\sd, \sd}(\sL, \sW, \sS, \sB)$ satisfying
        $$
             \int_{\bR^\sD} \lrnorm{ \sN_{\rho}(\x, \tau| \btheta_{\mathrm{score}}) - \nabla_\x \log p_\tau(\x)}_2^2 \, p_{\tau}(\x) \, d\x \,\lesssim\,  \frac{\log^4(n)}{n},
        $$
        with
        $$
        \lrvert{\sN_{\rho}\lrs{\x, t| \btheta_{\mathrm{score}}}}_\infty \lesssim \frac{\sqrt{\log(n)}}{\st_\sA}
        $$
        where
        \begin{align*}
            &\sL = \cO\lrs{\log^4(n)}, \quad \sW = \cO\lrs{ \lrs{\st_\sA \log(n)}^{-\sd/2}\,\lrl{\log^6(n) + \log^{\sd+3}(n)\,\fL\, \binom{\fL + D}{D}} },
            \\
            &\sS = \cO\lrs{ \lrs{\st_\sA \log(n)}^{-\sd/2}\,\lrl{\log^8(n) + \log^{5+\sd}(n)\,\fL\, \binom{\fL + D}{D} }}, \quad \sB = e^{\cO\lrs{\log^4(n)}}, 
            \\
            &\textnormal{and}\quad \fL = \frac{-\log(\sqrt{n})}{\log\lrs{{\st_\sA}\sqrt{\log^3(n)}}}.
        \end{align*}

        \item\label{item: score approx t away zero} For $ n^{-\frac{1}{6(2\alpha + d)}} \log^{-3}(n) \le \st_\sA \le 1$, there exists a network $\btheta_{\mathrm{score}} \in \bTheta_{\sd, \sd}(\sL, \sW, \sS, \sB)$ satisfying
        $$
             \int_{\bR^\sD} \lrnorm{ \sN_{\rho}(\x, \tau| \btheta_{\mathrm{score}}) - \nabla_\x \log p_\tau(\x)}_2^2 \, p_{\tau}(\x) \, d\x \,\lesssim\,  \frac{\log^4(n)}{n},
        $$
        with
        $$
        \lrvert{\sN_{\rho}\lrs{\x, t| \btheta_{\mathrm{score}}}}_\infty \lesssim \frac{\sqrt{\log(n)}}{\st_\sA}
        $$
        where
        \begin{align*}
            &\sL = \cO\lrs{{\log^2(n)}},\qquad \sW = \cO\lrs{ {n^{\frac{\sd}{6(2\alpha+\sd)}} \log^{2\sd}(n)} \cdot \max\lrm{\log^3(n),\, \binom{\sD + 6(2\alpha + \sd)}{\sD}}},
            \\
            &\sS = \cO\lrs{ {n^{\frac{\sd}{6(2\alpha+d)}} \log^{2\sd+1}(n)}\cdot \max\lrm{\log^3(n),\, \binom{\sD + 6(2\alpha + \sd)}{\sD}}}, \qquad \sB = e^{\cO\lrs{{\log^2(n)}}}.
        \end{align*}
    \end{enumerate}
\end{corollary}

\begin{proof}
    \Cref{cor: score approximation no delta} follows from \Cref{lemma: score approximation} with $\delta = \tfrac{1}{12(2\alpha + d)}$.
\end{proof}


\section{Empirical process results}
The following \Cref{lemma: emp process bound} outlines an empirical process technique in M-estimation and is based on \citep[Theorem C.4]{suzuki023diffusion}. It separates the the minimization into two components, the bias and the variance. 
\begin{lemma}[Risk bound]\label{lemma: emp process bound}
    Let $\delta>0$. Let $\bG = \lft\{ g \; : \; g : \cZ \subset \real^d \to \real_{\ge0} \text{ and } \|g\|_\infty  = \sup_{\z \in \cZ} g(\z) < B_\bG  \rgt\}$. Let $\calN^{(\delta)}_{\bG}$ be the $\delta-$covering number of $\bG$ with respect to the $\|\cdot\|_\infty$ norm. Suppose  $B_\bG \ge 1$ and $e < \calN^{(\delta)}_{\bG} < \infty$. Suppose we have i.i.d data $\calD = \{Z_j\}_{j=1}^n$ (with $Z_j \in \cZ$) and
    $$
    \what{g}  = \argmin_{g\in \bG} \frac{1}{n} \sum_{j=1}^n g(Z_j).
    $$
    Then we have
    $$
        \bE_\calD \lft[ \bE_{\z} \lft[ \what{g}(\z) \rgt] \rgt] \le 2 \inf_{g \in \bG} \bE_{\z} [g(\z)] + \frac{148 B_\bG \log \left(\calN^{(\delta)}_{\bG}\right)}{9\,n} + \frac{64 B_\bG}{n} + \frac{64 \, B_\bG}{n \, \log\lft(\calN^{(\delta)}_{\bG} \rgt)} + 5\delta,
    $$
    where $\bE_{\z}[\cdot]$ expectation with respect to data point $\z$ independent from the data $\calD$.
\end{lemma}

\
\\
\begin{proof}
    Let $\calD^\prime = \{Z_j^\prime\}_{j=1}^n $ be a ghost sample (identical and independent). With slight abuse of notation, for any function $g\in \bG$, denote $g^{(n)}({\bf{Z}}) = n^{-1} \sum_{j=1}^n g(Z_j)$, $g^{(n)}({\bf{Z^\prime}}) = n^{-1} \sum_{j=1}^n g(Z_j^\prime)$, and $g^{(n)}({\bf{Z}}, {\bf{Z^\prime}}) = g^{(n)}({\bf{Z}}) - g^{(n)}({\bf{Z^\prime}})$.

    Observe that we may write
    \begin{equation}\label{eq: T0}
                \bE_\calD \lft[ \bE_\z \lft[ \what{g}(\z) \rgt] \rgt] =  \bE_\calD \lft[ \bE_{\calD^\prime} \lft[  \what{g}^{(n)}({\bf{Z^\prime}}) \rgt] \rgt] =   \bE_{\calD, \calD^\prime} \lft[  \what{g}^{(n)}({\bf{Z^\prime}}) \rgt]        
    \end{equation}
    Denote 
    $$
    \Lambda = \Big|  \bE_\calD \lft[ \bE_{\calD^\prime} \lft[ \what{g}^{(n)}({\bf{Z}}) - \what{g}^{(n)}({\bf{Z^\prime}}) \rgt] \rgt] \Big| = \Big|  \bE_{\calD, \calD^\prime}   \lft[ \widehat g^{(n)}({\bf{Z}}, {\bf{Z^\prime}}) \rgt] \Big|
    $$

\
\\
{\bf Step~$1$:}
\\
Let $\{g_k\}_{k=1}^{\calN^{(\delta)}_{\bG}}$ be the $\delta-$cover of $\bG$. Fix a positive number $\Theta$ to be specified later and denote $r_k^2 = \max \lft\{\Theta^2, \lft|\bE_{\calD, \calD^\prime} \lft[  \gnkp \rgt] \rgt| \rgt\}$. Observe that
\begin{equation}\label{eq: tt21}
    \begin{aligned}
          \textnormal{for all } k = 1,\ldots, \calN^{(\delta)}_{\bG}, \quad  &\max\left\{ \frac{g_k(Z_j)}{r_k}, \frac{g_k(Z_j^\prime)}{r_k} \right\} \le \frac{\,B_\bG}{\Theta} 
         \\
         \implies  &\lft| \frac{\gnkd}{r_k} \rgt| \le \frac{2\,B_\bG}{\Theta}, \textnormal{ and } \lft| \frac{\gnkp}{r_k} \rgt| \le \frac{\,B_\bG}{\Theta};
    \end{aligned}
\end{equation}
and 
\begin{equation}\label{eq: tt22}
    \begin{aligned}
        \frac{1}{n} \var\lft( \sum_{j=1}^n\frac{g_k(Z_j) - g_k(Z_j^\prime)}{r_k} \rgt) &= \frac{1}{n} \sum_{j=1}^n \var\lft( \frac{g_k(Z_j) - g_k(Z_j^\prime)}{r_k} \rgt) = \frac{1}{n} \sum_{j=1}^n \bE\left[\left|\frac{g_k(Z_j) - g_k(Z_j^\prime)}{r_k} \right|^2 \right] 
        \\
        &\le \frac{4}{n} \bE\left[\sum_{j=1}^n  \left|\frac{g_k(Z_j^\prime)}{r_k} \right|^2 \right] \le 4 {B_\bG} \bE\left[\frac{\gnkp}{r_k^2} \right] \le 4\,B_\bG.
    \end{aligned}
\end{equation}
Using Bernstein inequality in \Cref{lemma: Bernstein} and the observation in \eqref{eq: tt21} and \eqref{eq: tt22}, we may write for any $t \ge 36\, \Theta^2$
\begin{equation}\label{eq: T1}
    \bbP\lft[ \lft|\frac{\gnkd}{r_k} \rgt| \ge \sqrt{t} \rgt] \le 2\, e^{ -\frac{nt}{2B_\bG\lft( 4 + \frac{2 \sqrt{t}}{3\Theta} \rgt)}}  \le 2\, e^{-\frac{3\,\Theta\,n\sqrt{t}}{8B_\bG}},
\end{equation}
where the last inequality follows from $(a+b) \le 2 \max\{a,b\}$.

\
\\
{\bf Step~$2$:}

    Let $1 \le \K \le \calN^{(\delta)}_{\bG}$ be random such that $g_\K$ is $\delta-$close to $\what{g}$. We have
    \begin{align}\nonumber
        \Lambda =& \bE_{\calD, \calD^\prime} \lft[ \lft| \wgnd \rgt| \rgt] 
        \\\nonumber
        \le& \bE_{\calD, \calD^\prime} \lft[ r_\K  \lft|\frac{{g}^{(n)}_\K  ({\bf{Z}}, {\bf{Z^\prime}})}{r_\K}\rgt| \rgt] + 2\delta 
        \\\label{eq: T2}
        \le& \frac{1}{2} \underbrace{\bE_{\calD, \calD^\prime} \lft[ r_\K^2 \rgt]}_{\mathrm{I}} + \frac{1}{2} \underbrace{ \bE_{\calD, \calD^\prime} \lft[ \lft|\frac{ {g}^{(n)}_\K  ({\bf{Z}}, {\bf{Z^\prime}}) }{r_\K}\rgt|^2 \rgt]}_{\mathrm{II}} + 2\delta
    \end{align}
Now we are going to bound $\mathrm{I}$ and $\mathrm{II}$ from the right side of \eqref{eq: T2}. For $\mathrm{I}$, observe from the definition of $r_\K$

\begin{align}\label{eq: T3}
    \bE_{\calD, \calD^\prime} \lft[ r_\K^2 \rgt] \le \Theta^2 + \lft|\bE_\calD \lft[ {g}^{(n)}_\K({\bf{Z}}) \rgt] \rgt| \le \Theta^2 + \lft|\bE_{\calD, \calD^\prime} \lft[ \wgn\rgt] \rgt| + \delta.
\end{align}
For $\mathrm{II}$, with $\alpha \ge 36\, \Theta^2$ to be specified later, observe that
\begin{align}\nonumber
     \bE_{\calD, \calD^\prime} \lft[ \lft|\frac{ {g}^{(n)}_\K  ({\bf{Z}}, {\bf{Z^\prime}}) }{r_\K}\rgt|^2 \rgt] \le& \bE_{\calD, \calD^\prime} \lft[ \max_{1\le k\le  \calN^{(\delta)}_{\bG}}\lft|\frac{ \gnkd}{r_k}\rgt|^2 \rgt]
     \\\nonumber
     \le & \int_{0}^\infty \bbP\lft[ \max_{1\le k\le \calN^{(\delta)}_{\bG}}\lft|\frac{ \gnkd}{r_k}\rgt| > \sqrt{t} \rgt] \; dt
     \\\nonumber
     \le& \alpha + \int_{\alpha}^\infty \bbP\lft[ \max_{1\le k\le  \calN^{(\delta)}_{\bG} }\lft|\frac{ \gnkd}{r_k}\rgt| > \sqrt{t} \rgt] \; dt
     \\\label{eq: T4}
     \le& \alpha + 2\,\left( \calN^{(\delta)}_{\bG}\right) \, \int_{\alpha}^\infty  e^{-\frac{3\,\Theta\,n\sqrt{t}}{8B_\bG}} \; dt
     \\\nonumber
     =& \alpha + 4\,\left(\calN^{(\delta)}_{\bG}\right) \; \lft[ e^{-\frac{3\,\Theta\,n\sqrt{\alpha}}{8B_\bG}}  \lft\{\frac{1}{ \lft( 3\Theta n/8 B_\bG \rgt)^2} + \frac{\sqrt{\alpha}}{ \lft( 3\Theta n/8 B_\bG) \rgt)}  \rgt\} \rgt] \tag{choosing $\alpha = 36\,\Theta^2 = \frac{16 B_\bG \log \left(\calN^{(\delta)}_\bG\right)}{n} $}
     \\\label{eq: T5}
     = & \frac{16 B_\bG \log \left(\calN^{(\delta)}_{\bG}\right)}{n} + \frac{64 \, B_\bG}{n \, \log\left(\calN^{(\delta)}_{\bG}\right)} + \frac{64 B_\bG}{n}, 
\end{align}
    where \eqref{eq: T4} follows from \eqref{eq: T1}.

\
\\
    Bringing together \eqref{eq: T3} and \eqref{eq: T5} to \eqref{eq: T2}, and from the definition of $\Lambda$, we get
    \begin{align}\nonumber
        &\bE_{\calD,\calD^\prime}\lft[ \wgnp \rgt] - \bE_{\calD}\lft[ \wgn \rgt] 
        \\\nonumber
        \le \; \Lambda \; \le & \;  \frac{1}{2} \left(\mathrm{I} + \mathrm{II}\right) + 2\delta
        \\\nonumber
        \le & \frac{1}{2} \left(\Theta^2 + \lft|\bE_{\calD, \calD^\prime} \lft[ \wgn\rgt] \rgt| + \delta \right) + \frac{1}{2}\left( \frac{16 B_\bG \log \left(\calN^{(\delta)}_{\bG}\right)}{n} + \frac{64 \, B_\bG}{n \, \log\left(\lft|\calN^{(\delta)}_{\bG} \rgt|\right)} + \frac{64 B_\bG}{n}\right) + 2\delta
        \\\nonumber
        = &\frac{1}{2} \left(\frac{4 B_\bG \log \left(\calN^{(\delta)}_{\bG}\right)}{9\,n} + \lft|\bE_{\calD, \calD^\prime} \lft[ \wgn\rgt] \rgt| + \delta \right) + \frac{1}{2}\left( \frac{16 B_\bG \log \left(\calN^{(\delta)}_{\bG}\right)}{n} + \frac{64 \, B_\bG}{n \, \log\left(\lft|\calN^{(\delta)}_{\bG} \rgt|\right)} + \frac{64 B_\bG}{n}\right) + 2\delta
        \\\nonumber
        \le & \frac{74 B_\bG \log \left(\calN^{(\delta)}_{\bG}\right)}{9\,n} + \frac{1}{2} \bE_{\calD,\calD^\prime}\lft[ \wgnp \rgt] + \frac{32 \, B_\bG}{n \, \log\lft(\calN^{(\delta)}_{\bG} \rgt)} + \frac{32 B_\bG}{n} + \frac{5\delta}{2}
        \\\label{eq: T6}
        \implies &\bE_{\calD,\calD^\prime}\lft[ \wgnp \rgt] \le 2 \bE_{\calD}\lft[ \wgn \rgt] + \frac{148 B_\bG \log \left(\calN^{(\delta)}_{\bG}\right)}{9\,n} + \frac{64 B_\bG}{n} + \frac{64 \, B_\bG}{n \, \log\lft(\calN^{(\delta)}_{\bG} \rgt)} + 5\delta,
    \end{align}
    where the second line follows from \eqref{eq: T2}, third line follows from \eqref{eq: T3} and \eqref{eq: T5}, fourth line follows by substituting the expression for $\bTheta$.

    \medskip
    The result now follows from \eqref{eq: T0}, \eqref{eq: T6} and the realization
    \begin{align*}
        \bE_{\calD}\lft[ \wgn \rgt] &= \bE_{\calD}\lft[ \frac{1}{n} \sum_{j=1}^n \what{g}(Z_j) \rgt] = \bE_{\calD}\lft[ \inf_{g\in G} \frac{1}{n} \sum_{j=1}^n g(Z_j) \rgt] 
        \\
        &\le \inf_{g\in G} \bE_{\calD}\lft[ \frac{1}{n} \sum_{j=1}^n g(Z_j) \rgt] = \inf_{g \in \bG} \bE_{\z} [g(\z)].        
    \end{align*}
    
\end{proof}

\subsection{Risk bound (on a high probability event)}
The following \Cref{lemma: emp process bound2} outlines an empirical process technique in M-estimation. It extends \Cref{lemma: emp process bound} for the case when the loss function is bounded on a high probability set $\cA$. 
\begin{lemma}\label{lemma: emp process bound2}
    Let $\cA \subset \cZ$ with $\bP(\cA) > 0$ and let $\bG$ be class of functions $g: \cZ \subset \bR^d \to \bR_{\ge 0}$. Assume
    $$
        \sup_{g \in \bG} \lrnorm{g}_{\infty, \cA} = \sup_{g \in \bG} \sup_{\z \in \cA} g(\z) \le B_\bG^\cA < \infty.
    $$
    Let $\bG_\cA := \lrm{g\one_{\cA} : g \in \bG}$, and for some $\delta > 0$ assume $e < \cN^{(\delta)}_{\bG_\cA} < \infty$, where the cover is in $\lrnorm{\cdot}_{\infty, \cA} $ over $\cA$. Suppose we have i.i.d data $\calD = \{Z_j\}_{j=1}^n$ (with $Z_j \in \cZ$) and
    $$
    \what{g}  = \argmin_{g\in \bG} \frac{1}{n} \sum_{j=1}^n g(Z_j).
    $$
    Then we have
    \begin{equation*}
        \begin{aligned}
            \bE_{\cD}\lrl{\bE_\z\lrl{\what{g}(\z)}} \le \sup_{g \in \bG} \bE_\z\lrl{{g}(\z) \one_{\cA^c}(\z)} &+ 2 \inf_{g\in \bG} \bE_{z}[g(\z)] + \frac{148 B_\bG^\cA \log \left(\calN^{(\delta)}_{\bG_\cA}\right)}{9\,n} + \frac{64 B_\bG^\cA}{n} 
            \\
            &+ \frac{64 \, B_\bG^\cA}{n \, \log\lft(\calN^{(\delta)}_{\bG_\cA} \rgt)} + 5\delta + n\, B^\cA_\bG\, \bP\lrs{\cA^c}.
        \end{aligned}
    \end{equation*}
\end{lemma}

\
\\
\begin{proof}
    Define the event $\cE = \lrm{Z_j \in \cA:\; j=1,\ldots, n}$, and let $Q = \bP(\cdot|\cA)$. Note that $Q^n  = Q^{\otimes n}= \bP(\cdot|\cE)$. Using the definition of the restricted function class $\bG_\cA$, we can write
    $$
        \what{g}_\cA = \argmin_{g\in \bG} \frac{1}{n} \sum_{j=1}^n g(Z_j)\, \one_{\cA}(Z_j) = \argmin_{g\in \bG_\cA} \frac{1}{n} \sum_{j=1}^n g(Z_j).
    $$
    Moreover, on the event $\cE$. $\what{g}_\cA = \what{g}\,\one_{\cA}$ where
    \begin{equation}\label{eq: 12ab}
        \what{g}_\cA = \argmin_{g\in \bG} \frac{1}{n} \sum_{j=1}^n g(Z_j)\, \one_{\cA}(Z_j) \qquad \text{and} \qquad \what{g} = \argmin_{g\in \bG} \frac{1}{n} \sum_{j=1}^n g(Z_j).
    \end{equation}
    
    \
    \\
    Since $ \lrnorm{g}_{\infty, \cA} =  \sup_{\z\in \cA} g(\z) \le B_{\bG}^\cA$, for all $g \in \bG$. Using \Cref{lemma: emp process bound}, with the the i.i.d.\ sample $\{Z_j|\cA\}_{j=1}^n$ drawn from the conditional distribution $Q^n$, we obtain
    \begin{equation*}
        \bE_{Q^n}\lrl{\bE_{\z \sim Q}\lrl{\what{g}_\cA(\z)}} \le 2 \inf_{g_\cA \in \bG_\cA} \bE_{\z\sim Q}\lrl{g_\cA(\z)} + \frac{148 B_\bG^\cA \log \left(\calN^{(\delta)}_{\bG_\cA}\right)}{9\,n} + \frac{64 B_\bG^\cA}{n} + \frac{64 \, B_\bG^\cA}{n \, \log\lft(\calN^{(\delta)}_{\bG_\cA} \rgt)} + 5\delta
    \end{equation*}
    This bound can be rewritten as
    \begin{equation*}
        \bE_{\cD}\lrl{\bE_{\z}\lrl{\what{g}(\z) \one_{\cA}(\z)}\,\big|\,\cE} \le 2 \inf_{g\in \bG} \bE_{z}[g(\z)] + \frac{148 B_\bG^\cA \log \left(\calN^{(\delta)}_{\bG_\cA}\right)}{9\,n} + \frac{64 B_\bG^\cA}{n} + \frac{64 \, B_\bG^\cA}{n \, \log\lft(\calN^{(\delta)}_{\bG_\cA} \rgt)} + 5\delta
    \end{equation*}
    where we used the identities
    \[
    \mathbb{E}_{Q^{n}}[\cdot]
    =
    \mathbb{E}_{\mathcal{D}}[\cdot \mid \mathcal{E}],
    \qquad
    \mathbb{E}_{\z \sim Q}[f(\z)]
    =
    \mathbb{E}_{\z}[f(\z) \mid \mathcal{A}]
    =
    \frac{\mathbb{E}_{\z}\big[f(\z)\mathbf{1}_{\mathcal{A}}(\z)\big]}{\mathbb{P}(\mathcal{A})},
    \]
    together with the definition ${g}_{\mathcal{A}} = {g}\,\one_{\mathcal{A}}$ for $g \in {\bG}_{\mathcal{A}}$, the observation at \eqref{eq: 12ab}, and the fact that $0 < {\bP}(\mathcal{A}) \le 1$. This bound can be further reduced to
    \begin{equation}\label{eq: 12ac}
        \bE_{\cD}\lrl{ \one_{\cE} \, \bE_{\z}\lrl{\what{g}(\z) \one_{\cA}(\z)}} \le 2 \inf_{g\in \bG} \bE_{z}[g(\z)] + \frac{148 B_\bG^\cA \log \left(\calN^{(\delta)}_{\bG_\cA}\right)}{9\,n} + \frac{64 B_\bG^\cA}{n} + \frac{64 \, B_\bG^\cA}{n \, \log\lft(\calN^{(\delta)}_{\bG_\cA} \rgt)} + 5\delta,
    \end{equation}
    using similar identity as the last bound and  $0 < {\bP}(\mathcal{\cE}) \le 1$.
    
    \
    \\
    For any random $\what{g} \in \bG$ that depends only on the data $\cD$ ( i.e., is $\sigma(\cD)-$measurable), we can decompose 
    \begin{equation}\label{eq: 12aa}
        \bE_{\cD}\lrl{ \bE_\z\lrl{\what{g}(\z)} } = \bE_{\cD}\lrl{ \bE_\z\lrl{\what{g}(\z) \one_{\cA^c}(\z)} } + \bE_{\cD}\lrl{ \one_{\cE}\, \bE_\z\lrl{\what{g}(\z) \one_{\cA}(\z)} } + \bE_{\cD}\lrl{ \one_{\cE^c}\, \bE_\z\lrl{\what{g}(\z) \, \one_{\cA}(\z)} }    
    \end{equation}
    Observe that 
    \begin{equation}\label{eq: 12ad}
        \bE_{\cD}\lrl{ \one_{\cE^c}\, \bE_\z\lrl{\what{g}(\z) \, \one_{\cA}(\z)} } \le B^\cA_\bG\, \bP\lrs{\cE^c} \le n\, B^\cA_\bG\, \bP\lrs{\cA^c},
    \end{equation}
    and 
    \begin{equation}\label{eq: 12ae}
        \bE_{\cD}\lrl{ \bE_\z\lrl{\what{g}(\z) \one_{\cA^c}(\z)}} \le \sup_{g \in \bG} \bE_\z\lrl{{g}(\z) \one_{\cA^c}(\z)}.
    \end{equation}
    Finally bringing together \eqref{eq: 12ac}, \eqref{eq: 12aa}, \eqref{eq: 12ad}, \eqref{eq: 12ae}, we obtain
    \begin{equation*}
        \begin{aligned}
            \bE_{\cD}\lrl{\bE_\z\lrl{\what{g}(\z)}} \le \sup_{g \in \bG} \bE_\z\lrl{{g}(\z) \one_{\cA^c}(\z)} &+ 2 \inf_{g\in \bG} \bE_{z}[g(\z)] + \frac{148 B_\bG^\cA \log \left(\calN^{(\delta)}_{\bG_\cA}\right)}{9\,n} + \frac{64 B_\bG^\cA}{n} 
            \\
            &+ \frac{64 \, B_\bG^\cA}{n \, \log\lft(\calN^{(\delta)}_{\bG_\cA} \rgt)} + 5\delta + n\, B^\cA_\bG\, \bP\lrs{\cA^c}.
        \end{aligned}
    \end{equation*}
\end{proof}

\section{Simple network approximation}
\begin{lemma}[Lemma F.7 of \cite{suzuki023diffusion}; Approximation of $1/x$]\label{lemma: NN 1/x} 
    Let $0 < \epsilon < 1$. Then there exists a network parameter $\btheta_{\mathrm{rec}} \in \bTheta_{1,1}(\sL,\sW,\sS,\sB)$ with
    $$
    \sL \equiv \log^2\left( \frac{1}{\epsilon} \right), \qquad \sW \equiv \log^3\left( \frac{1}{\epsilon} \right), \qquad \sS \equiv \log^4\left( \frac{1}{\epsilon} \right), \qquad B \equiv \left( \frac{1}{\epsilon^2} \right),  
    $$
    such that
    $$
    \left| \sN_\sigma( x^\prime | \btheta_{\mathrm{rec}}) - \frac{1}{x} \right| \le \epsilon +  \frac{\left|x^\prime - x \right|}{\epsilon^2}, \qquad \textnormal{for all } x\in[\epsilon, \epsilon^{-1}].
    $$
\end{lemma}
\begin{lemma}\label{lem:basic ReLU}
     For any positive constant $K$ the following hold.
    \begin{enumerate}[label=(\alph*)]
        \item \label{lem: ReLU prod p} \textnormal{(Lemma A.2 of \cite{schmidt2017nonparametric})} There is a network parameter $\btheta_{\times} \in \bTheta_{2,1} \del{K+4,6}$ with $\abs{\btheta_{\times}}_\infty \le 1$ such that
        \begin{equation*}
            \sup_{\bx\in[0,1]^2} \abs{\sN_\sigma(\x|\btheta_{\times}) - x_1 x_2} \le \frac{1}{2^K}, \quad\text{with}\quad \sN_\sigma(\x|\btheta_{\times}) \in [0,1].
        \end{equation*}
        Moreover, $\sN_\sigma((x_1, 0)|\btheta_{\times}) = \sN_\sigma((0, x_2)|\btheta_{\times}) = 0$.
        

    \end{enumerate}
\end{lemma}

\
\\
\begin{lemma}\label{lem:basic ReLU A}
     Let $\sA>1$. For any positive constant $K$ the following hold.
    \begin{enumerate}[label=(\alph*)]
        \item \label{lem: ReLU prod A} There is a neural network  parameter $\btheta_{\times,\sA}\in\Theta_{2,1}\del{9+2\log_2(\sA)+K, 7}$ with $\abs{\btheta_{\times,\sA}}_\infty \le 4 \sA^2$ such that
        \begin{equation*}
            \sup_{\x\in[-\sA,\sA]^2}\abs{\sN_\sigma(\x|\btheta_{\times,\sA})-x_1x_2}\le \frac{1}{2^K}. 
        \end{equation*}

    \end{enumerate}
\end{lemma}

\begin{proof}[Proof of Lemma \ref{lem:basic ReLU A}\ref{lem: ReLU prod A}]
    Observe that 
    \begin{equation}\label{eq: t1}
        x_1 x_2 = 4 A^2 \del{\frac{x_1}{2A} + 1} \del{\frac{x_2}{2A} + 1} - 4A^2 - 2A(x_1 + x_2).
    \end{equation}
    Denote $\btheta^{(1)} \in \bTheta_{2,2}(0,2)$ as a network where \(|\btheta^{(1)}|_\infty \leq \max\left\{0.5 A^{-1}, 1\right\}\), and there are no deep layers, computing the transformation \((x_1, x_2) \mapsto \left(\frac{x_1}{2A} + 1, \frac{x_2}{2A} + 1\right)\). 
    
    Following from Lemma \ref{lem:basic ReLU}\ref{lem: ReLU prod p}, the network 
    $$
    \sN_\sigma \del{\sN_\sigma (\x| \btheta^{(1)})|\btheta_{\times}}
    $$
    with $\btheta_{\times} \in \bTheta_{2,1} \del{K + 4,6}$ and $\abs{\btheta_{\times}}_\infty \le 1$, approximates $\del{\frac{x_1}{2A} + 1} \del{\frac{x_2}{2A} +1}$ up to a uniform error of $1/2^K$. 
    
    We increase the width by one unit to have the affine computation $(x_1, x_2) \mapsto 1 + \tfrac{(x_1 + x_2)}{2A}$ which is positive.  Finally, to perform the remaining linear transform as specified in \eqref{eq: t1}, we add one deep layer at end (right most side) for affine computation $(a, b) \mapsto 4A^2 (a -b)$. Let $\btheta_{\times,A}$ denote this constructed network. We can verify that
    $$
        \sup_{\x\in[-A,A]^2}\abs{\sN_\sigma(\x|\btheta_{\times,A})-x_1x_2}\le \frac{4 A^2}{2^K}
    $$
    with $\btheta_{\times,A}\in\Theta_{2,1}\del{K + 7, 7}$ and $\abs{\btheta_{\times,A}}_\infty \le 4 A^2$. The result follows by redefining the constant $K = K + 2 + 2\log_2(\sA)$.
\end{proof}

\section{Auxiliary results}

\begin{lemma}[Bernstein Inequality]\label{lemma: Bernstein}
Let $\{X_j\}_{j\ge1}$ be sequence of centered independent random variables. Suppose $|X_j| \le a$, for all $j \ge 1$, and $n^{-1}\var\lft( \sum_{j=1}^n X_j \rgt) \le \sigma^2$. Then
$$
\bbP\lft[ |\overline{X}_n| \ge t \rgt] \le 2 \, e^{\frac{-n\,t^2}{2\lft( \sigma^2 + \frac{at}{3}\rgt)}}
$$
\end{lemma}

\begin{lemma}[Gaussian $\ell_2^2$ moment on an $\ell_\infty$ tail event]\label{lemma: tail G}
Let $Z=(Z_1,\ldots,Z_d)\sim \mathtt N(0,\bI_d)$. For any $t>0$,
\[
\mathbb E\!\left[\|Z\|_2^2\,\one_{\{\|Z\|_\infty\ge t\}}\right]
\;\le\;
2\varphi(t)\left(dt+\frac{d^2}{t}\right),
\]
where $\varphi(t)=(2\pi)^{-1/2}e^{-t^2/2}$ is the standard normal density.
\end{lemma}

\begin{proof}
Let $A:=\{\|Z\|_\infty\ge t\}=\bigcup_{i=1}^d\{|Z_i|\ge t\}$. Then
\[
\one_A \le \sum_{i=1}^d \one_{\{|Z_i|\ge t\}},
\]
and hence, by linearity and nonnegativity,
\begin{align*}
\mathbb E\!\left[\|Z\|_2^2\,\one_A\right]
&\le \sum_{i=1}^d \mathbb E\!\left[\|Z\|_2^2\,\one_{\{|Z_i|\ge t\}}\right].
\end{align*}
Fix $i\in\{1,\ldots,d\}$. Using $\|Z\|_2^2=\sum_{j=1}^d Z_j^2$ and independence,
\begin{align*}
\mathbb E\!\left[\|Z\|_2^2\,\one_{\{|Z_i|\ge t\}}\right]
&=
\mathbb E\!\left[Z_i^2\,\one_{\{|Z_i|\ge t\}}\right]
+\sum_{j\neq i}\mathbb E\!\left[Z_j^2\,\one_{\{|Z_i|\ge t\}}\right] \\
&=
\mathbb E\!\left[Z_i^2\,\one_{\{|Z_i|\ge t\}}\right]
+\sum_{j\neq i}\mathbb E[Z_j^2]\;\mathbb P(|Z_i|\ge t) \\
&=
\mathbb E\!\left[Z_i^2\,\one_{\{|Z_i|\ge t\}}\right]
+(d-1)\,\mathbb P(|Z_i|\ge t).
\end{align*}
Summing over $i$ yields
\begin{align}
\mathbb E\!\left[\|Z\|_2^2\,\one_A\right]
&\le
\sum_{i=1}^d \mathbb E\!\left[Z_i^2\,\one_{\{|Z_i|\ge t\}}\right]
+\sum_{i=1}^d (d-1)\,\mathbb P(|Z_i|\ge t) \notag\\
&=
d\,\mathbb E\!\left[W^2\,\one_{\{|W|\ge t\}}\right]
+d(d-1)\,\mathbb P(|W|\ge t),
\label{eq:split}
\end{align}
where $W\sim\mathtt N(0,1)$. We now bound the one-dimensional terms. Using symmetry and integration by parts,
\[
\mathbb E\!\left[W^2\,\one_{\{|W|\ge t\}}\right]
=2\int_t^\infty x^2\varphi(x)\,dx
=2\big(t\varphi(t)+(1-\Phi(t))\big),
\]
where $\Phi$ is the standard normal cdf. Moreover, the standard tail bound
$1-\Phi(t)\le \varphi(t)/t$ implies
\[
\mathbb E\!\left[W^2\,\one_{\{|W|\ge t\}}\right]
\le 2\varphi(t)\left(t+\frac{1}{t}\right),
\qquad
\mathbb P(|W|\ge t)=2(1-\Phi(t))\le \frac{2\varphi(t)}{t}.
\]
Plugging these bounds into \eqref{eq:split} gives
\begin{align*}
\mathbb E\!\left[\|Z\|_2^2\,\one_A\right]
&\le
d\cdot 2\varphi(t)\left(t+\frac{1}{t}\right)
+d(d-1)\cdot \frac{2\varphi(t)}{t}
\;\le\;
2\varphi(t)\left(dt+\frac{d^2}{t}\right),
\end{align*}
which proves the claim.
\end{proof}

\end{document}